%% file: main.tex
\documentclass[]{bytedance_seed}

\input{math_commands.tex}

\usepackage{float} 

\usepackage{url}
\usepackage{amsmath,amssymb,amsthm,bm}
\graphicspath{{figures/}{Figures/}{Appendix/Figures/}}
\usepackage{placeins}
\usepackage{etoc}
\usepackage{colortbl}
\input{cleveref_setup}
\crefname{assumption}{Assumption}{Assumptions}
\Crefname{assumption}{Assumption}{Assumptions}
\usepackage{CJKutf8}

\newtheorem{theorem}{Theorem}

\newtheorem{lemma}[theorem]{Lemma}

\newtheorem{assumption}[theorem]{Assumption}

\AddToHook{env/table/begin}{\setlength{\belowcaptionskip}{6pt}}
\AddToHook{env/table*/begin}{\setlength{\belowcaptionskip}{6pt}}
\colorlet{CorrectShade}{green!15}
\colorlet{WrongShade}{red!15}
\newcommand{\CorrectOrder}[1]{\cellcolor{CorrectShade}#1}
\newcommand{\WrongOrder}[1]{\cellcolor{WrongShade}#1}
\newcommand{\SeqOk}[1]{{\setlength{\fboxsep}{0.6pt}\colorbox{CorrectShade}{\texttt{#1}\vphantom{Ag}}}}
\newcommand{\SeqBad}[1]{{\setlength{\fboxsep}{0.6pt}\colorbox{WrongShade}{\texttt{#1}\vphantom{Ag}}}}
\newcommand{\NegDelta}[1]{\begingroup\setlength{\fboxsep}{1pt}\kern-\fboxsep\colorbox{orange!25}{#1}\kern-\fboxsep\endgroup}

\newcommand{\cL}{\mathcal L}
\newcommand{\Coracle}{\mathcal G_{\mathrm{oracle}}}
\newcommand{\Clinear}{\mathcal G_{\mathrm{linear}}}

\newcommand{\name}{phase sensitivity}
\newcommand{\Name}{Phase Sensitivity}

\fancypagestyle{firststyle}{%
  \fancyhead[L]{\vskip 4mm\includegraphics[width=53.3mm,height=6mm]{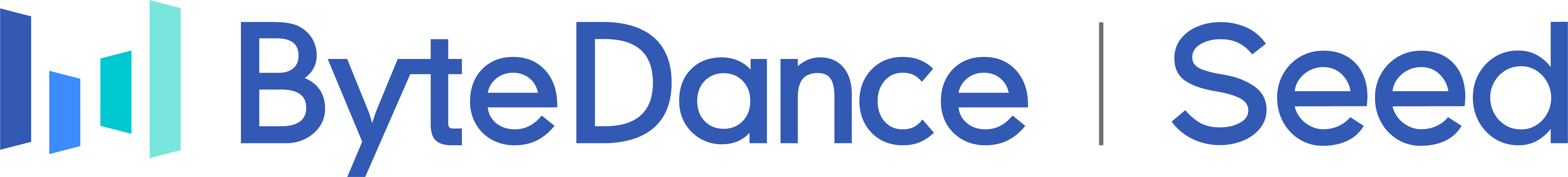}}%
  \fancyhead[R]{\vskip 4mm\raisebox{-1.025mm}{\includegraphics[height=8.05mm,trim=0 1.044 0 0.726,clip]{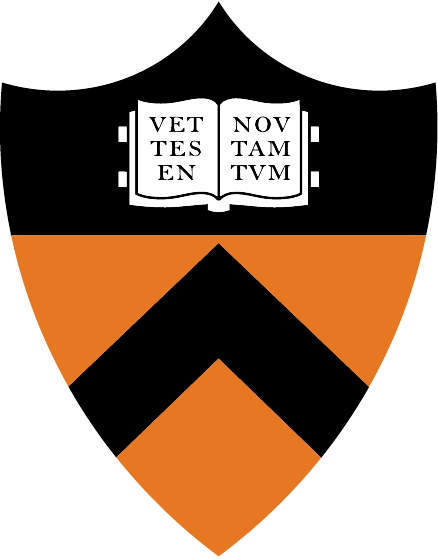}}\hspace{3.857mm}%
    \raisebox{-1.025mm}{\includegraphics[height=8.05mm,trim=1.4 1.22 0.34 1.7,clip]{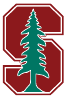}}\hspace{3.006mm}%
    \raisebox{-1.025mm}{\includegraphics[height=8.05mm,trim=0 0.738 0 0,clip]{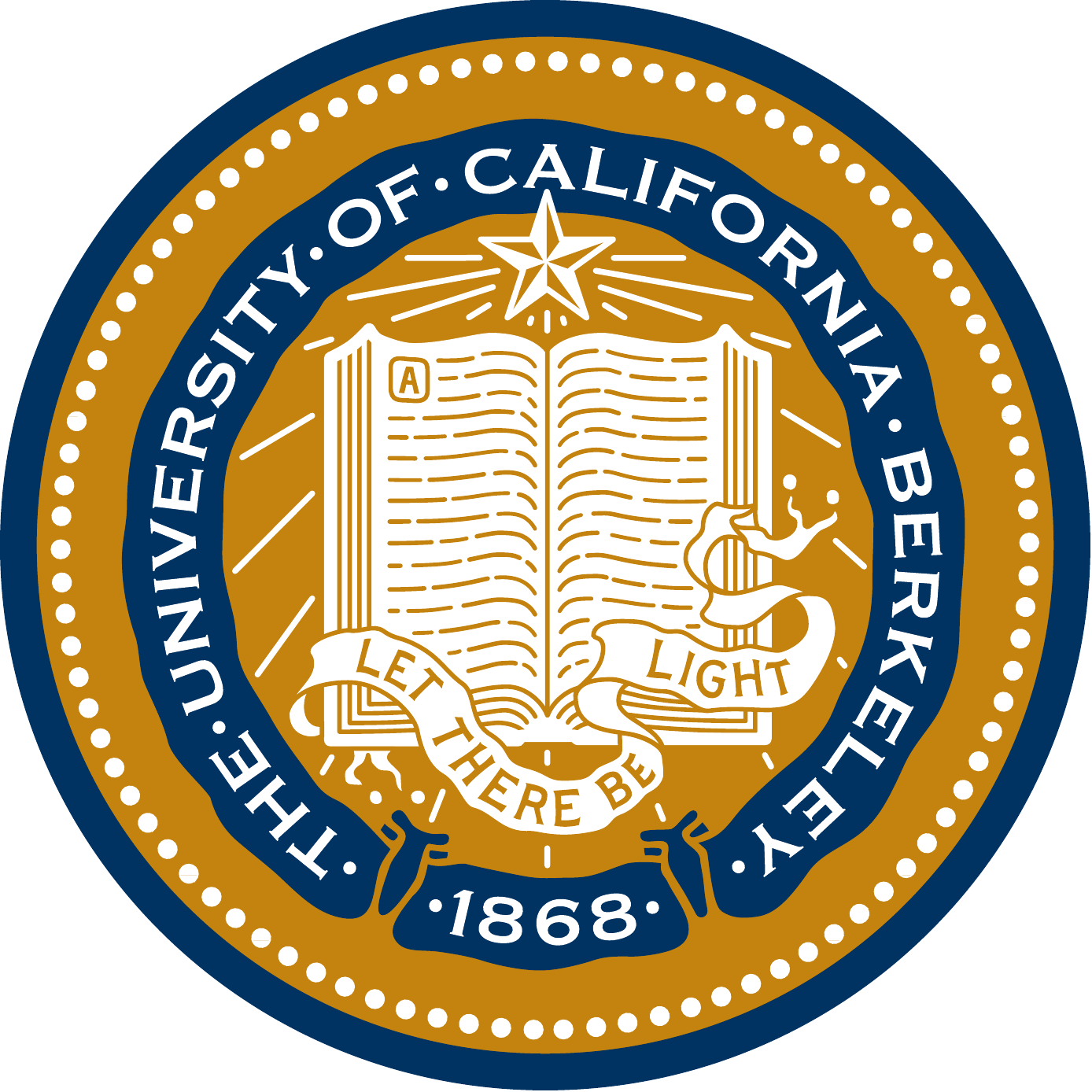}}}%
}

\title{Periodic Weak Spots:\\ Phase Sensitivity from Chunked KV-Cache Compression}

\author[1,2]{Xingyu Zhu*}
\author[1,3]{Pu (Luke) Yi*}
\author[1,4]{Ziheng Cheng*}
\author[1]{Ang Lv*}
\author[1]{Jing Liu}
\author[1]{Lexing Ying}
\author[1,\dagger]{Yiyuan Ma}
\author[1,\dagger]{Xin Dong}

\affiliation[1]{ByteDance Seed}
\affiliation[2]{Princeton University}
\affiliation[3]{Stanford University}
\affiliation[4]{University of California, Berkeley}

\makeatletter
\renewcommand\authorlist{%
  \authorformat[1,2]{Xingyu Zhu*}, \authorformat[1,3]{Pu (Luke) Yi*}, \authorformat[1,4]{Ziheng Cheng*}, \authorformat[1]{Ang Lv*},\\[1mm]
  \authorformat[1]{Jing Liu}, \authorformat[1]{Lexing Ying}, \authorformat[1,\dagger]{Yiyuan Ma}, \authorformat[1,\dagger]{Xin Dong}}
\renewcommand\affiliationlist{%
  \affiliationformat[1]{ByteDance Seed}\\[0.5mm]
  \affiliationformat[2]{Princeton University}, \affiliationformat[3]{Stanford University}, \affiliationformat[4]{University of California, Berkeley}}
\makeatother

\contribution[*]{Core contributors. Work done during internship at ByteDance Seed}
\contribution[\dagger]{Corresponding authors}

\abstract{\input{Sections/abstract}}

\date{\today}
\correspondence{\begin{tabular}[t]{@{}l@{}}Yiyuan Ma at \email{mayiyuan.unicorn@bytedance.com}\\ Xin Dong at \email{xin.dong@bytedance.com}\end{tabular}}

\begin{document}

\maketitle

\etocdepthtag.toc{main}
\input{Sections/introduction}
\input{Sections/mod_sensitivity}
\input{Sections/controlled_pretraining}
\input{Sections/specialization}
\input{Sections/theory}
\input{Sections/related_work}
\input{Sections/conclusion}

\clearpage

\bibliographystyle{plainnat}
\bibliography{references}

\clearpage

\beginappendix
\etocdepthtag.toc{appendix}
\input{Appendix/contents}
\input{Appendix/Sections/induction_head}
\input{Appendix/Sections/mod_sensitivity_details}
\input{Appendix/Sections/small_model_niah}
\input{Appendix/Sections/controlled_pretraining_details}
\input{Appendix/Sections/specialization_details}

\input{Appendix/Sections/complete_sweep}

\input{Appendix/Sections/gate_knockout_comparison}
\input{Appendix/Sections/gate_cycling}
\input{Appendix/Sections/extended_related_work}

\end{document}

%% file: math_commands.tex
\usepackage{amsmath,amsfonts,bm}

\def\1{\bm{1}}

\def\vmu{{\bm{\mu}}}

\def\valpha{{\bm{\alpha}}}
\def\va{{\bm{a}}}
\def\vb{{\bm{b}}}
\def\vc{{\bm{c}}}

\def\ve{{\bm{e}}}

\def\vh{{\bm{h}}}

\def\vk{{\bm{k}}}

\def\vo{{\bm{o}}}
\def\vp{{\bm{p}}}
\def\vq{{\bm{q}}}
\def\vr{{\bm{r}}}
\def\vs{{\bm{s}}}

\def\vv{{\bm{v}}}

\def\vx{{\bm{x}}}
\def\vy{{\bm{y}}}
\def\vz{{\bm{z}}}

\def\mC{{\bm{C}}}

\def\mE{{\bm{E}}}

\def\mW{{\bm{W}}}
\def\mX{{\bm{X}}}

\def\mZ{{\bm{Z}}}

\def\vO{{\bm{O}}}

\DeclareMathAlphabet{\mathsfit}{\encodingdefault}{\sfdefault}{m}{sl}
\SetMathAlphabet{\mathsfit}{bold}{\encodingdefault}{\sfdefault}{bx}{n}

\newcommand{\E}{\mathbb{E}}

\newcommand{\R}{\mathbb{R}}

%% file: cleveref_setup.tex
\usepackage[nameinlink,noabbrev]{cleveref}
\usepackage{etoolbox}

\AtBeginEnvironment{proposition}{\crefalias{theorem}{proposition}}
\AtBeginEnvironment{lemma}{\crefalias{theorem}{lemma}}
\AtBeginEnvironment{corollary}{\crefalias{theorem}{corollary}}
\AtBeginEnvironment{assumption}{\crefalias{theorem}{assumption}}
\AtBeginEnvironment{remark}{\crefalias{theorem}{remark}}

\AddToHook{cmd/appendix/after}{%
  \crefalias{section}{appendix}%
  \crefalias{subsection}{subappendix}%
  \crefalias{subsubsection}{subsubappendix}%
}

%% file: Sections/abstract.tex
Chunked KV-cache compression reduces the memory and attention costs of long-context inference by compressing windows of consecutive tokens into fewer cache entries at a fixed stride. Such compression also introduces a new positional coordinate: a token's phase, or its position relative to compression-window boundaries. We uncover a systematic asymmetry in models using such compression: the same information can be easy to retrieve at one phase and difficult at another. We call this periodic variation in retrieval performance phase sensitivity. In large open-weight models with such compression, long-context retrieval accuracy can differ by up to 40 percentage points across phases, revealing periodic weak spots that average benchmark scores can conceal.

\vspace{0.1in}

To investigate this behavior, we pretrain a family of transformers from scratch across multiple KV-compression designs, reproducing phase sensitivity across the variants. Mechanistic analysis using causal interventions in these models reveals phase specialization: different attention components contribute asymmetrically to retrieving information at different source phases. We further analyze idealized retrieval models, showing how gradient flow dynamics may favor sharp phase specialization. Evaluating models with chunked KV-cache compression thus requires measuring across compression phases: high average accuracy can coexist with systematic positional failures.

%% file: Sections/introduction.tex
\vspace{0.2in}
\section{Introduction}\label{sec:intro}

Long-context inference in transformers can be bottlenecked by the KV cache, whose memory and attention costs grow with context length. A practical remedy, used in models such as DeepSeek-V4~\citep{deepseek2026v4}, is \emph{chunked KV-cache compression}: the model groups consecutive tokens into fixed-size windows, compresses each window into fewer cache entries, and uses these summaries for long-context retrieval~\citep{rae2020compressive,wang2024loma}. The compression happens as context is processed, before a later query specifies what information will need to be recalled. Each summary is therefore query-agnostic and must preserve information that may be needed later.

This design imposes a periodic structure on the context. A new compression window begins every $S$ tokens, where $S$ is the compression stride. Each token thus has a new  coordinate relative to these boundaries beyond its absolute position. We call this coordinate its \emph{phase}, defined as its position modulo $S$. Shifting the input by a few tokens can change which tokens are compressed together without changing their content or relative positions. The same information may thus be compressed differently depending on its phase, and hence retained with different fidelity.

In LLMs trained with chunked KV-cache compression, we find that retrieval performance varies sharply and periodically with the phase of the source information, a failure mode we term \emph{\name} (\Cref{sec:modsens}). A motivating example comes from asking DeepSeek-V4 to complete its own inference code (\Cref{fig:headline}). When we prepend a decorative comment and vary only its length, the model's preferred completion of one token flips with a period of four tokens, matching its compression stride. We then measure retrieval directly through a controlled needle-in-a-haystack evaluation (\Cref{fig:v4_niah}). At 128K context, retrieval accuracy in the DeepSeek-V4 family can differ by up to 40 percentage points across phases. To the best of our knowledge, this is the first report of such internal periodic variation, which average benchmark scores can conceal.

To test whether chunked KV compression is the source of such variation, we pretrain from scratch a broad family of transformers with the Qwen3-0.6B architecture~\citep{qwen3technicalreport} as backbone and KV compression as the only substantial architectural modification (\Cref{sec:controlled}). Every compressed model we evaluate shows periodicity matching the compression stride, including models that simply average token representations within each window. By contrast, full-attention baselines have no comparable periodicity. A natural first explanation is the window boundary itself. In our needle-in-a-haystack task, each key is immediately followed by its value, so with non-overlapping windows a key--value pair either falls inside one window or is split across two, and a split pair is plausibly harder to retrieve. However, boundaries alone do not account for the full periodic pattern: accuracy also varies widely across phases that keep the key and its value within a single window. This suggests a deeper look into how the model writes information into compressed memory and reads it back.

Causal interventions on KV heads and layers support this view: retrieval at different phases depends differently on different heads, a systematic pattern we call \emph{phase specialization} (\Cref{sec:head-knockouts}). The pattern appears in both our pretrained models and DeepSeek-V4. To examine how such specialization can develop, we study compression gates, which determine how much each within-window position contributes to a compressed entry (\Cref{sec:qwen3-spec}). In a pretrained model whose gates can learn separate positional preferences, some heads develop persistent preferences during pretraining that recur across inputs. Their intervention effects are strongest near the phases they favor, supporting a division of retrieval work among the examined heads.

The emergence of these preferences suggests that training can favor selective retention even when future queries are unknown. We examine this possibility in an idealized induction task with compressed attention (\Cref{sec:theory}), where mixing information in the compression window can weaken the match between a compressed key and a later query. We show that the optimal compressor favors particular phases, and gradient flow drives gates toward persistent preferences. Together, these findings show that average accuracy is not enough to evaluate models with chunked KV-cache compression. Their reliability should also be measured across phases, for example by shifting the input by a few tokens while keeping the queried information fixed.

%% file: Sections/mod_sensitivity.tex
\section{Stride-Periodic Behavior in DeepSeek-V4 and DeepSeek-V4.1}\label{sec:modsens}
\newcommand{\CodeIn}[1]{{\small\textbf{\texttt{#1}}}}
\newcommand{\LeaveOut}[1]{}

\begin{figure}[tb]
    \centering
    \includegraphics[width=0.98\linewidth]{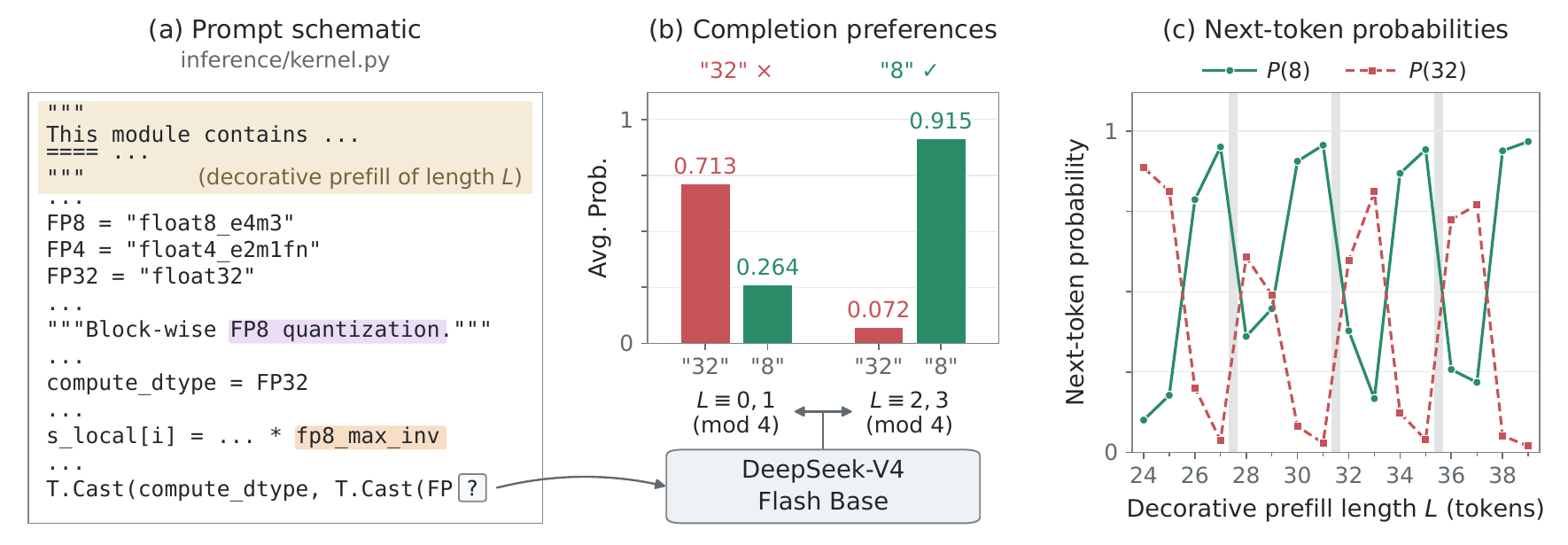}
    \caption{
    \textbf{Next-token predictions by DeepSeek-V4-Flash-Base on a code snippet.} The prediction flip periodically as a variable-length decorative filler shifts the input.
    (a) The decorative filler is a docstring whose length $L$ varies by the number of  ``\CodeIn{=}'' characters.
    The ground-truth completion is ``\CodeIn{8}'', supported by the highlighted components in the context.
    (b) Averaged over the eight lengths in each group, the model places higher probability on the incorrect ``\CodeIn{32}'' when $L\bmod4\in\{0,1\}$ and on the correct ``\CodeIn{8}'' when $L\bmod4\in\{2,3\}$.
    (c) The preferred completion switches every two tokens and the pattern repeats every four, matching DeepSeek-V4's compression stride.
    }
    \label{fig:headline}
\end{figure}

\subsection{A Periodically Flipping Code-Completion Example}\label{sec:v4-code}

We begin by asking DeepSeek-V4-Flash-Base to complete DeepSeek-V4's official inference code (\Cref{fig:headline}a). The prompt ends at \CodeIn{FP} inside a type cast in an FP8 quantization function, and the original source continues with ``\CodeIn{8}''. The surrounding code favors this continuation. The function quantizes to FP8, while the outer cast already converts to \CodeIn{compute\_dtype}, which is FP32. Completing ``\CodeIn{32}'' would duplicate the outer cast and skip the rounding step.

To test whether this prediction depends on position, we prepend a decorative docstring and vary its length by changing the number of repeated ``\CodeIn{=}'' characters on its second line. The code is unchanged, and each additional filler token shifts the code and the completion position by one. Across 16 filler lengths $L$ from 24 to 39 tokens, the top-1 prediction flips periodically (\Cref{fig:headline}b,c). In particular, the model predicts \CodeIn{8} when $L\bmod4\in\{2,3\}$ and \CodeIn{32} at the remaining eight lengths, and $P(\text{\CodeIn{8}})-P(\text{\CodeIn{32}})$ ranges from $-0.79$ to $0.95$. Changing only the docstring length thus reverses the preferred continuation. The reversals repeat every four tokens, matching the stride $S=4$ of DeepSeek-V4's compressed sparse attention (CSA).

This match suggests that the prediction depends on where the code falls relative to the compression windows. The reversals are not limited to one filler. In all four filler families we tested, most predictions follow the same four-token pattern (\Cref{app:v4-code-details}), and the post-trained DeepSeek-V4-Flash-0731 and DeepSeek-V4.1-Flash~\citep{deepseek2026v41} also show periodic reversals under several filler sweeps, matching their compression strides ($S=4$ and $S=2$ respectively). In contrast, DeepSeek-V3.1-Base, without chunked KV compression, ranks \CodeIn{32} above \CodeIn{8} at only 4 of the same 64 inputs, with very small margins (\Cref{app:v4-code-posttrained}).
We give the exact inputs, results for the other models, and two further code-completion examples in \Cref{app:v4}.

\subsection{Measuring Phase Sensitivity in Retrieval}\label{sec:v4-niah}

To measure this effect more systematically, we turn to a controlled key--value based retrieval task in which we can measure the model's capability of retrieving the same piece of information, but at different relative positions with respect to the compression boundaries. For compression stride $S$, we call $t\bmod S$ the \emph{phase} of a key-value pair starting at position $t$, and we say that retrieval is \emph{phase-sensitive} when certain retrieval metric, such as accuracy or the probability assigned to the correct answer, changes systematically with the phase of the source.

We evaluate the base and post-trained versions of DeepSeek-V4-Flash and DeepSeek-V4-Pro, as well as the post-trained DeepSeek-V4.1-Flash, on single-token needle-in-a-haystack (NIAH) retrieval. Each 128K-token prompt holds 16k key--value pairs and asks for the value of a source key near its middle (\Cref{fig:v4_niah}a).
We sort the prompts into eight residue groups by the source key's position modulo eight, which covers two stride cycles for DeepSeek-V4 and four for DeepSeek-V4.1. All groups share the same prompt length, query position, and key--value bindings, and we choose their filler lengths so that the mean and variance of the source key's absolute position are the same in every group. We defer details about the construction of the retrieval dataset to \Cref{app:v4-niah-details}.

\begin{figure}[t]
    \centering
    \includegraphics[width=\linewidth]{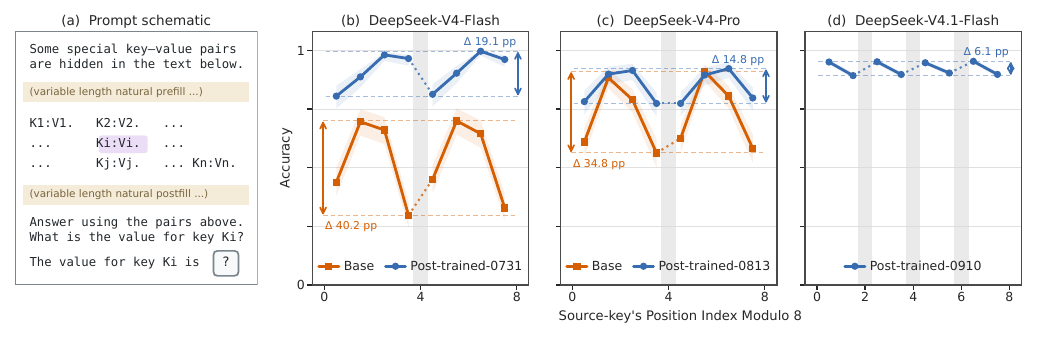}
    \caption{
    \textbf{NIAH retrieval accuracy of DeepSeek-V4 family for key-value pairs with different offsets.}
    All five models exhibit \emph{\name}, with a period of 2 or 4 matching the CSA stride.
    An answer counts as correct if, after stripping whitespace, it begins with the gold value.
    The marker for residue $r$ of the target key's position modulo 8 sits between ticks $r$ and $r{+}1$.
    Gray bands mark stride boundaries, shading shows two-sided 95\% Wilson intervals, and $\Delta$ is the best-to-worst gap.
    }
    \label{fig:v4_niah}
\end{figure}

In \Cref{fig:v4_niah}, we can see that the base checkpoints are sharply phase-sensitive, with maximum gaps of up to 40.2 percentage points. Post-training raises accuracy and narrows these gaps, but the gaps of the post-trained DeepSeek-V4 models remain large. DeepSeek-V4.1-Flash has the smallest gap, yet all four of its even residue groups remain above all four odd ones. Average accuracy alone can conceal these differences across phases.

%% file: Sections/controlled_pretraining.tex
\section{Controlled Pretraining Studies}\label{sec:controlled}

DeepSeek-V4 differs from a standard transformer in many ways besides chunked KV compression, and we cannot retrain it without compression to see whether the periodicity disappears. To test whether chunked compression is the source of \name, we pretrain a Qwen3-derived family of models from scratch in which chunked compression is the only substantial architectural change, together with full-attention baselines trained the same way. We also vary how compression is performed to see which aspects of \name, such as its period, depend on that design. These models also support the mechanistic analyses in \Cref{sec:specialization}.

\subsection{A Family of KV Compression Kernels}\label{sec:setup}

Each kernel compresses a window into one key and one value per KV head. A \emph{payload map} sets what each token contributes, and a \emph{compression gate} sets how strongly it enters the summary.

For one head, let \(\vh_{j,i}\in\mathbb R^d\) be the hidden state at offset \(0\leq i<W\) in a window \(j\) of size \(W\). For branch \(B\in\{K,V\}\), the payload map, gate map, and positional bias have shapes \(\mC_i^B\in\mathbb R^{d_h\times d}\), \(\mZ_i^B\in\mathbb R^{d_g\times d}\), and \(\vb_i^B\in\mathbb R^{d_g}\). Here \(d\) is the hidden width, \(d_h\) is the head dimension, and \(d_g\) is the gate dimension. For the key branch, we compute gate logits
\begin{equation}\label{eq:gate-logit}
\vs_{j,i}^K=\mZ_i^K\vh_{j,i}+\vb_i^K.
\end{equation}
A softmax over the window turns these logits into gate scores, which weight the payloads:
\begin{equation}\label{eq:controlled-compression}
\valpha_{j,i}^K
=\frac{\exp(\vs_{j,i}^K)}{\sum_{u=0}^{W-1}\exp(\vs_{j,u}^K)},
\qquad
\widehat K_j
=\sum_{i=0}^{W-1}\valpha_{j,i}^K\odot\left(\mC_i^K\vh_{j,i}\right).
\end{equation}
Scalar gates (\(d_g=1\)) weight all payload dimensions together, while vector gates (\(d_g=d_h\)) weight them separately. Softmax acts along the offset dimension, and \(\odot\) broadcasts scalar scores or multiplies element-wise. The value branch is analogous.

We vary this family to test candidate explanations of \name. Our reference model has one KV head per layer and uses non-overlapping eight-token windows with scalar gates. Window size \(W\) sets how many tokens an entry summarizes, and stride \(S\) sets how often an entry is produced, so the reference is W8/S8. Varying them separately shows whether the period follows the window or the stride, and whether overlapping windows (\(S<W\)), in which a token appears at several offsets while its phase stays fixed, remove the effect. Replacing learned gates with uniform averaging, which fixes each gate score at \(1/W\), tests whether \name\ requires learned gates. We describe parameter-sharing variants in \Cref{app:scratch}. The family also includes DeepSeek-V4's CSA kernel~\citep{deepseek2026v4}, which tests whether the kernel studied in \Cref{sec:modsens} produces \name while being used on the Qwen backbone.

\begin{figure}[b]
\centering
\includegraphics[width=\linewidth]{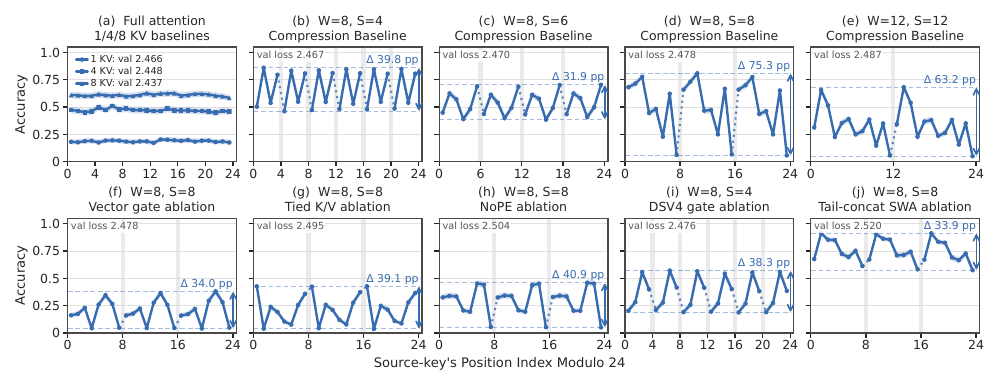}
\caption{
\textbf{Phase-sensitive retrieval accuracy by source-key position modulo 24.}
(a) shows full-attention baselines with one, four, and eight KV heads per layer.
(b)--(e) show W8/S4, W8/S6, W8/S8, and W12/S12 models that otherwise match the one-KV-head reference.
Ablations (f) uses vector gates and (g)--(j) vary K/V tying, positional encoding, compression kernel, and local-memory policy. Phase sensitivity persists across the compressed-model panels.
}
\label{fig:phase-accuracy}
\end{figure}
To measure \name\ in these models, we adapt the NIAH task of \Cref{sec:v4-niah} to 64 random key--value pairs with adjacent single-token keys and values, and report full-vocabulary next-token accuracy. \emph{Prefix Padding} follows the construction in \Cref{sec:v4-niah}, shifting all records together with outer fillers. Because the records sit back to back, this shift changes every distractor's phase along with the target's, so a phase effect could in principle come from the distractors rather than the target. We therefore add \emph{In-sequence Padding}, which varies the gaps between records so that distractor phases no longer move in lockstep with the target. Within each protocol, we fix the query position and match the mean and variance of the target-position sampling distributions across residue groups of source-key position modulo 24, a common multiple of the tested strides (\Cref{app:phase-balanced-evaluation}).

\subsection{\Name\ across the Kernel Family}\label{sec:phase-accuracy}

If \name\ arises from chunked compression, compressed models should show it and full-attention baselines should not. We see this contrast in \Cref{fig:phase-accuracy}. Every compressed model varies periodically with the source-key phase, including the DeepSeek-style kernel (\Cref{fig:phase-accuracy}i), while the full-attention baselines stay flat (\Cref{fig:phase-accuracy}a). The same holds for every compressed run in \Cref{tab:controlled-sweep-results}, with Prefix Padding gaps of up to 78 percentage points against at most 6.1 for the baselines.

Note that summary metrics conceal much of this difference. With eight KV heads, the W8/S8 compressed model nearly matches its full-attention baseline in mean accuracy and validation loss, yet its worst residue is only 9.9\%, against 58.4\% for the baseline. At one and four KV heads, the compressed models reach higher mean accuracy than the baselines at comparable validation loss.

Varying window and stride separately shows that the period follows the stride in the tested geometries, whether or not windows overlap. At \(W=8\), strides \(S\in\{4,6,8\}\) give periods of about four, six, and eight tokens (\Cref{fig:phase-accuracy}b--d). W4/S4 shares W8/S4's four-token period despite its smaller window (\Cref{fig:controlled-group-compression-geometry}), and W12/S12 produces an approximately twelve-token period (\Cref{fig:phase-accuracy}e).
In these experiments, \name\ needs neither rotary position embeddings (RoPE) nor learned gates: it persists without RoPE (\Cref{fig:phase-accuracy}h) and when uniform averaging replaces the gates (\Cref{fig:controlled-group-compression-operator}). In \Cref{app:complete-controlled-sweep}, we extend these results across head counts, random seeds, and other variants, and to In-sequence Padding, so the pattern does not require all records to shift together.

A natural first explanation is the window boundary, since adjacent keys and values can fall in separate windows. In the W8/S8 reference, only a key at phase seven has its value in the next window. Yet accuracy varies by more than 50 percentage points among phases zero through six, whose pairs share a window (\Cref{fig:phase-accuracy}d). The boundary does not capture the full accuracy trend. This remaining variation suggests that the answer lies in how the model writes to and reads from compressed memory, so we next ask how retrieval at each phase depends on individual KV heads (\Cref{sec:specialization}).

%% file: Sections/specialization.tex
\section{Phase Specialization}\label{sec:specialization}

Retrieval at each phase runs through the attention heads that write and read compressed entries, and a head may serve some phases more reliably than others. We call this \emph{phase specialization}: the contributions of some heads to retrieval depend systematically on phase, so different heads contribute asymmetrically to information at different phases. The definition presupposes no cause, whether boundary straddling or within-window differences.

We measure this asymmetry with knockouts in our pretrained models and DeepSeek-V4-Flash-Base (\Cref{sec:head-knockouts}), and then examine a sharper form in one model, where a head's preference for some phases is also visible in how it compresses (\Cref{sec:qwen3-spec}).

\subsection{Probing Phase Specialization through KV-Head Knockouts}
\label{sec:head-knockouts}

To find which heads matter for retrieval at each phase, we knock out one KV head at a time~\citep{olsson2022induction,wu2025retrieval} and see where accuracy drops. In particular, we use \emph{mean replacement}~\citep{wang2023interpretability}, which replaces its output at the intervened positions with its average over many other prompts, so that output no longer carries anything specific to the current prompt, such as which value belongs to the queried key. A drop in accuracy at a phase is evidence that the head's prompt-specific output contributes causally to retrieval at that phase.

\begin{figure}[b]
\centering
\includegraphics[width=\linewidth]{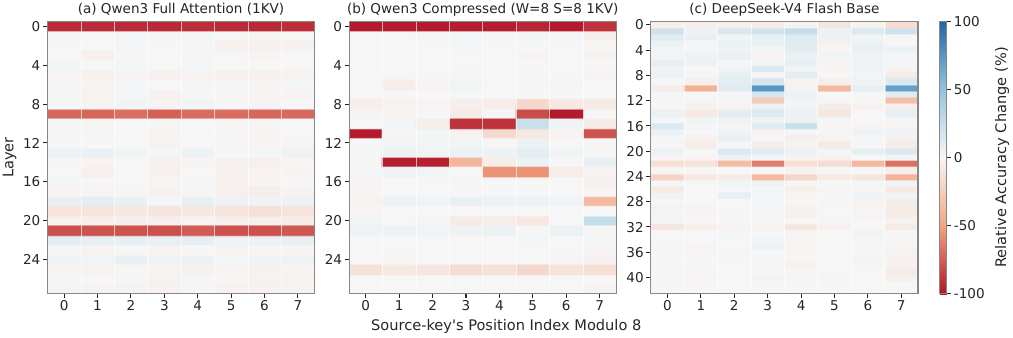}
\caption{\textbf{Retrieval contributions differ across phases.}
(a) Qwen3 full attention, (b) Qwen3 with W8/S8 compression, and (c) DeepSeek-V4-Flash-Base, all with one KV head per layer. Each row knocks out one layer's KV head, and columns group source positions modulo eight. Colors show accuracy change divided by clean accuracy at each phase. The model without chunked KV compression (a) shows uniform horizontal stripes across residue groups, indicating that each layer contributes similarly at every source position, while the models with chunked KV compression (b,c) show heterogeneous patterns across the residues of the target key. In (c), the pattern also repeats with period four, matching the CSA stride, and the odd layers, which use heavily compressed attention (HCA) with compression rate 128, contribute little to retrieval.}
\label{fig:phase-knockouts}
\end{figure}

In our pretrained models, we run these knockouts on the same evaluation sets used in \Cref{sec:controlled}, replacing the outputs of the query heads that read the KV head at the final position with their means over a disjoint set of calibration prompts. The computation at other token positions stays unchanged. We adopt a similar knockout strategy for DeepSeek-V4-Flash-Base and replace a whole layer's attention output at a band of query suffix, with 256 evaluation prompts per residue group (see \Cref{app:patching-deepseek} for details).
We visualize the effect of knocking out each head on retrieval performance for information at each phase in \Cref{fig:phase-knockouts}. The knockout maps differ between full attention and compression. In the full-attention baseline, the layers with large losses reduce accuracy by similar amounts at every position (\Cref{fig:phase-knockouts}a). In the W8/S8 compressed model, some heads have losses concentrated on a few phases, and different heads have different phase-dependent effects (\Cref{fig:phase-knockouts}b). In DeepSeek-V4-Flash-Base, the largest losses affect every phase unevenly and a few earlier layers affect only some phases, repeating every four positions, its compression stride (\Cref{fig:phase-knockouts}c).

This pattern is not specific to one model. In \Cref{app:complete-controlled-sweep}, we give knockout maps for every compressed model in \Cref{tab:controlled-sweep-results} under both padding protocols. In most compressed models, some heads have losses concentrated on a few phases that repeat every stride, and in the others, losses spread over all phases with periodic variation. With several KV heads, some individual heads still show losses concentrated on one or two phases, so the asymmetry is not only a property of layers (\Cref{fig:controlled-group-native-kv-head-count}).

\subsection{A Case Study of Emerging Gate Preferences}
\label{sec:qwen3-spec}

Knockouts show which heads matter at which phases, but not why. A natural candidate for the source of cause is the compression gate, which sets each offset's weight in a compressed entry. A head whose gate persistently favors some offsets should matter most at the corresponding phases. We test this in the W8/S8 reference model, whose scalar key and value gates have separate parameters for each offset and whose windows do not overlap ($W=S$), so offset $i$ is phase $i$.

Because gate scores depend on the input, a gate can concentrate its scores among offsets in two ways. \emph{Static concentration} favors the same offsets in every window, whereas \emph{input-dependent concentration} favors offsets that vary with content. Only static concentration ties a head to fixed phases. We measure both with the \emph{effective number of offsets}, defined as $\exp(H(\cdot))$, where $H$ is Shannon entropy. This metric equals 1 for a one-hot distribution and $W$ for a uniform one. With $\valpha\in\R^W$ the gate scores of a window and expectations over held-out natural-language text, we measure static concentration by $\exp(H(\mathbb E[\valpha]))$ and the average per-window concentration by $\mathbb E[\exp(H(\valpha))]$. A gate that applies the same scores to every window would be near the diagonal of \Cref{fig:phase-emergence}a, since expectation would commute when $\valpha$ is approximately constant.

\begin{figure}[b]
\centering
\includegraphics[width=\textwidth]{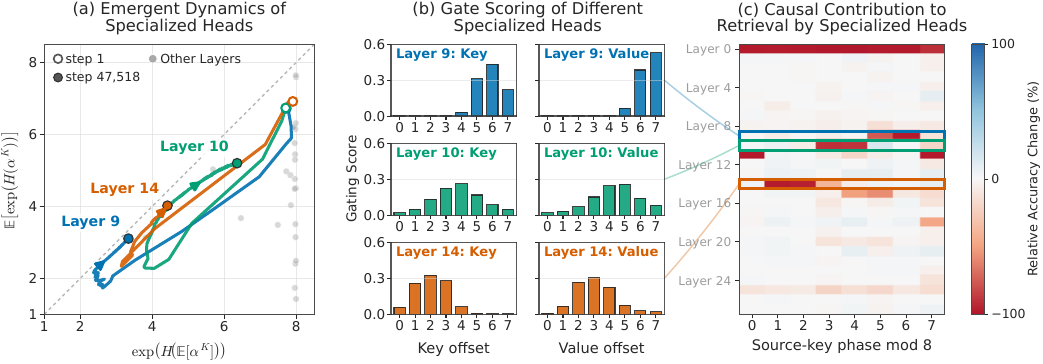}
\caption{\textbf{Emergence and retrieval roles of gate preferences.}
(a) Key-gate trajectories during pretraining: some heads drift along the diagonal and become statically concentrated (as highlighted), while others remain unbiased in expectation with $\exp(H(\mathbb E[\valpha]))\approx 8 $. (b) The highlighted layers 14, 10, and 9 develop static preferences for early, middle, and late offsets in the compression window. (c) For the statically concentrated heads, their peaking phases directly match where they are causally important for retrieval. This is the same figure as \Cref{fig:phase-knockouts}b for Qwen3 with W8/S8 compression.}
\label{fig:phase-emergence}
\end{figure}

As shown in \Cref{fig:phase-emergence}, all gates appear to be near uniform around initialization, showing neither kind of concentration (open circles in \Cref{fig:phase-emergence}a, close to the upper right diagonal). During pretraining, some gates, such as those in layers 9, 10, and 14, move toward the lower left along the diagonal, forming static concentration. The other key gates end near the right edge, with different levels of input-dependent concentration but no consistent bias (gray points). In \Cref{fig:phase-emergence}b we visualize the mean gate scores of these heads. Each key gate peaks at a few adjacent offsets, early in the window for layer 14, in the middle for layer 10, and late for layer 9, and each value gate peaks one offset later, so each head favors a token and its successor, as in a bigram key--value pair.

These peaks match the phases each head specializes in as observed in \Cref{sec:head-knockouts}: its knockout losses concentrate where its key gate favors the queried token and its value gate favors the successor (\Cref{fig:phase-emergence}c). Apart from layer 0, the largest knockout losses come from statically concentrated heads. With more heads per layer, static concentration becomes sharper: in the W8/S8 model with four KV heads, many heads peak at a single offset and lose accuracy at a single phase (\Cref{fig:gate-knockout-overlays-heads}). We defer further results across KV-head counts, random seeds, and compression configurations, including the final gate concentration of every head (\Cref{fig:gate-concentration-all}), to \Cref{app:gate-knockout-comparison}.

If learned gate preferences help determine which phases are retrieved well, cycling the offset-specific gate parameters by $k$ offsets should also approximately shift the accuracy pattern by $k$ phases. In \Cref{app:phase-cycling}, we test this prediction in five models with gates exhibiting sharp static concentration (including the W8/S8 reference model). For cycles of one to three offsets in either direction, within-window accuracy patterns largely shift as predicted. These interventions support a causal contribution of learned gate preferences to within-window retrieval asymmetry in these models.

We conjecture that the observed \name\ in accuracy reflects both boundary and within-window asymmetries. The first distinguishes key--value pairs that share a window from pairs that straddle one. The second concerns differences among phases within a window, with static gate concentration as one mechanism that can shape which phases each head retains. Accuracy in the reference model varies by more than 50 percentage points among phases whose key--value pairs all share a window (\Cref{sec:phase-accuracy}), while uniform averaging also exhibits within-window variation, showing that this second asymmetry can arise without learned gates.

%% file: Sections/theory.tex
\section{Theory: Emergence of Phase Specialization in Compression}
\label{sec:theory}

In \Cref{sec:qwen3-spec}, some gates develop persistent preferences for the same offsets across windows, and these heads matter most at the corresponding phases. In an idealized induction model, we ask why compression should concentrate when the future query is unknown, why the concentration becomes static, and whether heads divide the phases among themselves.

\subsection{Concentration in Optimal Solutions}
\label{sec:theory-optimal}

Consider a context $\mX=(\vx_{\ell,r})_{\ell\leq L,\,r\leq S}$ of $LS$ distinct token embeddings, sampled uniformly without replacement from a vocabulary $\mathcal V\subset\mathbb R^d$ of size $N$ and arranged in $L\geq4$ non-overlapping chunks of stride $S\geq3$. Position $r$ within a chunk is offset $r-1$ and thus a single phase. After compression, a query $\vq$ repeats a token sampled uniformly from the $L(S-1)$ nonfinal chunk positions and requests its successor $\vy$ in the same chunk~\citep{olsson2022induction,bietti2023birth}.

In particular, each of $1\leq H<S$ heads compresses chunk $\ell$ into
\begin{equation}
 \vk_{h,\ell}=\sum_{r=1}^S \vp_{h,\ell}^{K}(r)\vx_{\ell,r}\in\R^d,
 \qquad
 \vv_{h,\ell}=\sum_{r=1}^S \vp_{h,\ell}^{V}(r)\vx_{\ell,r}\in\R^d,
 \label{eq:main-compressed-states}
\end{equation}
where the compression coefficients $\vp_{h,\ell}^{K},\vp_{h,\ell}^{V}$ are probability vectors in $\R^S$ and are chosen before the query is revealed. For cleaner analysis, we fix the embeddings during training and set the query, key, and value maps to identities.

Each head weights its compressed values by a softmax over the chunk scores $\gamma\langle\vq,\vk_{h,\ell}\rangle$, with $\gamma\in(0,1]$, and a tied unembedding of scale $\beta>0$ maps the summed head outputs to next-token logits. We write $\mathcal L$ for the resulting population cross-entropy loss of predicting $\vy$ (\Cref{app:full-vector-theory}).

We assume an idealized embedding geometry with a common component $\kappa>0$ and error $\varepsilon\geq0$:
\begin{equation}
 \left|\vx^\top\vx'-\kappa-\mathbf1\{\vx=\vx'\}\right|\leq\varepsilon
 \qquad\forall\vx,\vx'\in\mathcal V.
 \label{eq:main-incoherent-gram}
\end{equation}
The common component reflects the anisotropy of token representations~\citep{gao2019representation,razzhigaev2024shape,queipo2026attention}, and exact geometry means $\varepsilon=0$.

Under exact geometry, a head's target-specific logit is proportional to its attention on the correct chunk times its value weight on the answer. We say that a head \emph{selects} position $r$ when it puts all its key weight on $r$ and all its value weight on $r+1$, the most concentrated compression. Selection strengthens both factors for queries at $r$ but gives up direct evidence for queries elsewhere.

To see what the loss itself favors, we minimize $\mathcal L$ over an oracle class $\Coracle$, whose compression coefficients may depend arbitrarily on the context but not on the query, without imposing a gate parameterization. Writing $\ve_r$ for the one-hot vector at position $r$, we obtain the following result.
\begin{theorem}[Informal: optimal compression is concentrated]
\label{thm:incoherent-hard}
For sufficiently large $N$ and sufficiently small $\varepsilon$, every global minimizer over $\Coracle$ sets $\vp_{h,\ell}^{K}=\ve_r$ and $\vp_{h,\ell}^{V}=\ve_{r+1}$ in every context, chunk $\ell$, and head $h$, where $r\in\{1,\ldots,S-1\}$ may depend on all three. When $\varepsilon=0$, the selected positions are distinct across heads within each context and chunk.
\end{theorem}
Although the query is unknown, optimal compression is maximally concentrated, with one effective offset per head and per K/V branch in each chunk. The selected position may vary across contexts and chunks, however, so optimality requires concentration in each chunk but not static concentration.

\subsection{Static Concentration under Gradient Flow}
\label{sec:theory-dynamics}

To study when concentration becomes static, we restrict the oracle class to $\Clinear$, a simplified form of the gates in \Cref{eq:gate-logit}. Each head has a key and a value gate vector $\vc_{h,r}^{B}\in\mathbb R^d$ for every position $r$, shared across contexts and chunks, which are the only parameters we train:
\begin{equation}
 \vp_{h,\ell}^{B}(r)=
 \frac{\exp\!\left(\langle\vc_{h,r}^{B},\vx_{\ell,r}\rangle\right)}
 {\sum_{j=1}^S\exp\!\left(\langle\vc_{h,j}^{B},\vx_{\ell,j}\rangle\right)},
 \qquad B\in\{K,V\}.
 \label{eq:main-log-linear-gate}
\end{equation}
Under exact geometry, the component of $\vc_{h,r}^{B}$ along the embedding mean acts as an input-independent bias for position $r$, so these gates can express static concentration, as in \Cref{sec:qwen3-spec}, although the parameterization alone does not force it. We analyze gradient flow on $\mathcal L$, the continuous-time limit of gradient descent, from a small random initialization.
\begin{samepage}
\begin{theorem}[Informal: gradient flow yields static concentration]
\label{thm:local-hardening}
Assume $\varepsilon=0$ and uniform initialization in a Euclidean ball centered at zero of sufficiently small radius $\rho$, in the effective gate-parameter space of \Cref{ass:gradient-flow-setting}. For fixed model scales and sufficiently large $N$ relative to $\log(e/\rho)$, with high probability over the initialization, gradient flow on $\mathcal L$ yields, for each head $h$, a position $r_h\in\{1,\ldots,S-1\}$ such that $\vp_{h,\ell}^{K}(r_h),\vp_{h,\ell}^{V}(r_h+1)\to1$ as training time tends to infinity, uniformly over contexts and chunks. Different heads may select the same $r_h$.
\end{theorem}
\end{samepage}

For each head, the key gate scores converge to the one-hot vector at $r_h$ and the value gate scores to the one at $r_h+1$ in every context and chunk, so both the static and the per-window effective numbers of offsets tend to one, the lower-left corner of \Cref{fig:phase-emergence}a. Whereas \Cref{thm:incoherent-hard} lets the selected position vary across inputs, gradient flow fixes it for each head. This gives an idealized account of the adjacent key and value preferences of layers 9, 10, and 14 (\Cref{fig:phase-emergence}b), whose gate scores remain softer than the one-hot limit.
We defer the formal theorem statements and proofs to \Cref{app:full-vector-theory}, with relaxed conditions on the number of chunks $L$ and attention scale $\gamma$.

%% file: Sections/related_work.tex
\section{Related Works}\label{sec:related}

Learned compressed memories reduce the cost of retaining long contexts. Compressive Transformer summarizes past activations~\citep{rae2020compressive}, LoMA trains models to use compressed KV representations~\citep{wang2024loma}, Activation Beacon replaces raw activations with the KV entries of regularly spaced learned tokens~\citep{zhang2025activation}, and CAT lets later chunks attend to compressed vectors of earlier chunks~\citep{prakash2025cat}. Sparse attention~\citep{child2019sparse,yuan2025nsa,lu2025moba,deepseek2026v4} and cache eviction~\citep{zhang2023h2o,li2024snapkv} impose related block or token structure on memory. Sparse Transformer already defines a strided head whose visibility depends on position modulo a stride~\citep{child2019sparse}, a precedent for the periodic structure we study. In hybrids that keep an exact branch over distant tokens, such as NSA~\citep{yuan2025nsa}, that branch could mask a failure of the compressed branch alone. DeepSeek-V4 keeps exact tokens only in its local window~\citep{deepseek2026v4}, and we find \name\ in the full models (\Cref{sec:modsens}).

On the failure mode side, the Lost in the Middle effect shows that retrieval quality depends on where relevant information appears in a long context~\citep{liu2024lost}. RULER shows that easy needle tests can overstate usable context length~\citep{hsieh2024ruler}, and \citet{chen2026pitfalls} show that aggregate scores can hide method-specific failures of KV cache compression. In gist-based compression, \citet{deng2025silver} identify degradation near generation-segment boundaries. We instead focus on the phase of the source token, which recurs every compression stride and is distinct from overall context depth or the query's position within a segment.

Mechanistic studies identify retrieval heads whose ablation impairs long-context retrieval~\citep{wu2025retrieval}, and cache policies such as RazorAttention and DuoAttention keep full caches only for such heads~\citep{tang2024razor,xiao2024duo}. These works establish that heads differ in their retrieval roles. We show that in compressed models, a head's retrieval contribution also depends on phase and, in the examined heads, aligns with the offsets its compression gate prefers (\Cref{sec:specialization}).

On the theory side, gradient descent can turn softmax attention into hard token selectors~\citep{tarzanagh2023svm,tarzanagh2023maxmargin}, and multi-head training dynamics can allocate tasks across heads~\citep{chen2024dynamics,yuksel2026incremental}. We do not claim the first theory of attention hardening or head specialization. Our analysis instead concerns compression chosen before the query is known, and asks when concentrated compression is optimal and when training makes it static (\Cref{sec:theory}). We give a more extensive discussion of related work in \Cref{app:related}.

%% file: Sections/conclusion.tex
\section{Conclusion and Discussion}\label{sec:conclusion}

In this paper, we identify \name\ as a systematic failure mode in models with chunk-compressed KV caches: retrieval accuracy depends on the source information's phase relative to the compression windows. The phenomenon appears in DeepSeek-V4, V4.1, and all compressed models we pretrained, with a period following the compression stride. KV-head knockouts in our pretrained models and attention-readout interventions reveal phase-dependent contributions to retrieval, providing causal evidence for phase specialization in these tested models. Finally, our optimization dynamics analysis shows how training can produce static concentration of compression gates in a simplified bigram retrieval model, illustrating a mechanism for phase specialization under the stated assumptions.

Our mechanistic conclusions are strongest in the controlled settings studied here and do not establish a universal cause of \name\ in large, heterogeneous models. In particular, gate cycling does not restore the boundary phase (\Cref{app:phase-cycling}), leaving a unified explanation of boundary and within-window asymmetries unresolved. Whether explicit coordination across heads and layers can make retrieval quality more uniform across phases, without sacrificing average performance or compression efficiency, remains open.

Beyond controlled retrieval, our code-completion examples document periodic errors in DeepSeek-V4 and DeepSeek-V4.1. These observations motivate investigating whether phase sensitivity also affects robustness-relevant behavior. For example, do meaning-preserving shifts of a request within its context alter refusal or compliance, and do any such changes track the compression phase? We leave a systematic study of their potential robustness implications in real application scenarios to future work.

%% file: Appendix/contents.tex
\begingroup
\hypersetup{hidelinks}
\setlength{\parskip}{3pt}
\etocsettagdepth{main}{none}
\etocsettagdepth{appendix}{subsection}
\etocsettocstyle{%
  \pdfbookmark[0]{Appendix contents}{appendix-contents}%
}{\par}
\tableofcontents
\endgroup
\clearpage

%% file: Appendix/Sections/induction_head.tex
\section{\texorpdfstring{Deferred Proofs for Theorem~\ref*{thm:incoherent-hard} and Theorem~\ref*{thm:local-hardening}}{Deferred Proofs for Theorem 1 and Theorem 2}}
\label{app:full-vector-theory}

Here we state and prove formal versions of \Cref{thm:incoherent-hard} and \Cref{thm:local-hardening}. \Cref{thm:incoherent-hard-formal} shows that every minimizer over the oracle class uses hard adjacent selection, with each head placing all key weight on one phase and all value weight on the next. We call such a pair of key and value phases a route. The selected route may change across contexts and chunks. \Cref{thm:local-hardening-formal} shows that gradient flow on linear compressors from a small random initialization reaches with high probability hard adjacent selection with one route per head for all inputs, although different heads may share a route. We state the formal conditions more generally than in \Cref{sec:theory}. The assumptions $L\geq4$ and $\gamma\leq1$ made there imply $\gamma<\log3\leq\log(L-1)$, so the formal conditions hold in the main-text setting.

Both results turn on a tradeoff between averaging and selecting. A head that averages the positions of a chunk keeps some evidence about every position, whereas a head that selects one position keeps that position's evidence at full strength and gives up direct evidence about the others. Selection is favored when the gain at the selected position outweighs the loss elsewhere. A second question is whether several heads divide the positions among themselves, which we call complementary coverage, since sharp heads can serve as many positions as possible only if their selections differ.

In the general model, this tradeoff is hard to see directly. Each head has key and value coefficients over $S$ positions in every chunk and attends to chunks through its own softmax, all heads are coupled through the softmax over the vocabulary, and under gradient flow the coefficients come from gates that depend on the input embeddings. We would therefore like to see the tradeoff in a setting small enough to solve by hand before proving the general results. In \Cref{app:two-slot-theory}, we study a toy example with two heads and two phases, in which each head's compression is set by a single number and the loss is a function of two variables. There we can compute in closed form why a head sharpens its compression and why the two heads prefer to split the phases between them. The toy example is not a special case of the general theorems, since it has $S=2$, keeps only keys, and idealizes the value readout. We use it for intuition, and at its end we explain which parts of its argument carry over to the general model and which do not (\Cref{app:toy-to-general}).

We then return to the general model. In \Cref{sec:model}, we define the model and the two compressor classes. In \Cref{sec:oracle-minimizers}, we prove \Cref{thm:incoherent-hard-formal}, in which sharpening reappears for paired key and value coefficients. In \Cref{sec:gradient-flow}, we prove \Cref{thm:local-hardening-formal}, using the toy example's strategy of starting near a symmetric point and asking along which directions gradient flow leaves it.

\input{Appendix/Sections/two_slot_theory}

\subsection{Model and Compressor Classes}
\label{sec:model}

We now state the general model of \Cref{sec:theory} with full notation. We index the key and value branches by $P\in\{K,V\}$, which \Cref{sec:theory} writes as $B$.

For every positive integer $m$, write $[m]:=\{1,\ldots,m\}$.  Let
$S\geq3$, use $1\leq H\leq S-1$ heads, and use $L\geq2$ chunks.
The chunks are non-overlapping and each contains $S$ positions, so the
compression stride and effective compression ratio are both $S$.
Let $\mathcal V\subset\mathbb R^d$ be a vocabulary of $N$ distinct embedding
vectors, where $N\geq LS$.
A context $\mX$ consists of the $L$ ordered chunks
\begin{equation}
 \mX=\bigl((\vx_{\ell,1},\ldots,\vx_{\ell,S})\bigr)_{\ell=1}^L,
\end{equation}
where the $LS$ embeddings are sampled uniformly without replacement from
$\mathcal V$.  Independently,
$\ell_*\sim\operatorname{Unif}[L]$ and
$r_*\sim\operatorname{Unif}[S-1]$.  The query vector is
$\vq=\vx_{\ell_*,r_*}$ and the target embedding is
$\vy=\vx_{\ell_*,r_*+1}$.
Let $\mathcal X$ be the finite set of all such
ordered contexts with distinct embeddings.

For the first global theorem, we assume that for some $\kappa\in\mathbb R$
and $0\leq\varepsilon<1$,
\begin{equation}
 \left|\vx^\top \vx'-\kappa-\mathbf1\{\vx=\vx'\}\right|
 \leq\varepsilon
 \qquad(\vx,\vx'\in\mathcal V).
 \label{eq:incoherent-gram}
\end{equation}
This is the embedding geometry of \Cref{eq:main-incoherent-gram} in
\Cref{sec:theory}.  For fixed
$\varepsilon>0$, it permits dimension logarithmic in $N$.  The dynamics result
uses its idealized high-dimensional limit $\varepsilon=0$.

Let $\mE\in\mathbb R^{N\times d}$ have row $\vx^\top$ indexed by
$\vx\in\mathcal V$, and fix $0<\gamma,\beta<\infty$.

We consider two compressor classes.  In the oracle class, after observing
$\mX$ and before $(\ell_*,r_*)$ is sampled, a compressor rule
may compute arbitrary vectors
\begin{equation}
 \vp_{h,\ell}^{K}(\mX),\vp_{h,\ell}^{V}(\mX)\in\Delta_S,
 \qquad
 \Delta_S:=\left\{\vp\in\mathbb R^S:
 \vp(r)\geq0,\ \sum_{r=1}^S \vp(r)=1\right\}
\end{equation}
for every $\ell\in[L]$ and $h\in[H]$.  These key and value compression
coefficients may depend arbitrarily on $\mX$.  The class of all such pre-query
compressor rules is
\begin{equation}
 \Coracle
 =\left\{g:\mathcal X\longrightarrow
   (\Delta_S\times\Delta_S)^{LH}\right\},
 \qquad
 g(\mX)=\bigl(\vp_{h,\ell}^{K}(\mX),\vp_{h,\ell}^{V}(\mX)\bigr)_{\ell,h}.
\end{equation}
Thus the information order is $\mX\mapsto g(\mX)$, followed by
$(\ell_*,r_*)\sim\operatorname{Unif}([L]\times[S-1])$.
Here and below, $\ve_r$ is the $r$th vertex of $\Delta_S$.  When a context and
query are fixed, we suppress their dependence in the notation.

Position $r$ is window offset $r-1$ in \Cref{sec:qwen3-spec} and corresponds
to one phase.  Both coefficient vectors use physical phase indices.  No phase is masked, and the
useful adjacent route $r\in[S-1]$ pairs key phase $r$ with value phase $r+1$.
The compressed KV entries are
\begin{equation}
 \vk_{h,\ell}=\sum_{r=1}^{S}\vp_{h,\ell}^{K}(r)\vx_{\ell,r},
 \qquad
 \vv_{h,\ell}=\sum_{r=1}^{S}\vp_{h,\ell}^{V}(r)\vx_{\ell,r}.
\end{equation}
For query $\vq$, the attention weights are
\begin{equation}
 \alpha_{h,\ell}=
 \frac{\exp\{\gamma\langle\vq,\vk_{h,\ell}\rangle\}}
 {\sum_{j=1}^L\exp\{\gamma\langle\vq,\vk_{h,j}\rangle\}}.
\end{equation}
The attention output and logits are
\begin{equation}
 \mathsf{TF}(\mX,\vq)=\sum_{h=1}^{H}\sum_{\ell=1}^L
 \alpha_{h,\ell}\vv_{h,\ell},
 \qquad
 \vO(\mX,\vq)=\beta\mE\,\mathsf{TF}(\mX,\vq).
\end{equation}

For a fixed context $\mX$, let $\vp$ denote the full collection of oracle
coefficients and substitute it into this architecture.  Let
$\vO(\mX,\vq)_{\vx}$ denote the output coordinate indexed by
$\vx\in\mathcal V$.
The fixed-context objective is
\begin{equation}
 \cL_{\mX}(\vp)=\frac{1}{L(S-1)}
 \sum_{\ell_*=1}^L\sum_{r_*=1}^{S-1}
 \left[
  \log\sum_{\vx\in\mathcal V} e^{\vO(\mX,\vq)_{\vx}}
  -\vO(\mX,\vq)_{\vy}
 \right].
\end{equation}
Its population average is
\begin{equation}
 \cL(g)=\frac1{|\mathcal X|}\sum_{\mX\in\mathcal X}
 \cL_{\mX}(g(\mX)).
\end{equation}
The finite product $\Coracle$ is compact, and $\cL$ is continuous on it.

For the training-dynamics result, we restrict the oracle class to linear
compressors, whose gate logits are linear in the embeddings. They use
vectors $\vc_{h,r}^{K},\vc_{h,r}^{V}\in\mathbb R^d$ shared across contexts and
chunks, with no scalar biases. Write
$\vc=(\vc_{h,r}^{P})_{P,h,r}$.  For $P\in\{K,V\}$, its compression coefficients are
\begin{equation}
 \vp_{h,\ell}^{P}(r;\vc,\mX)
 =\frac{\exp\{\langle\vc_{h,r}^{P},\vx_{\ell,r}\rangle\}}
 {\sum_{j=1}^{S}\exp\{\langle\vc_{h,j}^{P},\vx_{\ell,j}\rangle\}}
\end{equation}
and
\begin{equation}
 \vp_{h,\ell}^{P}(\vc;\mX):=
 (\vp_{h,\ell}^{P}(r;\vc,\mX))_{r=1}^S,
 \qquad
 g_{\vc}(\mX)=\bigl(\vp_{h,\ell}^{K}(\vc;\mX),
 \vp_{h,\ell}^{V}(\vc;\mX)\bigr)_{\ell,h}.
\end{equation}
The resulting linear compressor class is
\begin{equation}
 \Clinear
 =\left\{g_{\vc}:
 \vc_{h,r}^{K},\vc_{h,r}^{V}\in\mathbb R^d\ \text{for all }h,r\right\}
 \subset\Coracle,
\end{equation}
and only the vectors $\vc_{h,r}^{K},\vc_{h,r}^{V}$ are trainable.  The inclusion
can be strict because an oracle rule may use the whole context in an
arbitrary way, whereas $g_{\vc}$ uses shared phasewise linear gate logits.

\subsection{\texorpdfstring{Proof of Theorem~\ref*{thm:incoherent-hard}}{Proof of Theorem 1}}
\label{sec:oracle-minimizers}

We first ask what the loss itself favors when compression may depend on the context in any way. As in the toy example, we compare mixed coefficients with one-hot ones, now with keys and values paired.

\paragraph{Symmetric comparison loss.}
We compare the true loss with an auxiliary functional of the same oracle
coefficients that isolates the symmetric part
$\kappa+\mathbf1\{\vx=\vx'\}$ of the Gram matrix.

For $0\leq u\leq1$, define
\begin{equation}
 R(u)=\frac{e^{\gamma u}}{e^{\gamma u}+L-1}.
\end{equation}
For a fixed context, query $(\ell_*,r_*)$, and coefficient collection $\vp$, set
\begin{equation}
 A_\ell^h=
 \begin{cases}
 R(\vp_{h,\ell_*}^{K}(r_*)),&\ell=\ell_*,\\
 \{1-R(\vp_{h,\ell_*}^{K}(r_*))\}/(L-1),&\ell\ne\ell_*.
 \end{cases}
\end{equation}
These are the comparison attention weights.  Using the same value compression coefficients as
the true model, define the unscaled comparison output coordinate at
$\vx\in\mathcal V$ by
\begin{equation}
 F_{\vx}=\sum_{h=1}^{H}\sum_{\ell=1}^L
 A_\ell^h\sum_{r=1}^{S}
 \vp_{h,\ell}^{V}(r)\,\mathbf 1\{\vx_{\ell,r}=\vx\}.
\end{equation}
Since both $A^h$ and $\vp_{h,\ell}^{V}$ have unit mass,
$\sum_{\vx\in\mathcal V}F_{\vx}=H$.  Vocabulary items absent from the
context have $F_{\vx}=0$.  Hence the comparison query loss is
$\log\sum_{\vx\in\mathcal V}e^{\beta F_{\vx}}-\beta F_{\vy}$.
Let $\cL_{{\rm sym},\mX}(\vp)$ be its average over the $L(S-1)$ queries, and
define its population average by
\begin{equation}
 \cL_{\rm sym}(g)=\frac1{|\mathcal X|}
 \sum_{\mX\in\mathcal X}\cL_{{\rm sym},\mX}(g(\mX)).
\end{equation}

When $\varepsilon=0$, the comparison functional equals the true loss.
Indeed, $\vx^\top\vx'=\kappa+\mathbf1\{\vx=\vx'\}$.  For query $\vq$,
the true attention scores are
$\gamma(\kappa+\vp_{h,\ell_*}^{K}(r_*))$ on chunk $\ell_*$ and
$\gamma\kappa$ on every other chunk, so their softmax is $A^h$.  The output
coordinate at $\vx\in\mathcal V$ is $\beta(\kappa H+F_{\vx})$, and the common term
$\beta\kappa H$ cancels from cross-entropy loss.  For $\varepsilon>0$, the proof below controls the change
from $\cL_{\rm sym}$ to $\cL$ along the same coefficient path.

For later use, set $R_1=R(1)$ and $Q_1=(1-R_1)/(L-1)$.

\begin{theorem}[Formal version of \Cref{thm:incoherent-hard}]
\label{thm:incoherent-hard-formal}
Under \Cref{eq:incoherent-gram}, there are constants
$c_{\gamma,\beta},C_{\gamma,\beta}>0$ such that the following holds.  Suppose
\begin{equation}
 N\geq C_{\gamma,\beta}LS
 \exp\left(\frac{C_{\gamma,\beta}H}{L}\right),
 \qquad
 \varepsilon\leq\frac{c_{\gamma,\beta}}{L^2HS}.
\end{equation}
Then every minimizer over $\Coracle$ uses hard adjacent
selection: for every $\mX\in\mathcal X$, $\ell\in[L]$, and $h\in[H]$, there is
$r_{\mX,\ell,h}\in[S-1]$ such that
\begin{equation}
 \vp_{h,\ell}^{K}(\mX)=\ve_{r_{\mX,\ell,h}},
 \qquad
 \vp_{h,\ell}^{V}(\mX)=\ve_{r_{\mX,\ell,h}+1}.
\end{equation}
If $\varepsilon=0$, the minimizers are exactly those assignments for which,
for every $\mX$ and $\ell$, the map $h\mapsto r_{\mX,\ell,h}$ is injective.
\end{theorem}

The first lemma rules out mixed coefficients, and the second compares the
remaining one-hot route assignments under exact geometry.

\begin{lemma}[Projection to adjacent vertices]
\label{lem:oracle-vertex-projection}
Under the assumptions of \Cref{thm:incoherent-hard-formal}, for each context
$\mX$ and coefficient collection $\vp$, there is a one-hot adjacent
collection with strictly smaller $\cL_{\mX}$ unless $\vp$ already has this form.
\end{lemma}

\begin{proof}
Fix a context and write $\E$ for its uniform-query average.  For each $(\ell,h)$,
choose $\hat r_{\ell,h}\in[S-1]$ maximizing
$\vp_{h,\ell}^{K}(\hat r)+\vp_{h,\ell}^{V}(\hat r+1)$, and replace the two coefficient
vectors by $\ve_{\hat r_{\ell,h}}$ and $\ve_{\hat r_{\ell,h}+1}$.  Denote the
resulting one-hot adjacent collection by $\widehat{\vp}$.  Suppressing
$(\ell,h)$, set
$\zeta=\min\{R_1,\min_{u\in[0,1]}R'(u)\}>0$.
The mean-value theorem and maximality of $\hat r$ give
\begin{align*}
 R_1-\sum_{r=1}^{S-1}R(\vp^K(r))\vp^V(r+1)
 &\geq \zeta\left(1-
   \sum_{r=1}^{S-1}\vp^K(r)\vp^V(r+1)\right)\\
 &\geq\frac\zeta4
 \left(\|\vp^K-\ve_{\hat r}\|_1+\|\vp^V-\ve_{\hat r+1}\|_1\right).
\end{align*}
The last step uses $uv\leq(u+v)^2/4$ and the fact that a simplex
point's distance to a vertex is twice its missing vertex mass.

Let the query-normalized projection distance be
\begin{equation}
 \Delta=\frac1{L(S-1)}\sum_{\ell,h}
 \left(\|\vp_{h,\ell}^{K}-\widehat{\vp}_{h,\ell}^{K}\|_1
 +\|\vp_{h,\ell}^{V}-\widehat{\vp}_{h,\ell}^{V}\|_1\right).
\end{equation}
Averaging this per-head gap over the uniform query gives
\begin{equation}
 \E[\widehat F_{\vy}-F_{\vy}]\geq\frac\zeta4\Delta.
 \label{eq:compact-target-gain}
\end{equation}
Inserting the output with projected attention but original values,
using $\E\widehat A_\ell^h=1/L$, and applying the Lipschitz bound for $R$
give
\begin{equation}
 \E\|F-\widehat F\|_1
 \leq \max\{S-1,2\max_{u\in[0,1]}R'(u)\}\Delta.
 \label{eq:compact-logit-change}
\end{equation}

Both outputs have mass $H$ and entries in $[0,HR_1]$.  Along the segment
between them, subtracting $1$ from each exponential in the derivative of
the log-normalizer bounds its change by
$\beta(e^{\beta HR_1}-1)\E\|F-\widehat F\|_1/(2N)$.
Hence \Cref{eq:compact-target-gain,eq:compact-logit-change} give
\begin{equation}
 \cL_{{\rm sym},\mX}(\vp)-\cL_{{\rm sym},\mX}(\widehat{\vp})
 \geq\beta\Delta\left\{\frac\zeta4-
 \frac{\max\{S-1,2\max_{u\in[0,1]}R'(u)\}
 (e^{\beta HR_1}-1)}{2N}\right\}.
 \label{eq:compact-comparison-gap}
\end{equation}
Here $\zeta\geq\gamma/[L(e^\gamma+1)]$ and $R_1\leq e^\gamma/L$.
The theorem's vocabulary bound makes the braces at least $\zeta/8$.

It remains to transfer this gap to the true Gram matrix.  The softmax
Jacobian $J(q)=\operatorname{diag}(q)-qq^\top$ satisfies
$\|J(q)v\|_1\leq2\|v\|_\infty$ and
$\|(J(q)-J(q'))v\|_1\leq3\|q-q'\|_1\|v\|_\infty$.
By \Cref{eq:incoherent-gram}, each attention score changes by at most
$\gamma\varepsilon$, and its path derivative by at most
$\gamma\varepsilon\sum_\ell\|\dot\vp_{h,\ell}^K\|_1$.
Applying the Jacobian bounds to attention and then to output
cross-entropy gives, for any differentiable coefficient path,
\begin{equation}
 \left|\frac{\mathrm d}{\mathrm dt}
 (\cL_{\mX}-\cL_{{\rm sym},\mX})(\vp(t))\right|
 \leq C_{\gamma,\beta}(1+\beta H)\varepsilon
 \sum_{\ell,h,P}\|\dot\vp_{h,\ell}^{P}(t)\|_1.
 \label{eq:compact-loss-stability}
\end{equation}
Indeed, after subtracting the common $\beta\kappa H$ from the output
logits, their true--comparison difference and its path derivative are
bounded respectively by
$O_{\gamma,\beta}(H\varepsilon)$ and
$O_{\gamma,\beta}(\varepsilon\sum_{\ell,h,P}
\|\dot\vp_{h,\ell}^{P}\|_1)$.  The cross-entropy gradient is
$2$-bounded in $\ell_1$ and $2$-Lipschitz from $\ell_\infty$ to $\ell_1$.
Integrating \Cref{eq:compact-loss-stability} along
$\vp(t)=(1-t)\widehat{\vp}+t\vp$ changes the comparison gap by at most
$C_{\gamma,\beta}L(S-1)(1+\beta H)\varepsilon\Delta$.
The theorem's bound on $\varepsilon$ makes this less than
$\beta\zeta\Delta/16$.  With \Cref{eq:compact-comparison-gap}, the true
gap is positive whenever $\Delta>0$.  An oracle can make the replacement
for one context without changing any other context.
\end{proof}

\begin{lemma}[Distinct routes under exact geometry]
\label{lem:oracle-distinct-routes}
Under the assumptions of \Cref{thm:incoherent-hard-formal} with
$\varepsilon=0$, all injective one-hot adjacent assignments have the same
fixed-context loss, and every assignment with a head collision has strictly
larger loss.
\end{lemma}

\begin{proof}
Fix a context, write $r_{\ell,h}\in[S-1]$ for the
route of head $h$ in chunk $\ell$, and let
$k=\sum_h\mathbf1\{r_{\ell_*,h}=r_*\}$ for a uniform query.  Then
\begin{equation}
 \E k=\frac{H}{S-1}\leq1,
 \qquad \E k^2\leq H.
\end{equation}
Every assignment has the same expected target logit
$\E F_{\vy}=HR_1/(S-1)$, so only the output normalizer matters.
Write $Z_{\rm inj}$ for this normalizer under an injection.  Then
$Z_{\rm inj}=Z_0+k d_{\rm inj}$, where
\begin{equation}
 Z_0:=N-LH+LH e^{\beta/L},
 \qquad
 d_{\rm inj}:=e^{\beta R_1}+(L-1)e^{\beta Q_1}-Le^{\beta/L}>0.
\end{equation}
The $L$ numbers $R_1,Q_1,\ldots,Q_1$ have mean $1/L$ and are not all
equal, so the last inequality follows from strict convexity.  Thus all
injections have the same expected target and normalizer terms, and
\begin{equation}
\E[\log Z_{\rm inj}]
 \leq\log Z_0+\frac{d_{\rm inj}H}{(S-1)Z_0}.
\end{equation}

For a collided assignment, contributions to a shared output coordinate add
before exponentiation.  Writing
$e^{\beta A}=1+(e^{\beta A}-1)$, the pair product from any collision
adds at least $d_{\rm col}:=(e^{\beta Q_1}-1)^2>0$ to the
normalizer, so $Z_{\rm col}\geq Z_0+k d_{\rm inj}+d_{\rm col}$.
Using $\log(1+t)\geq t-t^2/2$ and $\E k^2\leq H$ gives
\begin{equation}
 \E[\log Z_{\rm col}]-\E[\log Z_{\rm inj}]
 \geq\frac{d_{\rm col}}{Z_0}
 -\frac{d_{\rm inj}^2H+2d_{\rm inj}d_{\rm col}
 +d_{\rm col}^2}{2Z_0^2}.
\end{equation}
Taylor's theorem at $1/L$ and $R_1=e^\gamma Q_1$ give
$d_{\rm inj}\leq C_{\gamma,\beta}Q_1^2$, whereas
$d_{\rm col}\geq\beta^2Q_1^2$ and $d_{\rm col}\leq(e^\beta-1)^2$.
The numerator in the error term is therefore at most
$C_{\gamma,\beta}H d_{\rm col}$.  Since $Z_0\geq N$ and the
vocabulary bound can be enlarged to ensure $N>C_{\gamma,\beta}H$,
the displayed gap is positive.
The expected target term is the same for both assignments, so every
collision strictly increases the fixed-context loss.
\end{proof}

\begin{proof}[Proof of \Cref{thm:incoherent-hard-formal}]
The oracle class is a product over contexts, so each fixed-context loss can be
minimized separately.  \Cref{lem:oracle-vertex-projection} makes every minimizer
one-hot and adjacent.  When $\varepsilon=0$,
\Cref{lem:oracle-distinct-routes} shows that precisely the injective route
assignments minimize the loss among them.
\end{proof}

The injection may depend on the context and chunk, so the oracle result does
not assign a fixed route to each named head.

\subsection{\texorpdfstring{Proof of Theorem~\ref*{thm:local-hardening}}{Proof of Theorem 2}}
\label{sec:gradient-flow}

The oracle result says which compressors are optimal but not whether training finds them, or whether a trained head keeps the same selection for every input. As in the toy example, we answer this by starting gradient flow near zero, where all heads agree and no useful phase is preferred over another, and asking along which directions the flow leaves the symmetric path that starts there.

Let $\vmu:=N^{-1}\sum_{\vx\in\mathcal V}\vx$.  We first collect the
assumptions for the dynamics result.  Its effective parameter space is
\begin{equation}
 \mathcal C_{\rm eff}:=\left\{\vc:
 \vc_{h,r}^{P}\in\operatorname{span}(\mathcal V),\quad
 \sum_{r=1}^S\vmu^\top\vc_{h,r}^{P}=0
 \ \text{for every }P\in\{K,V\},\ h\in[H]\right\}.
 \label{eq:effective-space}
\end{equation}
This removes embedding-orthogonal directions and a common mean-logit shift
in each head and branch, which do not change the gates.

\begin{assumption}[Setting for the gradient-flow result]
\label{ass:gradient-flow-setting}
\leavevmode
\begin{enumerate}
\item \textbf{Retrieval distribution.} We use the model of
\Cref{sec:model}, with $N\geq LS$ vocabulary embeddings.  Contexts contain
$L$ non-overlapping chunks of $S$ distinct tokens each, sampled uniformly
without replacement across the entire context.  After compression, the
query is uniform over the $L(S-1)$ nonfinal positions and requests the
next token in the same chunk.
\item \textbf{Fixed model parameters.} We fix $L\geq3$, $S\geq3$,
$1\leq H\leq S-1$, $0<\gamma<\log(L-1)$, and $\beta>0$ as $N$ varies.
\item \textbf{Exact geometry and positive mean.} The embeddings satisfy
\Cref{eq:incoherent-gram} with $\varepsilon=0$, and the mean norm
$\|\vmu\|>0$ is fixed as $N$ varies.
\item \textbf{Gate-only population training.} Only the offset-specific
linear gate vectors in $\Clinear$, shared across contexts and chunks, are
trained; the embeddings and all other maps remain fixed as in
\Cref{sec:model}.  We use Euclidean gradient flow on the population loss:
\begin{equation}
 \dot{\vc}=-\nabla_{\vc}\{\cL(g_{\vc})/\beta\}.
 \label{eq:parameter-flow}
\end{equation}
The factor $1/\beta$ only rescales time.
\item \textbf{Initialization.} For $\rho>0$, $\vc(0)$ is uniform with
respect to Euclidean volume on the radius-$\rho$ ball centered at zero in
$\mathcal C_{\rm eff}$.  The required joint bounds on $\rho$ and $N$ are
specified in \Cref{thm:local-hardening-formal}.
\end{enumerate}
\end{assumption}

This initialization includes input-dependent gate directions.  The proof
below establishes the emergence of static preferences from these starts.
Under exact geometry, averaging the Gram identity gives
$\vx^\top\vmu=\|\vmu\|^2=\kappa+1/N$, and hence
\begin{equation}
 \vmu^\top(\vx-\vmu)=0,
 \qquad
 (\vx-\vmu)^\top(\vx'-\vmu)=\mathbf1\{\vx=\vx'\}-\frac1N
 \quad(\vx,\vx'\in\mathcal V).
\end{equation}
The exact Gram matrix has eigenvalues $1$ on vocabulary contrasts and
$N\|\vmu\|^2>0$ on the constant direction, so its rank is $N$ and this
idealization requires $d\geq N$.
The ambient gradient is tangent to $\mathcal C_{\rm eff}$: each parameter gradient is
a linear combination of embeddings, and the gate-logit gradients in
each row sum to zero.  Since
$\|\vx\|^2=\|\vmu\|^2+1-1/N$ for $\vx\in\mathcal V$, the vector field in
\Cref{eq:parameter-flow} is smooth and globally bounded for each $N$.
Its trajectories therefore exist for all positive time.  Below,
$\vp_{h,\ell}^{P}(r;t,\mX):=
\vp_{h,\ell}^{P}(r;\vc(t),\mX)$.

\begin{theorem}[Formal version of \Cref{thm:local-hardening}]
\label{thm:local-hardening-formal}
Under \Cref{ass:gradient-flow-setting}, for every $0<\delta<1$,
there are $0<c_\delta\leq1$ and
$C_\delta<\infty$ such that the following holds.  If
\begin{equation}
 0<\rho\leq c_\delta,
 \qquad
 N\geq C_\delta\log\frac e\rho,
\end{equation}
then, with probability at least $1-\delta$, there
are $r_1,\ldots,r_H\in[S-1]$ such that
\begin{equation}
 \max_{h\in[H]}\sup_{\mX\in\mathcal X,\,\ell\in[L]}
 \max\left\{
 1-\vp_{h,\ell}^{K}(r_h;t,\mX),
 1-\vp_{h,\ell}^{V}(r_h+1;t,\mX)
 \right\}\longrightarrow0\quad\text{as }t\to\infty.
\end{equation}
The routes $r_h$ need not be distinct.  The constants $c_\delta,C_\delta$
can be chosen uniformly for $\|\vmu\|$ in any compact subset of $(0,\infty)$.
Hence fixing $\kappa>0$ suffices, since then $\|\vmu\|^2=\kappa+1/N\in(\kappa,\kappa+1]$.
\end{theorem}

The lemmas below track the dynamics from the symmetric zero-start path to
aligned winners and then to persistent one-hot routing.  Winners are aligned
when each head's key gate favors some phase $r_h\in[S-1]$ and its value gate
favors the next phase $r_h+1$.

In the supporting lemmas, $C$ denotes a positive constant depending only on
$L,S,H,\gamma,\beta$, and the fixed value of $\|\vmu\|$, and its value may
change between displays.

Let $\mathcal M$ be the mean-aligned subspace
\begin{equation}
 \vc_{h,r}^{P}=\frac{b_{h,r}^{P}}{\|\vmu\|^2}\vmu,
 \qquad
 \sum_{r=1}^S b_{h,r}^{P}=0,
 \qquad P\in\{K,V\}.
 \label{eq:mean-space}
\end{equation}
Since $\vmu^\top\vx=\|\vmu\|^2$ for $\vx\in\mathcal V$, $b_{h,r}^{P}$ is the
input-independent gate logit of phase $r$.  Write $P_{\mathcal M}$ for the
orthogonal projection onto $\mathcal M$ and
$\mathcal T:=\mathcal M^\perp\cap\mathcal C_{\rm eff}$.

The linear class raises a question that the toy example cannot pose. There, each head's compression is a single trainable logit with no input, so any preference is automatically the same for every input. With linear gates, the gate logits depend on the embeddings, and a preference could change from one context to another. On $\mathcal M$, however, the gate logit $b^{P}_{h,r}$ of each phase is the same for every input, just as $\theta_h$ is in the toy example, so we track the flow near $\mathcal M$. The flow from zero stays on $\mathcal M$ with all heads identical, which plays the role of the toy example's uniform point (\Cref{lem:symmetric-baseline}). Differences between routes are unstable along this path, and the unstable directions raise a key phase together with the value phase after it. As in the toy example, selection thus comes from an instability of the symmetric state, but here the instability acts on each head separately, since different heads do not couple in the limit. From a small random start, each head acquires such a route with high probability (\Cref{lem:mean-entrance}), which is then amplified and persists for every input (\Cref{lem:mean-amplification,lem:terminal-cone}).

\begin{lemma}[Vocabulary suppression]
\label{lem:vocabulary-suppression}
For fixed context and query, the first three derivatives of the output
log-normalizer divided by $\beta$ are $O(N^{-1})$, uniformly over
$(\Delta_S\times\Delta_S)^{LH}$.  This holds for derivatives in the
compression coefficients, sampled gate logits
$\langle\vc_{h,r}^{P},\vx_{\ell,r}\rangle$, and compressor parameters,
using Euclidean product norms.  In particular, its contributions to
coefficient gradients and the parameter linearization are $O(N^{-1})$.
\end{lemma}

\begin{proof}
When $\varepsilon=0$, the common vocabulary logit cancels and the remaining
logits are $\beta F_{\vx}$, with $F_{\vx}$ as in
\Cref{sec:oracle-minimizers} after substituting the softmax coefficients.
Thus $F_{\vx}=0$ outside the context,
$0\leq F_{\vx}\leq H$, and the normalized log-normalizer is
\begin{equation}
 \frac1\beta\log\left\{
 N+\sum_{\ell=1}^L\sum_{r=1}^S
 (e^{\beta F_{\vx_{\ell,r}}}-1)\right\}.
 \label{eq:shifted-normalizer}
\end{equation}
The first three derivatives of $F_{\vx}$ in sampled gate logits or
compression coefficients are uniformly bounded because they use
fixed-dimensional softmax maps.  The nonconstant part of the expression
inside the logarithm, and its first three derivatives, are bounded
independently of $N$, whereas that expression is at least $N$.
Differentiation gives the $O(N^{-1})$ bounds.  The sampled-logit map is
linear in the compressor parameters and has operator norm at most
$C\sup_{\vx\in\mathcal V}\|\vx\|\leq C$, yielding the parameter bound.
\end{proof}

Define the rowwise route-contrast space
\begin{equation}
 \mathcal W:=\left\{U\in\mathbb R^{H\times(S-1)}:
 \sum_{r=1}^{S-1}U_{h,r}=0\ \text{for every }h\right\}.
 \label{eq:route-contrast-space}
\end{equation}
Let $\mathcal F(\vc)=-\nabla_{\vc}\{\cL(g_{\vc})/\beta\}$, let
$\vc_N^0(t)$ be its trajectory from zero, and let $A_N(t)$ be the
derivative on $\mathcal M$ of the target part of $\mathcal F$ at
$\vc_N^0(t)$.

\begin{lemma}[Symmetric baseline and route-contrast growth]
\label{lem:symmetric-baseline}
For sufficiently large $N$, the trajectory $\vc_N^0(t)$ remains in
$\mathcal M$.  All heads agree, and key phases $1,\ldots,S-1$ and value
phases $2,\ldots,S$ are respectively identical.  Call key phase $S$ and value
phase $1$ the dummy phases, since no useful route uses them and they serve no
query.  If $\omega_K,\omega_V$ denote the dummy-phase coefficients, then
\begin{equation}
 \omega_K(t),\omega_V(t)\asymp(1+t)^{-1},\qquad
 \log\frac{1-\omega_P(t)}{(S-1)\omega_P(t)}\longrightarrow\infty
 \quad(P\in\{K,V\}).
 \label{eq:compact-dummy-decay}
\end{equation}
Moreover, $A_N(t)$ converges to an operator $A_\infty$ that acts on each
route contrast, in the order key then successor value, by $\|\vmu\|^2$
times the block
\begin{equation}
 \begin{pmatrix}
 \dfrac{R''(1/(S-1))}{(S-1)^4}&
 \dfrac{R'(1/(S-1))}{(S-1)^3}\\[6pt]
 \dfrac{R'(1/(S-1))}{(S-1)^3}&0
 \end{pmatrix},
 \label{eq:late-contrast-block}
\end{equation}
and vanishes on the useful-versus-dummy directions, with
\begin{equation}
 \|A_N(t)-A_\infty\|\leq \frac C{1+t}.
\end{equation}
The positive eigenspace of $A_\infty$ is a copy of $\mathcal W$, with
same-sign key and successor-value coordinates.
\end{lemma}

\begin{proof}
Vocabulary-permutation, head, and useful-phase symmetries keep
$\vc_N^0(t)$ in $\mathcal M$ and give useful coefficients
$(1-\omega_K)/(S-1)$ and $(1-\omega_V)/(S-1)$.  For either row,
\begin{equation}
 \dot\omega_P=-\|\vmu\|^2\frac S{S-1}d_P(t)\omega_P^2(1-\omega_P)^2,
 \qquad P\in\{K,V\},
\end{equation}
where $d_P(t)$ is the useful-minus-dummy fitness gap for $\cL/\beta$.  Here
the fitness of a phase is the negative derivative of $\cL/\beta$ with respect
to its compression coefficient, and the fitness gap is the fitness of a useful
phase minus that of the dummy phase in the same row.
The target term gives $d_V\geq1/[L(S-1)]$.  Once $\omega_V\leq1/S$,
it also gives $d_K\geq R'(0)/[S(S-1)]$.  By
\Cref{lem:vocabulary-suppression}, these lower bounds survive for
large $N$, and both gaps are also uniformly bounded above.
Integrating the displayed equation from $\omega_P(0)=1/S$ proves
\Cref{eq:compact-dummy-decay}.

Route symmetry separates useful contrasts from useful--dummy directions.
At the limiting coefficients, softmax sends a useful contrast $y$ to
$y/(S-1)$, and its second derivatives sum to zero.  The target
Hessian on key and successor-value contrasts is therefore
\begin{equation}
 \begin{aligned}
 &D^2\!\left\{\frac1{S-1}\sum_{r=1}^{S-1}
 R(\vp_h^{K}(r))\vp_h^{V}(r+1)\right\}[(y_K,y_V)]^2\\
 &\qquad=\frac{R''(1/(S-1))\|y_K\|^2}{(S-1)^4}
 +\frac{2R'(1/(S-1))\langle y_K,y_V\rangle}{(S-1)^3}.
 \end{aligned}
 \label{eq:limiting-contrast-form}
\end{equation}
This gives \Cref{eq:late-contrast-block}.  The useful--dummy
directions vanish in the limit because the dummy coefficient vanishes, and
different heads do not couple.  The block has negative determinant,
and its positive eigenvector has same-sign coordinates.  Smoothness
and $\omega_P(t)=O((1+t)^{-1})$ give the stated convergence rate.
\end{proof}

\begin{lemma}[Aligned route gaps from random initialization]
\label{lem:mean-entrance}
For every $0<\delta<1$, there are
$\eta_\delta,C_\delta>0$ and $0<c_\delta\leq1$ such that, if
$0<\rho\leq c_\delta$ and
$N\geq C_\delta\log(e/\rho)$, and if $\vc(0)$ is uniform on the
radius-$\rho$ ball in $\mathcal C_{\rm eff}$, then, with probability at least
$1-\delta$, there are a time
\begin{equation}
 \tau\leq C_\delta\log\frac{e\sqrt N}{\rho}
\end{equation}
and a point $\bar{\vc}\in\mathcal M$ satisfying
\begin{equation}
 \|\vc(\tau)-\bar{\vc}\|\leq C_\delta\rho,
 \label{eq:spectral-entrance-distance}
\end{equation}
such that every head has a route $r_h\in[S-1]$ with
\begin{equation}
 b_{h,r_h}^{K}-\max_{j\ne r_h}b_{h,j}^{K}\geq2\eta_\delta,
 \qquad
 b_{h,r_h+1}^{V}-\max_{j\ne r_h+1}b_{h,j}^{V}\geq2\eta_\delta.
 \label{eq:entrance-gaps}
\end{equation}
\end{lemma}

\begin{proof}
The proof first splits the flow linearized along $\vc_N^0(t)$ into growing
and complementary parts, and then shows that deviations orthogonal to
$\mathcal M$ stay small.  It next bounds the random initialization, defines an
exit time $\tau$ from the linearized flow, and compares the nonlinear flow with
its linearization up to $\tau$.  At that time, the growing part carries
aligned route gaps.

The target is a sum over heads, and simultaneous permutation of useful key
and successor-value indices separates useful contrasts from useful--dummy
directions.  Thus, in the $b$-coordinates, $A_N(t)$ is
$(B_N(t)\otimes I_{\mathcal W})\oplus D_N(t)$, where $B_N(t)$ is
two-dimensional, $B_\infty$ is the block in
\Cref{eq:late-contrast-block}, including its factor $\|\vmu\|^2$, and
$\|B_N(t)-B_\infty\|+\|D_N(t)\|\leq C/(1+t)$.
Let $\lambda_+>0>\lambda_-$ be the eigenvalues of $B_\infty$ and put
$\Delta_\lambda=\lambda_+-\lambda_-$.  In an orthonormal eigenbasis, write
$B_N(t)=\operatorname{diag}(\lambda_+,\lambda_-)+E_N(t)$.
Choose $T$, uniformly in large $N$, so that $\|E_N(t)\|,\|D_N(t)\|\leq\eta_*$
for $t\geq T$, with $\eta_*$ sufficiently small.  Two invariant contrast
lines are $x_-=r(t)x_+$ and $x_+=z(t)x_-$, defined by
\[
 \begin{aligned}
 \dot r&=-\Delta_\lambda r+E_{21}+(E_{22}-E_{11})r-E_{12}r^2,\qquad r(T)=0,\\
 z(t)&=-\int_t^\infty e^{-\Delta_\lambda(u-t)}
 [E_{12}+(E_{11}-E_{22})z-E_{21}z^2](u)\,\mathrm du.
 \end{aligned}
\]
For $\zeta=4\eta_*/\Delta_\lambda<1/2$, the first vector field points inward at
$r=\pm\zeta$, and the second equation is a contraction on $\|z\|_\infty\leq\zeta$.
Hence these lines are uniformly transverse.  The scalar convolution in
the first equation and $E_N(t)\to0$ give $r(t)\to0$, uniformly in large $N$.
Their scalar rates are $\lambda_++E_{11}+E_{12}r$ and
$\lambda_-+E_{22}+E_{21}z$, respectively.  Taking $\eta_*$ small, the
former is at least $a>0$, the latter is negative, and $\eta_*\leq a/2$.
Use these same lines for every head and route contrast, placing the whole
useful--dummy sector in the complement.  This defines invariant projections
$P_+(t)$ and $P_\perp(t)=I-P_+(t)$ with uniformly bounded norms.
Extend them to $[0,T]$ by the propagator $\Phi(t,s)$ of $\dot u=A_N(t)u$;
boundedness of $A_N$ on this fixed interval bounds both transport directions
uniformly in $N$.  Write $\Phi_j(t,s)=P_j(t)\Phi(t,s)P_j(s)$ for
$j\in\{+,\perp\}$.  The preceding scalar rates and
$\|D_N(t)\|\leq\eta_*$ give, after adjusting $C$, for $t\geq s\geq0$,
\begin{equation}
 \begin{aligned}
 \|\Phi_+(t,s)^{-1}\|&\leq Ce^{-a(t-s)},\\
 \|\Phi_+(t,s)\|&\leq Ce^{C(t-s)},\\
 \|\Phi_\perp(t,s)\|&\leq Ce^{a(t-s)/2}.
 \end{aligned}
 \label{eq:dichotomy-bounds}
\end{equation}
The inverse is on the growing fiber; the zero limiting modes are included
in the complementary bound.  All constants are uniform in large $N$.

In the original key/value coordinates, choose the positive unit eigenvector
$e_+$ of $B_\infty$ and let $Q_N(t)$
be the contrast solution with $Q_N(T)=e_+$.  The growing fiber is
$Q_N(t)\otimes\mathcal W$ and $Q_N(t)/\|Q_N(t)\|\to e_+$.
Let $\alpha_N$ take the coefficient of $Q_N(0)$ in the splitting along
the complementary contrast line, and define $\Lambda_N\vc$ by applying
$\alpha_N\otimes I_{\mathcal W}$ to the useful-contrast part of
$P_{\mathcal M}\vc$, expressed in $b$-coordinates.  Fixed-interval transport
bounds $\|Q_N(0)\|$ above and away from zero; transversality and
$\alpha_N Q_N(0)=1$ do the same for $\|\alpha_N\|$.
The useful-contrast projection is an orthogonal surjection, and
$\|b\|=\|\vmu\|\|\vc\|$ on $\mathcal M$, so all nonzero singular values of
$\Lambda_N$ are uniformly bounded above and below.  In particular, every
growing solution is $Q_N(t)\otimes w$ for a fixed $w\in\mathcal W$.

To control deviations from the mean-aligned trajectory, for a parameter
direction $w$ define
\begin{equation}
 \|w\|_{\mX}^2=\sum_{P,h,\ell,r}
 \bigl((w_{h,r}^{P})^\top\vx_{\ell,r}\bigr)^2.
\end{equation}
The $\varepsilon=0$ geometry and sampling without replacement imply
\begin{equation}
 \E_{\mX}\|w\|_{\mX}^2=\frac LN\|w\|^2\quad(w\in\mathcal T),
 \qquad
 \|w\|_{\mX}=\sqrt L\,\|\vmu\|\,\|w\|\quad(w\in\mathcal M).
 \label{eq:compact-score-moments}
\end{equation}
The fixed-context objective is smooth in the sampled gate logits.  Its
target derivatives are bounded, and its normalizer derivatives are
$O(N^{-1})$ by \Cref{lem:vocabulary-suppression}.  Vocabulary
equivariance gives
$\mathcal F(\bar{\vc})\in\mathcal M$ and
$P_{\mathcal M}D\mathcal F(\bar{\vc})[u]=0$ for
$\bar{\vc}\in\mathcal M$, $u\in\mathcal T$.
Applying \Cref{eq:compact-score-moments} to the first two derivatives
and using this cancellation gives, for $u\in\mathcal T$ with $\|u\|\leq1$,
\begin{equation}
 \|(I-P_{\mathcal M})\mathcal F(\bar{\vc}+u)\|
 \leq\frac CN\|u\|,
 \qquad
 \|P_{\mathcal M}\mathcal F(\bar{\vc}+u)-\mathcal F(\bar{\vc})\|
 \leq\frac CN\|u\|^2.
 \label{eq:compact-nonlinear-averaging}
\end{equation}

The effective parameter space has dimension $2H(SN-1)$.  Uniform
radius-$\rho$ ball initialization is an isotropic Gaussian direction
times an independent radial factor.  Gaussian concentration, the
uniform singular-value bounds for $\Lambda_N$, and a finite union
bound over the nondegenerate Gaussian head--route contrasts give
$a_\delta,A_\delta,q_\delta>0$ such
that, with probability at least $1-\delta$,
\begin{equation}
 a_\delta\frac\rho{\sqrt N}\leq
 g_0:=\|\Lambda_N\vc(0)\|\leq A_\delta\frac\rho{\sqrt N},
 \quad \|P_{\mathcal M}\vc(0)\|\leq A_\delta\frac\rho{\sqrt N},
 \quad \min_h\operatorname{gap}((\Lambda_N\vc(0))_h)
 \geq q_\delta g_0.
 \label{eq:initial-norm-event}
\end{equation}
Here $\operatorname{gap}$ is the largest minus second-largest row entry.

On $\mathcal M$, use the Euclidean norm of the $b$-array, which is uniformly
equivalent to the parameter norm because $\|b\|=\|\vmu\|\|\vc\|$
there.  Let $u_{\rm lin}$ solve $\dot u=A_N(t)u$ from
$u_{\rm lin}(0)=P_{\mathcal M}\vc(0)$.  The estimates below hold for all
exit radii in a fixed sufficiently small interval.  Choose $r_\delta$
in that interval below, and let $\tau$ be the first time
$\|P_+(\tau)u_{\rm lin}(\tau)\|=r_\delta$.  From
\Cref{eq:dichotomy-bounds,eq:initial-norm-event},
\begin{equation}
 A_\delta^{-1}\log\frac{r_\delta\sqrt N}{\rho}
 \leq\tau\leq
 A_\delta\log\frac{r_\delta\sqrt N}{\rho}.
 \label{eq:exit-time}
\end{equation}
The lemma's lower bound on $N$, after increasing $C_\delta$, makes
$\tau/N$ arbitrarily small.

Define
\begin{equation}
 u_{\mathcal M}(t):=P_{\mathcal M}(\vc(t)-\vc_N^0(t)),
 \qquad
 u_{\mathcal T}(t):=(I-P_{\mathcal M})\vc(t).
\end{equation}
The first bound in \Cref{eq:compact-nonlinear-averaging} and Gronwall give
\begin{equation}
 \|u_{\mathcal T}(t)\|\leq A_\delta\rho
 \qquad(0\leq t\leq\tau).
 \label{eq:token-tube}
\end{equation}
Shrinking $c_\delta$ closes the assumed unit token tube.  Taylor's theorem,
\Cref{lem:vocabulary-suppression}, and the second bound in
\Cref{eq:compact-nonlinear-averaging} give
\begin{equation}
 \dot u_{\mathcal M}=A_N(t)u_{\mathcal M}+E(t),\qquad
 \|E(t)\|\leq C\left(
 \frac{\|u_{\mathcal M}\|}{N}+\|u_{\mathcal M}\|^2+
 \frac{\|u_{\mathcal T}\|^2}{N}\right).
 \label{eq:projected-perturbation}
\end{equation}

Put $v(t)=P_+(t)u_{\rm lin}(t)$.  By
\Cref{eq:dichotomy-bounds,eq:initial-norm-event} and scalarity on
$\mathcal W$, for $0\leq s\leq t\leq\tau$,
\begin{equation}
 \begin{gathered}
 \|v(s)\|\leq Ce^{-a(t-s)}\|v(t)\|,\qquad
 \|\Phi_+(t,s)\|\|v(s)\|\leq C\|v(t)\|,\\
 \|P_\perp(t)u_{\rm lin}(t)\|
 \leq A_\delta e^{-a t/2}\|v(t)\|.
 \end{gathered}
 \label{eq:linear-exit-bounds}
\end{equation}
In particular, $\|v(t)\|\leq Cr_\delta
e^{-a(\tau-t)}$ before exit.  The final bound is the
relative decay: the inverse growing propagator contributes
$e^{-a t}$ and the complementary propagator contributes $e^{a t/2}$.

Variation of constants for \Cref{eq:projected-perturbation} gives
$u_{\mathcal M}(t)=u_{\rm lin}(t)+
\int_0^t\Phi(t,s)E(s)\,\mathrm ds$.
Scalarity, bounded projections, and \Cref{eq:linear-exit-bounds} give
the full-propagator bounds, for $0\leq s\leq t\leq\tau$,
\[
 \|\Phi(t,s)\|\|v(s)\|\leq C\|v(t)\|,\qquad
 \|\Phi(t,s)\|\leq C\frac{\|v(t)\|}{g_0}e^{-as}.
\]
Indeed, the complementary contribution to the first bound is at most
$Ce^{-a(t-s)/2}\|v(t)\|$, and $\|v(s)\|\geq C^{-1}g_0e^{as}$
gives the second.  On the bootstrap bound
$\|u_{\mathcal M}(s)\|\lesssim_\delta
\|u_{\rm lin}(s)\|+\rho^2\|v(s)\|/(Ng_0)$, the right-hand side is
$O_\delta(\|v(s)\|)$ when $\rho^2/(Ng_0)\leq1$.
Together with $\int_0^t\|v(s)\|\,\mathrm ds\leq C\|v(t)\|\leq Cr_\delta$
and \Cref{eq:token-tube}, these kernel bounds yield, uniformly for $t\leq\tau$,
\begin{equation}
 \begin{aligned}
 \int_0^t\|\Phi(t,s)\|\|u_{\mathcal M}(s)\|^2\,\mathrm ds
 &\lesssim_\delta r_\delta\|v(t)\|,\\
 \frac1N\int_0^t\|\Phi(t,s)\|\|u_{\mathcal M}(s)\|\,\mathrm ds
 &\lesssim_\delta\frac tN\|v(t)\|,\\
 \frac1N\int_0^t\|\Phi(t,s)\|\|u_{\mathcal T}(s)\|^2\,\mathrm ds
 &\lesssim_\delta\frac{\rho^2}{Ng_0}\|v(t)\|.
 \end{aligned}
 \label{eq:compact-duhamel-bounds}
\end{equation}
For small $r_\delta$, $\tau/N$, and $\rho^2/(Ng_0)$, these estimates
improve the bootstrap, so continuity closes it.  Since
$g_0\geq a_\delta\rho/\sqrt N$, the last term at $t=\tau$ is
$O_\delta(r_\delta\rho/\sqrt N)$.  Adding the complementary linear
component gives
\begin{equation}
 \|u_{\mathcal M}(\tau)-v(\tau)\|
 \lesssim_\delta r_\delta^2+\frac{r_\delta\tau}{N}
 +\frac{r_\delta\rho}{\sqrt N}
 +r_\delta e^{-a\tau/2}.
 \label{eq:first-exit-error}
\end{equation}
The growing fiber converges to the positive eigenvector of
\Cref{eq:late-contrast-block}, whose key and value coordinates have
the same positive sign.  Thus \Cref{eq:initial-norm-event} gives the
leading term $v(\tau)$ the same winning route in the key row and the
successor-value row, with every winner--loser gap among useful phases at least a fixed
$\delta$-dependent multiple of $r_\delta$.  Choose $r_\delta$ small,
then $C_\delta$ large and $c_\delta$ small, so
\Cref{eq:first-exit-error} is smaller than these gaps.  The lower
exit-time bound and \Cref{eq:compact-dummy-decay} make the
useful--dummy gaps positive as well.  Hence
$\bar{\vc}:=P_{\mathcal M}\vc(\tau)$ has aligned adjacent winners
with a common gap $2\eta_\delta>0$, while \Cref{eq:token-tube} gives
$\|\vc(\tau)-\bar{\vc}\|\leq A_\delta\rho$.
\end{proof}

\begin{lemma}[Mean amplification]
\label{lem:mean-amplification}
The space $\mathcal M$ is invariant.  Fix $0<\eta<M$.  For fixed head $h$ and
branch $P$, call the phase with the largest gate logit $b_{h,r}^{P}$ the winner
and the other phases losers.  For all sufficiently
large $N$, suppose a trajectory in $\mathcal M$ has, in every head, a key
winner $r_h$, the value winner $r_h+1$, and all corresponding winner--loser
gate-logit gaps at least $\eta$.  These winners persist, every gap reaches $M$ in
time bounded independently of $N$, and every gap then diverges.
\end{lemma}

\begin{proof}
Every vocabulary permutation is represented by an orthogonal map that fixes
$\vmu$ and sends $\vx-\vmu$ to the corresponding permuted residual.  The loss
is invariant and its gradient field is equivariant under this action.  Its
common fixed space in $\mathcal C_{\rm eff}$ is $\mathcal M$, proving
invariance.  The induced flow in the coordinates of
\Cref{eq:mean-space} is the gate-logit flow multiplied by $\|\vmu\|^2$.

On $\mathcal M$, write $\vp_h^{P}(r)$ for the coefficient shared by all chunks.
A compressor row is the coefficient vector of one head and branch.  Within a
compressor row, write its coefficients as $p_i$, its winner as $s$,
and the fitness of phase $i$ as $f_i=-\partial(\cL/\beta)/\partial p_i$.  Differentiating the
target term in \Cref{eq:limiting-contrast-form}, the target
contributions to $f_i$ for key phase $r<S$ and value phase
$r+1$ are
$R'(\vp_h^{K}(r))\vp_h^{V}(r+1)/(S-1)$ and
$R(\vp_h^{K}(r))/(S-1)$, respectively.  The dummy key and value fitnesses are
zero.  If every winner gap is at least $\eta>0$, then
\begin{equation}
 p_s-p_j\geq\frac{1-e^{-\eta}}{1+(S-1)e^{-\eta}}
 \qquad(j\ne s).
\end{equation}
Because $\gamma<\log(L-1)$ makes $R'$ increasing on $[0,1]$, the
target-fitness advantage of either winner is at least
\begin{equation}
 g_\eta:=\frac1{S-1}\min\left\{
 R'(0)\frac{1-e^{-\eta}}{1+(S-1)e^{-\eta}},\frac1L
 \right\}>0.
 \label{eq:mean-fitness-gap}
\end{equation}
By \Cref{lem:vocabulary-suppression}, $f_s-f_j\geq g_\eta/2$
when $N$ is large.

Direct differentiation of the gate-logit flow gives
\begin{equation}
 \frac d{dt}(b_s-b_j)=\|\vmu\|^2\left[
 p_j(f_s-f_j)+(p_s-p_j)
 \left(f_s-\sum_i p_if_i\right)\right].
\end{equation}
The second factor in the last term equals
$\sum_i p_i(f_s-f_i)$, so both terms are nonnegative.  While
$b_s-b_j\leq M$, the losing coefficient is at least $e^{-M}/S$, so
the gap grows at rate at least
$\|\vmu\|^2g_\eta e^{-M}/(2S)$.  This gives an $N$-independent time
bound to reach $M$.  Repeating at successively larger levels proves
divergence.
\end{proof}

\begin{lemma}[Persistence of aligned winners]
\label{lem:terminal-cone}
For every $r_1,\ldots,r_H\in[S-1]$ and all sufficiently large $N$, there is
$M_0<\infty$, independent of $N$ and the route choice, such that the terminal
cone, defined as the region where, for every head and all distinct $\vx,\vx'\in\mathcal V$,
\begin{align}
 \langle\vc_{h,r_h}^{K},\vx\rangle-
 \langle\vc_{h,j}^{K},\vx'\rangle\geq M_0
 &\qquad(j\in[S]\setminus\{r_h\}),\notag\\
 \langle\vc_{h,r_h+1}^{V},\vx\rangle-
 \langle\vc_{h,j}^{V},\vx'\rangle\geq M_0
 &\qquad(j\in[S]\setminus\{r_h+1\})
\end{align}
is forward invariant, and every displayed margin tends to $+\infty$.
\end{lemma}

\begin{proof}
For each context and compressor row, define
\begin{equation}
 f_{h,\ell}^{P}(i;\mX)
 :=-\frac{\partial(\cL_{\mX}/\beta)}{\partial \vp_{h,\ell}^{P}(i)},
 \qquad
 \psi_{h,\ell}^{P}(i;\mX)
 :=\vp_{h,\ell}^{P}(i)
 \left(f_{h,\ell}^{P}(i;\mX)-\sum_j\vp_{h,\ell}^{P}(j)
 f_{h,\ell}^{P}(j;\mX)\right).
\end{equation}
Thus $\psi_{h,\ell}^{P}(i;\mX)$ is the negative gradient with respect to the
corresponding sampled gate logit.  Fix a context and row, suppress these indices as
$p_i,f_i,\psi_i$, and write $s$ for the selected phase.  At a one-hot adjacent
corner, the target fitness advantage is $R'(1)/\{L(S-1)\}$ for the selected key and
at least $\{R(1)-1/L\}/\{L(S-1)\}$ for the selected value.  Both are
positive.  Set
\begin{equation}
 d_*:=\frac1{L(S-1)}
 \min\left\{R'(1),R(1)-\frac1L\right\}>0.
\end{equation}
Uniform continuity and \Cref{lem:vocabulary-suppression} imply that,
throughout the terminal cone for sufficiently large $M_0$ and for all large $N$,
\begin{equation}
 f_s-f_j\geq\frac{d_*}{2},
 \qquad |f_i-f_j|\leq C.
\end{equation}
Choose $M_0$ large enough that these bounds hold and
$(S-1)e^{-M_0}C<d_*/4$.  Then
$1-p_s\leq(S-1)e^{-M_0}p_s$, and
\begin{equation}
 \psi_s=p_s\sum_{i\ne s}p_i(f_s-f_i)>0,
 \qquad
 -\psi_j\geq \frac{d_*}{4}p_jp_s>0.
\end{equation}

For $N\geq\|\vmu\|^{-2}$, the Gram matrix is entrywise nonnegative:
\begin{equation}
 \vx^\top\vx'=\|\vmu\|^2-\frac1N+\mathbf1\{\vx=\vx'\}\geq0
 \qquad(\vx,\vx'\in\mathcal V).
\end{equation}
The parameter flow is
\begin{equation}
 \dot{\vc}_{h,r}^{P}=\frac1{|\mathcal X|}\sum_{\mX\in\mathcal X}
 \sum_{\ell=1}^L
 \psi_{h,\ell}^{P}(r;\mX)\vx_{\ell,r}.
\end{equation}
Thus the derivative of a selected-versus-loser margin is a sum of terms
$\psi_s\vx_{\ell,s}^\top\vx-
\psi_j\vx_{\ell,j}^\top\vx'$, all nonnegative.  The terminal cone is therefore
forward invariant.

Suppose one margin, for embeddings $\vx,\vx'$, stayed bounded above by $B$.
Restrict the context average to contexts placing $\vx$ and $\vx'$ in the
corresponding phases of one fixed chunk.  This event has probability
$1/\{N(N-1)\}$.  On it,
$p_s\geq1/S$, $p_j\geq p_se^{-B}$, and
$\psi_s\geq d_*p_sp_j/2$.  Since all other contributions are
nonnegative and $\|\vx\|^2\geq1$, the derivative of this margin is at
least
\begin{equation}
 \frac{d_*e^{-B}}{2S^2N(N-1)}>0,
\end{equation}
contradicting boundedness.  Every margin diverges.
\end{proof}

\begin{proof}[Proof of \Cref{thm:local-hardening-formal}]
Fix $\delta$ and choose the constants so that the three dynamical lemmas
apply.  With probability at least $1-\delta$,
\Cref{lem:mean-entrance} gives a time $\tau$ and
$\bar{\vc}\in\mathcal M$ with aligned adjacent gaps at least
$2\eta_\delta$, while
$\|\vc(\tau)-\bar{\vc}\|\leq C_\delta\rho$.

Let $\bar{\vc}(t)$ be the mean flow from
$\bar{\vc}(0)=\bar{\vc}$.  Apply \Cref{lem:mean-amplification} with
$\eta=\eta_\delta$ and $M=M_0+2$.  After enlarging $C_\delta$ to
meet its large-$N$ threshold and that of \Cref{lem:terminal-cone}, all
mean gaps exceed $M_0+2$ within a time $T$ independent of $N$.
Bounded softmax derivatives control the target field, while
\Cref{lem:vocabulary-suppression} controls the normalizer Hessian.
Together with $\sup_{\vx\in\mathcal V}\|\vx\|^2
\leq\|\vmu\|^2+1$, this gives an $N$-independent Lipschitz constant
for the parameter field.  Gronwall gives
\begin{equation}
 \|\vc(\tau+t)-\bar{\vc}(t)\|
 \leq C_\delta e^{Ct}\rho,
 \qquad 0\leq t\leq T.
\end{equation}
Shrink $c_\delta$ further so that this perturbation changes every
gate-logit margin by less than one.  Then $\vc(\tau+T)$ lies in the
terminal cone, and \Cref{lem:terminal-cone} makes every relevant margin
diverge.  For fixed $N$ there are finitely many such margins, and softmax
gives
\begin{equation}
 1-\vp_{h,\ell}^{P}(r;t,\mX)
 \leq(S-1)\exp\left\{-\min_{j\ne r}
 \bigl(\langle\vc_{h,r}^{P}(t),\vx_{\ell,r}\rangle
 -\langle\vc_{h,j}^{P}(t),\vx_{\ell,j}\rangle\bigr)\right\}
\end{equation}
for the selected phase $r$.  The convergence is therefore uniform over
contexts and chunks, proving the theorem.
\end{proof}

%% file: Appendix/Sections/two_slot_theory.tex
\subsection{A Two-Phase Toy Example}
\label{app:two-slot-theory}

In the toy example, each chunk holds two key entries, each head compresses them with a single trainable logit, and an idealized readout replaces the values. Unlike \Cref{thm:local-hardening}, which does not ensure that different heads select different routes, gradient flow here assigns the two heads different phases, provided the output scale $\beta$ is large enough (condition \eqref{eq:toy-origin-balance-condition} below).

Each of $L$ chunks contains two fixed key entries $\vk_{\ell,1},\vk_{\ell,2}$.
Each head $h\in\{1,2\}$ uses a trainable logit $\theta_h$ to form one
compressed key:
\begin{equation*}
 p_h=\sigma(\theta_h),\qquad
 \bar{\vk}_{h,\ell}=p_h \vk_{\ell,1}+(1-p_h)\vk_{\ell,2}.
\end{equation*}
Here $\sigma(t)=(1+e^{-t})^{-1}$, $p_h=1$ selects phase $1$, $p_h=0$ selects
phase $2$, and $p_h=1/2$ mixes them equally. A subsequent query requests either phase with equal probability.
Under orthogonal addressing, the analogue of exact geometry with $\kappa=0$, a query has inner product one with its requested key in the unique target chunk and zero with every other key. With attention scale one, the target-chunk logit of head $h$ is therefore $p_h$ for phase $1$ and $1-p_h$ for phase $2$, while
all $L-1$ distractor logits are zero. The total correct-chunk attention masses of the two heads
for phases $1$ and $2$ are
\begin{equation*}
 \bar R_1=\sum_{h=1}^2\frac{e^{p_h}}{e^{p_h}+L-1},
 \qquad
 \bar R_2=\sum_{h=1}^2\frac{e^{1-p_h}}{e^{1-p_h}+L-1}.
\end{equation*}
The value readout is idealized: each head contributes one unit of correct-class evidence per unit of attention mass on the target chunk, and the contributions of the two heads add. With a fixed output scale $\beta>0$ and no additional constant offset, the correct-class margins for phases $1$ and $2$ are $\beta \bar R_1$ and $\beta \bar R_2$. Training minimizes the population
cross-entropy, which is the logistic loss of this margin averaged over the two phases:
\begin{equation}
 \mathcal L(p_1,p_2)
 =\big[\log(1+e^{-\beta \bar R_1})
 +\log(1+e^{-\beta \bar R_2})\big]/2.
 \label{eq:main-two-phase-objective}
\end{equation}
Set $D=(L-1)e^{-1/2}$, so that $D>1$ for $L\geq3$.

Informally, moving any mixed $p_h\in(0,1)$ to either endpoint increases $\bar R_1+\bar R_2$, and at a fixed sum the convexity of the logistic loss in attention mass favors $\bar R_1=\bar R_2$. Together these favor the two complementary endpoint optima established below.

\subsubsection{Complementary Two-Phase Selection}

Training updates the logits $\theta_h$ rather than $p_h$, so we rewrite the objective in terms of them. Let $\bm\theta=(\theta_1,\theta_2)$, and write head $h$'s signed phase preference as
\begin{equation}
 \chi(\theta_h)=2p_h-1=\tanh(\theta_h/2).
\end{equation}
For $s=+1$ (phase $1$) or $s=-1$ (phase $2$), the target-chunk logit of head $h$ is $(1+s\chi(\theta_h))/2$, so its correct-chunk attention mass is
\begin{equation}
 a(s\chi(\theta_h)),\qquad
 a(x)=\frac{1}{1+De^{-x/2}}.
 \label{eq:toy-retrieval-function}
\end{equation}
In these coordinates, the two total attention masses and the logistic loss are
\begin{align}
 \bar R_1&=a(\chi(\theta_1))+a(\chi(\theta_2)),&
 \bar R_2&=a(-\chi(\theta_1))+a(-\chi(\theta_2)),
 \label{eq:toy-positive-margin}\\
 \varphi(m)&=\log(1+e^{-\beta m}),&
 \mathcal L(p_1,p_2)&=\frac12\left[\varphi(\bar R_1)+\varphi(\bar R_2)\right],
 \label{eq:toy-bias-objective}
\end{align}
which is \Cref{eq:main-two-phase-objective} rewritten through $\chi(\theta_h)=2p_h-1$.  The corresponding gradient flow is
$\dot{\bm\theta}=-\nabla_{\bm\theta}
\mathcal L(\sigma(\theta_1),\sigma(\theta_2))$, and
$\varphi'<0<\varphi''$.

To analyze this flow, split retrieval into its even and odd parts:
\begin{equation}
 \begin{aligned}
 &a_+(x)=\frac{a(x)+a(-x)}2,\qquad
 a_-(x)=\frac{a(x)-a(-x)}2,\\
 &\bar a_+(\theta)=a_+(\chi(\theta)),\qquad
 \bar a_-(\theta)=a_-(\chi(\theta)).
 \end{aligned}
 \label{eq:toy-bias-evidence}
\end{equation}
Then $\bar R_1=M+I$ and $\bar R_2=M-I$, where
\begin{equation}
 M=\bar a_+(\theta_1)+\bar a_+(\theta_2),\qquad
 I=\bar a_-(\theta_1)+\bar a_-(\theta_2),
 \label{eq:toy-mean-imbalance}
\end{equation}
and hence
\begin{equation}
 \mathcal L(p_1,p_2)
 =\frac12\left[\varphi(M+I)+\varphi(M-I)\right].
 \label{eq:toy-separated-loss}
\end{equation}

This decomposition separates sharpening from balancing.  For $x>0$, set
$t=e^{x/2}>1$. Then
\begin{equation}
 2a_+(x)=\frac{t}{t+D}+\frac1{1+Dt},\qquad
 2\frac{\mathrm d}{\mathrm dt}a_+(2\log t)
 =D\left[\frac1{(t+D)^2}-\frac1{(1+Dt)^2}\right]>0
 \label{eq:toy-mean-margin-gradient}
\end{equation}
for $D>1$ and $x>0$.  Thus increasing $|\theta_h|$ increases its contribution
to the mean attention mass.  At fixed $M$, strict convexity gives
$\mathcal L(p_1,p_2)\geq\varphi(M)$, with equality only at $I=0$.
Hence retrieval competition favors sharp heads, whereas loss curvature favors
balanced coverage.

Both pressures are satisfied together only when the two heads sit at opposite endpoints, and the following theorem makes this precise and shows that gradient flow from near the uniform point reaches these endpoints when $\beta$ is large enough. Let $G(x)=a(x)+a(-x)=2a_+(x)$, which is strictly increasing on $[0,1]$ by the calculation above.

\begin{theorem}[Complementary two-phase selection]
\label{thm:toy-specialization-formal}
Suppose $D>1$.  On $[0,1]^2$, the only global minimizers of
$\mathcal L(p_1,p_2)$ are $(1,0)$ and $(0,1)$, corresponding in the
compactified $\bm\theta$-space to
\begin{equation}
    (+\infty,-\infty)
    \qquad\text{and}\qquad
    (-\infty,+\infty).
    \label{eq:toy-global-minimizers}
\end{equation}
If also
\begin{equation}
    2\beta\sigma\!\left(\frac{2\beta}{1+D}\right)>D-\frac1D,
    \label{eq:toy-origin-balance-condition}
\end{equation}
then any absolutely continuous initialization supported in a sufficiently
small neighborhood of the origin converges almost surely under gradient
flow to one of these minimizers.  If its density is bounded, then for
$\delta,\eta\in(0,1/2)$, outside an event of probability at most $\delta$,
the time until $(p_1(t),p_2(t))$ lies within $\ell_\infty$-distance $\eta$
of $(1,0)$ or $(0,1)$ is
\begin{equation}
 T_{\eta,\delta}
 =O\!\left(\log\frac1\delta+\frac1\eta\right),
 \label{eq:toy-specialization-time}
\end{equation}
where the hidden constant depends on $L$, $\beta$, and the initialization law
through its fixed neighborhood and density bound.
\end{theorem}

We establish the loss landscape first, then show that nearby trajectories
escape toward and remain near a complementary route.

\begin{lemma}[Global minimizers]
\label{lem:toy-global-minimizers}
If $D>1$, the only minimizers of $\mathcal L$ on $[0,1]^2$ are
$(1,0)$ and $(0,1)$.
\end{lemma}

\begin{proof}
The mean attention mass is
\begin{equation}
    \frac{\bar R_1+\bar R_2}{2}
    =\frac{G(|\chi(\theta_1)|)+G(|\chi(\theta_2)|)}{2}
    \leq G(1).
    \label{eq:toy-mean-margin-bound}
\end{equation}
Jensen's inequality and monotonicity of $\varphi$ yield
\begin{equation}
    \mathcal L(p_1,p_2)
    \geq \varphi\!\left(\frac{\bar R_1+\bar R_2}{2}\right)
    \geq \varphi(G(1)).
    \label{eq:toy-global-loss-bound}
\end{equation}
Equality in the second inequality requires
$|\chi(\theta_1)|=|\chi(\theta_2)|=1$, and strict convexity makes the first tight only when
$\bar R_1=\bar R_2$, which forces opposite signs.
\end{proof}

The remaining lemmas analyze the flow near the complementary line
$(\theta_1,\theta_2)=(\xi,-\xi)$.  Use the orthonormal coordinates
\begin{equation}
    u=\frac{\theta_1-\theta_2}{\sqrt2},
    \qquad
    v=\frac{\theta_1+\theta_2}{\sqrt2}.
    \label{eq:toy-common-complementary-coordinates}
\end{equation}
For $\xi\geq0$, put
\begin{equation}
    \bar M_\xi=2\bar a_+(\xi)=G(\chi(\xi))
    \label{eq:toy-complementary-margin}
\end{equation}
and let
\begin{equation}
    \lambda_\perp(\xi)
    =\varphi'(\bar M_\xi)\bar a_+''(\xi)
     +2\varphi''(\bar M_\xi)\bigl(\bar a_-'(\xi)\bigr)^2.
    \label{eq:toy-transverse-curvature}
\end{equation}
This is the Hessian eigenvalue in the common-head direction
$(1,1)/\sqrt2$ at $(\xi,-\xi)$.  At the uniform point, $\bar M_0=2/(1+D)$, and the escape rate along
$(1,-1)/\sqrt2$ is
\begin{equation}
    \lambda_{\mathrm{esc}}
    =-\varphi'(\bar M_0)\frac{D(D-1)}{16(1+D)^3}>0.
    \label{eq:toy-origin-eigenvalues}
\end{equation}

\begin{lemma}[Stability of complementary routing]
\label{lem:toy-transverse-stability}
Suppose $D>1$ and \Cref{eq:toy-origin-balance-condition} holds.  The
complementary line is
invariant, $\lambda_\perp(\xi)>0$ for every $\xi\geq0$, and
\begin{equation}
    \inf_{\xi\geq0}e^\xi\lambda_\perp(\xi)>0.
    \label{eq:toy-uniform-transverse-balance}
\end{equation}
\end{lemma}

\begin{proof}
Invariance follows by exchanging the two heads and phases.  Let $x=\chi(\xi)$.
Where $\bar a_+''(\xi)>0$, the curvature satisfies
\begin{equation}
    \frac{\bar a_+''(\xi)}{2(\bar a_-'(\xi))^2}
    \leq\frac12\left(D-\frac1D\right).
    \label{eq:toy-curvature-ratio-bound}
\end{equation}
Indeed, $a_+'(x)>0$ for $x>0$, and
\begin{equation}
    \frac{\bar a_+''(\xi)}{2(\bar a_-'(\xi))^2}
    =\frac{a_+''(x)-2xa_+'(x)/(1-x^2)}{2(a_-'(x))^2}
    \leq\frac{a_+''(x)}{2(a_-'(x))^2}.
    \label{eq:toy-bias-to-preference-curvature}
\end{equation}
To verify the last bound, put $t=e^{x/2}\geq1$ and
$Q=(1+Dt)^2+(D+t)^2$.  Substituting
$a(x)=(1+De^{-x/2})^{-1}$ and clearing positive denominators gives
\begin{multline*}
 (D^2-1)tQ^2-2(D+t)(1+Dt)\\
 {}\times\bigl[(D-t)(1+Dt)^3+(Dt-1)(D+t)^3\bigr]
 =(t-1)^2(D-1)P_D(t),
\end{multline*}
where
\begin{align*}
 P_D(t)={}&(t+t^3)(1+D^5)
 +(t+8t^2+t^3)(D+D^4)\\
 &+(2+6t+4t^2+6t^3+2t^4)(D^2+D^3).
\end{align*}
Every coefficient is positive, proving
\Cref{eq:toy-curvature-ratio-bound}.

Furthermore,
\begin{equation}
    \frac{\varphi''(\bar M_\xi)}{-\varphi'(\bar M_\xi)}
    =\beta\sigma(\beta \bar M_\xi)
    \geq\beta\sigma(\beta \bar M_0)
    =\frac{\varphi''(\bar M_0)}{-\varphi'(\bar M_0)},
    \label{eq:toy-loss-curvature-monotonicity}
\end{equation}
because $\bar M_\xi=G(\chi(\xi))$ is increasing.  Combining this inequality,
\Cref{eq:toy-origin-balance-condition}, and
\Cref{eq:toy-curvature-ratio-bound} proves $\lambda_\perp(\xi)>0$ whenever
$\bar a_+''(\xi)>0$. If $\bar a_+''(\xi)\leq0$, positivity is immediate from
$\varphi'<0<\varphi''$.

Finally, $e^\xi\lambda_\perp(\xi)$ is continuous and positive for finite $\xi$, and direct
differentiation gives
\begin{equation}
    e^\xi\lambda_\perp(\xi)
    \longrightarrow 2[-\varphi'(G(1))]a_+'(1)>0.
    \label{eq:toy-weighted-curvature-limit}
\end{equation}
Compactness after adjoining $\xi=\infty$ proves
\Cref{eq:toy-uniform-transverse-balance}.
\end{proof}

We next use this curvature bound to trap trajectories in a tube around the complementary line.

\begin{lemma}[Attracting complementary tube]
\label{lem:toy-attracting-tube}
Suppose $D>1$ and \Cref{eq:toy-origin-balance-condition} holds.  For every
$\xi_*>0$, there are $\zeta_T,c_T,C_T,c_N>0$ such that, in the coordinates
\begin{equation*}
    (\theta_1,\theta_2)
    =\left(\xi+\frac{v}{\sqrt2},-\xi+\frac{v}{\sqrt2}\right),
    \qquad \xi\geq \xi_*,\quad |v|\leq\zeta_T,
\end{equation*}
the gradient flow satisfies
\begin{equation}
    c_Te^{-\xi}\leq\dot\xi\leq C_Te^{-\xi},\qquad
    v\dot v\leq-c_Ne^{-\xi}v^2.
    \label{eq:toy-tube-estimates}
\end{equation}
This tube is forward invariant, and its trajectories have
$\xi(t)\to\infty$ and $v(t)\to0$.
\end{lemma}

\begin{proof}
The tangential speed at $v=0$ is
\begin{equation}
    \Omega(\xi)
    :=-\frac14\varphi'(G(\chi(\xi)))G'(\chi(\xi))
        \bigl(1-\chi(\xi)^2\bigr)>0.
    \label{eq:toy-tangential-speed}
\end{equation}
On the line $v=0$, $\dot\xi=\Omega(\xi)>0$, and the transverse linearization is
$\dot v=-\lambda_\perp(\xi)v$.  Since $\chi(\xi)=\tanh(\xi/2)$,
\begin{equation}
    e^\xi\Omega(\xi)\longrightarrow
    -\varphi'(G(1))G'(1)>0,
    \label{eq:toy-weighted-tangential-limit}
\end{equation}
while \Cref{eq:toy-uniform-transverse-balance} bounds
$e^\xi\lambda_\perp(\xi)$ away from zero.  Every derivative with respect to a
saturated logit carries a factor $\chi'(\xi)=O(e^{-\xi})$.  Taylor expansion in $v$,
uniform after weighting the vector field by $e^\xi$, gives
\begin{align}
    \dot\xi&=\Omega(\xi)+O(e^{-\xi}|v|),
    \label{eq:toy-tube-tangent-expansion}\\
    \dot v&=-\lambda_\perp(\xi)v+O(e^{-\xi}v^2).
    \label{eq:toy-tube-normal-expansion}
\end{align}
The weighted coefficients extend continuously to $\xi=\infty$.  Their positive
lower bounds and finite upper bounds on $[\xi_*,\infty]$ yield
\Cref{eq:toy-tube-estimates} after choosing $\zeta_T$ sufficiently small.
The sign of $v\dot v$ makes the tube forward invariant, and the tangential
bounds give $\xi(t)\to\infty$.  Moreover, whenever $v\ne0$,
$\mathrm d\log|v|/\mathrm d\xi\leq-c_N/C_T$, so $v(t)\to0$.
\end{proof}

It remains to show that trajectories starting near the uniform point enter this tube.

\begin{lemma}[Exit into the complementary tube]
\label{lem:toy-local-exit}
Under the assumptions of \Cref{lem:toy-attracting-tube}, there are
$u_*,\rho_0>0$ such that every trajectory with
$\lVert(u(0),v(0))\rVert_2\leq\rho_0$ and $u(0)\ne0$ reaches
$|u|=u_*$ within time
\begin{equation}
 \tau\leq\frac{2}{\lambda_{\mathrm{esc}}}
 \log\frac{u_*}{|u(0)|}.
 \label{eq:toy-local-exit-time}
\end{equation}
At exit, the trajectory lies in the tube from
\Cref{lem:toy-attracting-tube} with $\xi_*=u_*/\sqrt2$, after exchanging
heads if $u(0)<0$.
\end{lemma}

\begin{proof}
The symmetry lines $u=0$ and $v=0$ are invariant.  Smoothness and the two
Hessian signs at the origin give
\begin{equation}
 \dot u=u\left(\lambda_{\mathrm{esc}}
       +O(\lVert(u,v)\rVert_2)\right),\qquad
 \dot v=v\left(-\lambda_\perp(0)
       +O(\lVert(u,v)\rVert_2)\right).
 \label{eq:toy-factorized-saddle-flow}
\end{equation}
Choose $u_*>0$ small enough that the coefficients in
\Cref{eq:toy-factorized-saddle-flow} stay within half their origin
magnitudes whenever $|u|,|v|\leq u_*$.  Apply
\Cref{lem:toy-attracting-tube} at
$\xi_*=u_*/\sqrt2$ and reduce its tube width $\zeta_T$ to at most $u_*$.  Then
whenever $|u|\leq u_*$ and $|v|\leq\zeta_T$,
\begin{equation*}
 \frac{\mathrm d}{\mathrm dt}\log|u|
 \geq\frac{\lambda_{\mathrm{esc}}}{2},\qquad
\frac{\mathrm d}{\mathrm dt}|v|
 \leq-\frac{\lambda_\perp(0)}{2}|v|.
\end{equation*}
The second inequality holds for $v\ne0$, and the line $v=0$ is invariant.
Choose $\rho_0<\min\{u_*,\zeta_T\}$.  Then $|v|$ cannot reach $\zeta_T$ before
$|u|$ reaches $u_*$. Integrating the first inequality gives
\Cref{eq:toy-local-exit-time}.  At exit $\xi=|u(\tau)|/\sqrt2=\xi_*$ and
$|v(\tau)|\leq\rho_0<\zeta_T$, so the point lies in the tube.
\end{proof}

\begin{proof}[Proof of \Cref{thm:toy-specialization-formal}]
\Cref{lem:toy-global-minimizers} identifies the only global minima, while
\Cref{lem:toy-transverse-stability,lem:toy-attracting-tube,lem:toy-local-exit}
give convergence for every sufficiently small initialization off
$u=0$, which has probability zero under an absolutely continuous law.

For the rate, enclose the fixed initialization support in a ball of radius
$\rho\leq\rho_0$ and let its density be bounded by $C$.  Then
$\Pr(|u(0)|\leq t)\leq4C\rho t$.  Outside an event of probability at most
$\delta$, we have $|u(0)|>\delta/(4C\rho)$, so
\Cref{eq:toy-local-exit-time} gives tube entry in
$O(\log(1/\delta))$ time.
The normal coordinate then remains within $\zeta_T$.  To reach
$\ell_\infty$-error at most $\eta$, it suffices that
$\xi\geq\log((1-\eta)/\eta)+\zeta_T/\sqrt2$.  Integrating the tube bound
$\dot\xi\geq c_Te^{-\xi}$ up to this threshold (if it exceeds $\xi_*$)
takes at most $e^{\zeta_T/\sqrt2}/(c_T\eta)$ additional time.
This proves \Cref{eq:toy-specialization-time} for
$\delta,\eta\in(0,1/2)$, and the same tube bounds ensure convergence to the
corresponding minimizer.
\end{proof}

\subsubsection{From the Toy Example to the General Model}
\label{app:toy-to-general}

The toy example isolates two pressures, one favoring sharp heads and one favoring balanced coverage of the phases. Sharpening comes from competition with the $L-1$ distractor chunks. Because $D>1$, a uniform head gives the correct chunk less than half of its attention, and in this regime the correct-chunk attention mass rises more when the target logit increases than it falls when the logit decreases by the same amount. A head that moves toward one phase therefore gains more attention mass on that phase than it loses on the other, which raises the mean attention mass over the two phases (\Cref{eq:toy-mean-margin-gradient}). Balance comes from the loss. Both heads add evidence to the same margin for each phase, and the logistic loss flattens as a margin grows, so evidence added to the phase with the larger margin lowers the loss less than evidence added to the phase with the smaller one. For a fixed mean, the loss is therefore smallest when both phases receive the same attention mass. Both pressures are satisfied together only at the complementary endpoints, which is why these are the only global minimizers (\Cref{lem:toy-global-minimizers}).

The same two pressures govern the dynamics. At the uniform point, the two heads are identical. Along the complementary direction $(1,-1)/\sqrt2$, in which the heads move toward different phases, the two phases stay balanced, so only sharpening acts and trajectories escape at rate $\lambda_{\mathrm{esc}}$. Along the common direction $(1,1)/\sqrt2$, in which both heads move toward the same phase, sharpening and balance compete. Condition \eqref{eq:toy-origin-balance-condition} is sufficient for balance to win, which makes this direction stable and the uniform point a saddle (\Cref{lem:toy-transverse-stability}). A small random initialization therefore leaves the uniform point along the complementary direction and then stays near the complementary line.

Both pressures reappear in the general model, but they change in different ways. Sharpening survives, but it now comes from pairing keys with values rather than from the asymmetry that $D>1$ provides. Under exact geometry, a head's contribution to the target logit for a query at position $r$ is its correct-chunk attention, which grows with its key weight at $r$, times its value weight at $r+1$. Because the value weights sum to one, this contribution summed over query positions is at most the attention that a one-hot key receives, and it reaches that value only when the key coefficient is one-hot at some $r$ and the value coefficient is one-hot at $r+1$. In \Cref{sec:oracle-minimizers}, we use this bound to rule out mixed coefficients, in the role that \Cref{lem:toy-global-minimizers} plays for the toy example.

Balance, by contrast, becomes weak when the vocabulary is large. In that regime, the cross-entropy loss is nearly linear in the target logit, because its normalizer is dominated by the many vocabulary items that do not appear in the context. Under exact geometry, two heads that select the same position keep the same expected target logit and pay only through the normalizer, by an amount at most of order $1/N$. This penalty still makes complementary coverage optimal when $\varepsilon=0$ (\Cref{sec:oracle-minimizers}). Our dynamics proof, however, treats all normalizer terms as small perturbations, and the remaining target terms do not couple different heads, so the proof has no force that separates them (\Cref{sec:gradient-flow}). Unlike the toy example, \Cref{thm:local-hardening-formal} therefore guarantees persistent selection without guaranteeing complementary coverage.

%% file: Appendix/Sections/mod_sensitivity_details.tex
\section{Additional Details on the Code-Completion Examples}\label{app:v4}\label{app:v4-coordinates}

Here we give the details behind the code-completion example of \Cref{sec:v4-code}: the exact token accounting, the same example on post-trained models and on DeepSeek-V3.1-Base, which has no chunked compression, and two further code-completion examples. As in \Cref{sec:modsens}, a token at zero-based position \(t\) has phase \(t\bmod S\) for compression stride \(S\).

\subsection{FP8 Example for DeepSeek-V4-Flash-Base}
\label{app:v4-code-details}

In \Cref{sec:v4-code}, we show the code-completion example with a single docstring filler. Here we give the exact inputs needed to reproduce it, together with three further filler families that test whether the reversals depend on the filler's format or wording. Unless noted otherwise, the results are measured on DeepSeek-V4-Flash-Base using the official TileLang inference implementation with model parallelism four and FP8 expert weights.

\paragraph{Source example and reference continuation.}
The example comes from the official DeepSeek-V4 inference code,
\texttt{inference/kernel.py}\footnote{\url{https://huggingface.co/deepseek-ai/DeepSeek-V4-Flash-0731/blob/0c18385a15b3b34253d6cf7132f3f4a0442d4ea9/inference/kernel.py}},
which implements TileLang kernels and PyTorch wrappers for quantized computation.
The scored prefix contains the first 88 source lines, ending immediately after
\CodeIn{T.Cast(compute\_dtype, T.Cast(FP}.
The original source continues with \CodeIn{8}.
The preceding code supports this reference: the function is documented as
block-wise FP8 quantization, defines FP8 range and scale quantities, and sets
\CodeIn{compute\_dtype} to \CodeIn{FP32} for internal computation.
At the scored location, an inner cast to \CodeIn{FP8} followed by a cast back
to \CodeIn{compute\_dtype} is consistent with this quantization operation.
These cues motivate reference-token matching, but they do not establish that
no alternative token admits a syntactically valid continuation.

\paragraph{Exact input construction and scoring.}
Using zero-based token indices, the reference token \CodeIn{8} in the
unmodified source occurs at index 983. A filler of \(L\) tokenizer tokens
moves this token to \(t=983+L\). For \(L=24,\ldots,39\), the scored token
positions are 1007 through 1022. Each input is one contiguous raw prefix
containing the filler followed by the 983 source-code tokens preceding
\CodeIn{8}. No instruction, beginning-of-sequence token, or chat template
is added, and the target token and subsequent source code are omitted.
We read the logits at the final input position \(t-1\), immediately after
\CodeIn{FP}.

The prefix ends with \CodeIn{FP} (token ID 20199). The reference
continuation \CodeIn{8} and competing continuation \CodeIn{32} have token
IDs 26 and 2111, respectively. Thus, \CodeIn{FP8} tokenizes as
\([20199,26]\), whereas \CodeIn{FP32} tokenizes as \([20199,2111]\).
We report their next-token probabilities \(P(8)\) and \(P(32)\), together
with their difference
\[
\Delta=P(8)-P(32).
\]
These probabilities are selected from the full next-token distribution, so \(P(8)+P(32)\) need not equal one. The reported input-shift coordinate is filler length \(L\), distinct from the scored-token position \(t\). The filler shifts both the source-code prefix and the scoring position, so this test does not isolate the position of an individual source token.

\paragraph{Filler constructions.}
Family A (the filler shown in \Cref{sec:v4-code}) prepends a triple-quoted Python docstring. At a filler length of 24 tokens, its exact text is
\begin{quote}
\small\ttfamily
\char34\char34\char34This module contains optimized tensor kernels used by the model inference runtime.\\
\mbox{\ }= = = = = = = = =\\
\char34\char34\char34
\end{quote}
The second line initially contains nine space-separated \CodeIn{=} characters. Each successive variant appends one additional space-separated \CodeIn{=}, ending with 24 such characters. The total filler lengths are therefore \(L=24,\ldots,39\), giving 16 prompts. The source-code tokens following the filler remain unchanged.

The other families are B, C, and D. As summarized in
\Cref{tab:v4-filler-constructions}, families A, B, and C keep their
natural-language sentence fixed throughout the sweep and add one
punctuation token at each successive length. These families differ in whether
the filler is a Python comment or docstring, in its description of the source
file, and in the repeated separator. Family D instead realizes the 16 token
counts using semantically similar but lexically distinct single-line
comments, without a repeated-symbol padding line. Thus, A/B/C provide
content-matched position shifts within each family, whereas D additionally
tests robustness to natural-language rephrasing.

\begin{table}[tb]
\centering
\caption{Filler-family constructions. Quoted sentences are reproduced verbatim. Repeated symbols are separated by spaces in the actual inputs.}
\label{tab:v4-filler-constructions}
\small
\begin{tabular}{@{}lp{0.37\linewidth}p{0.48\linewidth}@{}}
\toprule
Family & Fixed text or 24-token endpoint & Length variation \\
\midrule
A &
Triple-quoted docstring: ``This module contains optimized tensor kernels used by the model inference runtime.'' &
The next docstring line contains 9--24 \CodeIn{=} characters before the closing triple quotes. \\

B &
Comment: ``This module contains optimized tensor kernels used by the model inference runtime.'' &
A second comment line contains 9--24 \CodeIn{=} characters. \\

C &
Comment: ``Internal model inference kernel implementations.'' &
A second comment line contains 16--31 \CodeIn{*} characters, with one added at each successive length. \\

D &
At 24 tokens: ``This file defines TileLang kernels for block-wise FP8 quantization and the related low-precision operations for model inference.'' &
One natural-language comment is used at each length, with no padding line. \\
\bottomrule
\end{tabular}
\end{table}

\paragraph{Results across filler families.}
In \Cref{tab:v4-filler-family-results}, we report the probability difference \(\Delta\) for all 16 lengths and all four filler families, next to the post-trained DeepSeek-V4-Flash-0731 results discussed in \Cref{app:v4-code-posttrained}. Positive values favor the reference continuation \CodeIn{8}, and negative values favor \CodeIn{32}. No ties occur in these measurements.

\begin{table}[tb]
\centering
\caption{DeepSeek-V4-Flash-Base (left) and DeepSeek-V4-Flash-0731 (right) results on the FP8 example. Each entry is \(\Delta=P(8)-P(32)\), computed from the reported probabilities and rounded to three decimal places. Green marks \(\Delta>0\), where the reference \CodeIn{8} is ranked above \CodeIn{32}, and red marks \(\Delta<0\). ``Tokens'' is the filler length, and thin rules separate groups of four lengths, matching the stride of four. DeepSeek-V4-Flash-Base receives the raw prefix with fillers of 24--39 tokens, and DeepSeek-V4-Flash-0731 receives the chat-formatted prompt of \Cref{app:v4-code-posttrained} with fillers of 23--38 tokens.}
\label{tab:v4-filler-family-results}
\small
\setlength{\tabcolsep}{3.5pt}
\renewcommand{\arraystretch}{1.05}
\begin{tabular}{@{}rrrrr@{\hspace{14pt}}c@{\hspace{14pt}}rrrrr@{}}
\toprule
\multicolumn{5}{c}{DeepSeek-V4-Flash-Base (raw prefix)} && \multicolumn{5}{c}{DeepSeek-V4-Flash-0731 (chat prompt)} \\
\cmidrule(lr){1-5}\cmidrule(lr){7-11}
Tokens & A & B & C & D && Tokens & A & B & C & D \\
\midrule
24 & \WrongOrder{-0.786} & \WrongOrder{-0.791} & \WrongOrder{-0.007} & \CorrectOrder{+0.060} && 23 & \WrongOrder{-0.910} & \WrongOrder{-0.775} & \WrongOrder{-0.364} & \WrongOrder{-0.789} \\
25 & \WrongOrder{-0.636} & \WrongOrder{-0.383} & \WrongOrder{-0.415} & \WrongOrder{-0.484} && 24 & \CorrectOrder{+0.615} & \CorrectOrder{+0.451} & \CorrectOrder{+0.644} & \WrongOrder{-0.829} \\
26 & \CorrectOrder{+0.586} & \CorrectOrder{+0.751} & \CorrectOrder{+0.845} & \CorrectOrder{+0.614} && 25 & \WrongOrder{-0.831} & \WrongOrder{-0.854} & \WrongOrder{-0.822} & \WrongOrder{-0.989} \\
27 & \CorrectOrder{+0.914} & \CorrectOrder{+0.862} & \CorrectOrder{+0.809} & \CorrectOrder{+0.681} && 26 & \WrongOrder{-0.936} & \WrongOrder{-0.939} & \WrongOrder{-0.824} & \WrongOrder{-0.875} \\
\cmidrule(lr){1-5}\cmidrule(lr){7-11}
28 & \WrongOrder{-0.248} & \WrongOrder{-0.410} & \WrongOrder{-0.354} & \WrongOrder{-0.337} && 27 & \WrongOrder{-0.679} & \WrongOrder{-0.668} & \WrongOrder{-0.266} & \WrongOrder{-0.151} \\
29 & \WrongOrder{-0.043} & \CorrectOrder{+0.119} & \WrongOrder{-0.533} & \WrongOrder{-0.764} && 28 & \CorrectOrder{+0.860} & \CorrectOrder{+0.073} & \CorrectOrder{+0.241} & \WrongOrder{-0.420} \\
30 & \CorrectOrder{+0.826} & \CorrectOrder{+0.808} & \CorrectOrder{+0.724} & \CorrectOrder{+0.696} && 29 & \WrongOrder{-0.901} & \WrongOrder{-0.238} & \WrongOrder{-0.713} & \WrongOrder{-0.813} \\
31 & \CorrectOrder{+0.927} & \CorrectOrder{+0.933} & \CorrectOrder{+0.822} & \CorrectOrder{+0.737} && 30 & \WrongOrder{-0.960} & \WrongOrder{-0.954} & \WrongOrder{-0.803} & \WrongOrder{-0.608} \\
\cmidrule(lr){1-5}\cmidrule(lr){7-11}
32 & \WrongOrder{-0.219} & \WrongOrder{-0.660} & \WrongOrder{-0.119} & \WrongOrder{-0.129} && 31 & \WrongOrder{-0.882} & \WrongOrder{-0.159} & \WrongOrder{-0.160} & \WrongOrder{-0.715} \\
33 & \WrongOrder{-0.644} & \WrongOrder{-0.276} & \WrongOrder{-0.118} & \WrongOrder{-0.624} && 32 & \CorrectOrder{+0.282} & \WrongOrder{-0.384} & \CorrectOrder{+0.822} & \WrongOrder{-0.566} \\
34 & \CorrectOrder{+0.746} & \CorrectOrder{+0.806} & \CorrectOrder{+0.543} & \CorrectOrder{+0.524} && 33 & \WrongOrder{-0.917} & \WrongOrder{-0.828} & \WrongOrder{-0.874} & \WrongOrder{-0.946} \\
35 & \CorrectOrder{+0.904} & \CorrectOrder{+0.829} & \CorrectOrder{+0.895} & \CorrectOrder{+0.653} && 34 & \WrongOrder{-0.968} & \WrongOrder{-0.940} & \WrongOrder{-0.786} & \WrongOrder{-0.968} \\
\cmidrule(lr){1-5}\cmidrule(lr){7-11}
36 & \WrongOrder{-0.468} & \WrongOrder{-0.503} & \CorrectOrder{+0.527} & \CorrectOrder{+0.249} && 35 & \WrongOrder{-0.869} & \WrongOrder{-0.438} & \WrongOrder{-0.826} & \WrongOrder{-0.787} \\
37 & \WrongOrder{-0.553} & \WrongOrder{-0.198} & \WrongOrder{-0.510} & \WrongOrder{-0.532} && 36 & \CorrectOrder{+0.706} & \CorrectOrder{+0.597} & \CorrectOrder{+0.715} & \WrongOrder{-0.895} \\
38 & \CorrectOrder{+0.887} & \CorrectOrder{+0.852} & \CorrectOrder{+0.768} & \CorrectOrder{+0.591} && 37 & \WrongOrder{-0.945} & \WrongOrder{-0.824} & \WrongOrder{-0.834} & \WrongOrder{-0.892} \\
39 & \CorrectOrder{+0.948} & \CorrectOrder{+0.953} & \CorrectOrder{+0.866} & \CorrectOrder{+0.790} && 38 & \WrongOrder{-0.903} & \WrongOrder{-0.790} & \WrongOrder{-0.770} & \WrongOrder{-0.880} \\
\bottomrule
\end{tabular}
\end{table}

Family A exhibits an exact four-token pattern: \(\Delta>0\) at
all lengths with \(L\bmod4\in\{2,3\}\), and \(\Delta<0\) at the remaining
lengths. The full-vocabulary top-1
prediction is \CodeIn{8} when \(\Delta>0\) and \CodeIn{32} when
\(\Delta<0\). Since \(t=983+L\), filler residues
\(L\bmod4\in\{2,3\}\) correspond to scored-token residues
\(t\bmod4\in\{1,2\}\). Across Family A's 16 lengths, the mean combined
probability \(P(8)+P(32)\) is 0.982, while \(\Delta\) ranges from
\(-0.786\) to \(+0.948\).

At lengths satisfying \(L\bmod4\in\{2,3\}\), all 32 measurements across
the four families favor \CodeIn{8}. At the remaining lengths, 28 of 32
favor \CodeIn{32}, and four favor \CodeIn{8}, with no ties.
These four exceptions are Family B at \(L=29\), Family C at \(L=36\),
and Family D at \(L=24\) and \(L=36\), with probability differences
\(+0.119\), \(+0.527\), \(+0.060\), and \(+0.249\), respectively. Across the complete sweep, \(\Delta\) ranges
from \(-0.786\) to \(+0.948\) for A, from \(-0.791\) to \(+0.953\) for B,
from \(-0.533\) to \(+0.895\) for C, and from \(-0.764\) to \(+0.790\) for D.

\input{Appendix/Sections/posttrained_code_completion}

\subsection{Additional Code-Completion Examples}\label{app:v4-code-additional}
To show that the reversals are not peculiar to the FP8 example, we give two further examples in which the same kind of filler sweep produces top-1 reversals, periodic for most of the evaluated models: a shell test from the Omarchy GitHub repository and a SQL task adapted from the \texttt{gretelai/synthetic\_text\_to\_sql} dataset. We evaluate DeepSeek-V4-Flash-Base, DeepSeek-V4-Flash-0731, and DeepSeek-V4.1-Flash.

\subsubsection{Omarchy Lid-Close Test}
\label{app:omarchy-lid-close}

Omarchy is a Linux distribution whose system utilities handle desktop
configuration and hardware events. This example comes from
\texttt{test/shell.d/lid-close-test.sh}, which tests the response to closing
a laptop lid.\footnote{\href{https://github.com/omacom/omarchy/blob/f99d33a8ddee7b36509a71a6d20d5d23355ce8b1/test/shell.d/lid-close-test.sh}{Full source file: \texttt{test/shell.d/lid-close-test.sh}}, commit \texttt{f99d33a8}.}
The test distinguishes an undocked laptop, which should lock before
suspending, from a laptop connected to an external display, which should
remain usable in clamshell mode.

The excerpt below retains the scenario setup and the cues preceding the
completion point. Ellipses mark omitted code or comment text, and
\texttt{[CURSOR]} denotes the next-token prediction position.
\begin{quote}
\footnotesize
\begin{verbatim}
setup_scenario() {
  scenario_dir="$tmpdir/$1"
  # ...
  local closed="$2" docked="$3"

  cat >"$mock_bin/omarchy-hw-laptop-closed" <<SH
#!/bin/bash
exit $closed
SH
  # ...
}

# An undocked lid close is about to suspend ...
setup_scenario undocked 0 1
# ...

# A docked lid close is clamshell mode:
# logind leaves the machine awake and the session
# stays in use on the external display,
# so locking it would be wrong.
setup_scenario docked [CURSOR]
\end{verbatim}
\end{quote}

The first argument names the scenario. The second controls the exit
status of the mocked lid-closed detector, and the third controls the
external-display detector. Shell status \texttt{0} indicates that the
detected condition holds. Thus, \texttt{setup\_scenario undocked 0 1}
represents a closed lid without an external display. The docked scenario
also explicitly describes a closed lid, so the next token should again
be \texttt{0}. Choosing \texttt{1} would instead simulate a lid that is
not closed, contradicting the scenario being tested. The comment about
avoiding screen locking specifies the expected response to this state,
whereas the argument being completed specifies the lid state itself.

We vary a natural-language filler over 16 consecutive
lengths, from 24 to 39 tokens under the DeepSeek-V4 tokenizer, while keeping the
source prefix fixed.
We report the next-token predictions in increasing filler-length order in \Cref{tab:omarchy-lid-close}.

\begin{table}[tb]
\centering
\caption{Omarchy lid-close completion across 16 filler variants.
The reference token is \texttt{0}, and every mismatch predicts \texttt{1}.
Green marks correct predictions and red marks mismatches. Spaces group consecutive predictions for readability.}
\label{tab:omarchy-lid-close}
\small
\setlength{\tabcolsep}{4pt}
\begin{tabular}{@{}lcrr@{}}
\toprule
Model & Top-1 sequence &
\shortstack{Mismatches/16} & $P(0)$ range \\
\midrule
DeepSeek-V4-Flash-Base (TileLang) &
\SeqOk{0}\SeqOk{0}\SeqOk{0}\SeqOk{0} \SeqOk{0}\SeqOk{0}\SeqOk{0}\SeqBad{1} \SeqOk{0}\SeqOk{0}\SeqOk{0}\SeqBad{1} \SeqOk{0}\SeqOk{0}\SeqOk{0}\SeqBad{1} & 3 & 0.175--0.951 \\
DeepSeek-V4-Flash-0731 &
\SeqBad{1}\SeqBad{1}\SeqOk{0}\SeqOk{0} \SeqBad{1}\SeqBad{1}\SeqOk{0}\SeqOk{0} \SeqBad{1}\SeqBad{1}\SeqOk{0}\SeqOk{0} \SeqBad{1}\SeqBad{1}\SeqOk{0}\SeqOk{0} & 8 & 0.270--0.725 \\
DeepSeek-V4.1-Flash &
\SeqOk{0}\SeqOk{0}\SeqOk{0}\SeqOk{0} \SeqOk{0}\SeqOk{0}\SeqOk{0}\SeqOk{0} \SeqBad{1}\SeqOk{0}\SeqBad{1}\SeqOk{0} \SeqBad{1}\SeqOk{0}\SeqBad{1}\SeqOk{0} & 4 & 0.378--0.867 \\
\bottomrule
\end{tabular}
\end{table}

DeepSeek-V4-Flash-0731 exhibits an exact four-position cycle,
alternating between two incorrect and two correct predictions.
DeepSeek-V4-Flash-Base makes errors at inputs 8, 12, and 16, while
DeepSeek-V4.1-Flash alternates between incorrect and
correct predictions over the final eight inputs.

\subsubsection{SQL Station-Distance Query}
\label{app:sql-station-distance}

This example is adapted from test record 2113 of
\texttt{gretelai/synthetic\_text\_to\_sql}, a dataset pairing natural-language
questions with database contexts and reference SQL queries.\footnote{\href{https://huggingface.co/datasets/gretelai/synthetic_text_to_sql/blob/740ab236e64503fba51be1101df7a1be83bf455d/synthetic_text_to_sql_test.snappy.parquet}{Full source file: \texttt{synthetic\_text\_to\_sql\_test.snappy.parquet}}, row 2112 (record 2113), revision \texttt{740ab23}.}
The original task asks for the distance between stations 5 and 12.
Its context supplies a table containing the two stations' coordinates,
and its reference query computes their spherical distance.

We retain the original schema, seed rows, and reference query, and adapt
the example into a constrained SQL replay task in which a recorded output
determines which operator is correct. We construct a contrasting
query by replacing one \texttt{+} in the distance formula with \texttt{-},
and include this complete query together with its recorded output.
The reference query is then presented as \texttt{q\_replay}, with its
recorded output marked authoritative, and truncated immediately before
the same arithmetic operator. The task requires choosing \texttt{+} or
\texttt{-} to reproduce that output without changing the remaining SQL.

The excerpt below shows the seed data, contrasting query, and completion
point. Line breaks are adjusted for readability, ellipses mark omitted
text, and \texttt{[CURSOR]} denotes the prediction position.
\begin{quote}
\footnotesize
\begin{verbatim}
-- Source request: What is the distance between
-- station 5 and station 12?
-- At the next arithmetic slot, write exactly + or -;
-- only one reproduces q_replay.

CREATE TABLE stations (...);
INSERT INTO stations
  (station_id, name, latitude, longitude) VALUES
  (5, 'Station 5', 40.712776, -74.005974),
  (12, 'Station 12', 40.718261, -74.004790);

-- q_control captured row multiset (different label):
-- [[5626.688008276917]]
-- q_control:
SELECT 3959 * acos(
  ... - sin(radians(stations.latitude))
      * sin(radians((
          SELECT latitude FROM stations
          WHERE station_id = 12))))
  as distance FROM stations WHERE station_id = 5;

-- q_replay captured row multiset (authoritative):
-- [[0.38403957581396597]]
-- q_replay:
SELECT 3959 * acos(
  cos(radians(stations.latitude))
  * cos(radians((
      SELECT latitude FROM stations
      WHERE station_id = 12)))
  * cos(radians(stations.longitude)
        - radians((
            SELECT longitude FROM stations
            WHERE station_id = 12)))
  [CURSOR]
\end{verbatim}
\end{quote}

The reference continuation is \texttt{+}. It adds the product of the
latitude sine terms to the preceding cosine product, yielding a distance
of approximately 0.38404 miles. Substituting \texttt{-} instead yields
approximately 5626.68801 miles, matching the contrasting query's output rather than
the authoritative replay output. Thus, both operators admit executable
SQL, but only \texttt{+} satisfies the stated replay constraint.
The complete contrasting query supplies a competing local pattern, while
the output labels explicitly distinguish its result from the required one.

In \Cref{tab:sql-station-distance}, we report results for fillers of 24--39 tokens written as SQL comments. DeepSeek-V4.1-Flash alternates
exactly between \texttt{+} and \texttt{-} across all 16 inputs.
DeepSeek-V4-Flash-Base and DeepSeek-V4-Flash-0731 also switch between the
two operators, with 2 and 11 mismatches, respectively.

\begin{table}[tb]
\centering
\caption{Station-distance completion across 16 filler variants,
ordered by increasing filler length. The reference token is
\texttt{+}, and every mismatch predicts \texttt{-}. Green marks correct
predictions and red marks mismatches. Spaces group
consecutive predictions for readability.}
\label{tab:sql-station-distance}
\small
\setlength{\tabcolsep}{4pt}
\begin{tabular}{@{}lcrr@{}}
\toprule
Model & Top-1 sequence &
\shortstack{Mismatches/16} & $P(+)$ range \\
\midrule
DeepSeek-V4-Flash-Base (TileLang) &
\SeqOk{+}\SeqOk{+}\SeqOk{+}\SeqOk{+} \SeqBad{-}\SeqOk{+}\SeqOk{+}\SeqOk{+} \SeqOk{+}\SeqOk{+}\SeqOk{+}\SeqOk{+} \SeqOk{+}\SeqBad{-}\SeqOk{+}\SeqOk{+} & 2 & 0.329--0.569 \\
DeepSeek-V4-Flash-0731 &
\SeqBad{-}\SeqBad{-}\SeqBad{-}\SeqOk{+} \SeqBad{-}\SeqBad{-}\SeqOk{+}\SeqBad{-} \SeqBad{-}\SeqBad{-}\SeqOk{+}\SeqBad{-} \SeqBad{-}\SeqBad{-}\SeqOk{+}\SeqOk{+} & 11 & 0.123--0.638 \\
DeepSeek-V4.1-Flash &
\SeqOk{+}\SeqBad{-}\SeqOk{+}\SeqBad{-} \SeqOk{+}\SeqBad{-}\SeqOk{+}\SeqBad{-} \SeqOk{+}\SeqBad{-}\SeqOk{+}\SeqBad{-} \SeqOk{+}\SeqBad{-}\SeqOk{+}\SeqBad{-} & 8 & 0.119--0.724 \\
\bottomrule
\end{tabular}
\end{table}

%% file: Appendix/Sections/posttrained_code_completion.tex
\subsection{FP8 Example for Other DeepSeek Models}
\label{app:v4-code-posttrained}

To test whether preference reversals persist under chat prompting, we
evaluate the post-trained DeepSeek-V4-Flash-0731 and DeepSeek-V4.1-Flash
models on the FP8 example in \Cref{app:v4-code-details}. We retain the
same source-code prefix and reference continuation, and give both models
identical chat-formatted causal prefixes.\footnote{DeepSeek-V4-Flash-0731 uses the official TileLang
implementation with model parallelism four in mixed precision, with FP4 (MXFP4) routed-expert weights and FP8 for most other weights.
DeepSeek-V4.1-Flash uses vLLM with tensor parallelism eight and expert parallelism.
The comparison therefore includes differences in inference backend.}
To check whether the reversals depend on chunked compression, we also evaluate DeepSeek-V3.1-Base~\citep{deepseekv3}, which has none.

\paragraph{Chat-formatted inputs.}
We disable thinking and place the following instruction at the start of
the user message:
\begin{quote}
\small
Continue the function without changing its block-wise quantization semantics; avoid a redundant nested cast
\end{quote}
The instruction is followed by the filler and then the first 974 tokens of
the source prefix, with blank lines separating these components. The
remaining nine source tokens, \CodeIn{compute\_dtype, T.Cast(FP}, form the
assistant prefill. This preserves the completion point in
\Cref{app:v4-code-details} while presenting the task as a user request:
\begin{quote}
\footnotesize
\begin{verbatim}
[BOS][USER]{instruction}\n\n{filler}\n\n{code head}
[ASSISTANT]</think>compute_dtype, T.Cast(FP
\end{verbatim}
\end{quote}
The two displayed lines are concatenated. \CodeIn{\textbackslash n} denotes
a newline, and the bracketed role labels stand for the template's special
tokens. We score the next token after the supplied assistant prefill.
For filler length \(L\), the scored position is \(t=1004+L\), including
the chat-template tokens.

\paragraph{Filler families.}
We use families A, B, C, and D from \Cref{app:v4-code-details}, now with
\(L=23,\ldots,38\) tokens, giving 16 variants per family
(\Cref{tab:v4-posttrained-fillers}). The instruction and source code remain
fixed throughout.

\begin{table}[tb]
\centering
\caption{Filler constructions for the post-trained models. The fixed
sentences and formats follow \Cref{tab:v4-filler-constructions}, and each
family spans filler lengths \(L=23,\ldots,38\).}
\label{tab:v4-posttrained-fillers}
\small
\begin{tabular}{@{}lp{0.32\linewidth}p{0.53\linewidth}@{}}
\toprule
Family & Format & Length variation \\
\midrule
A & Triple-quoted docstring & 8--23 space-separated \CodeIn{=} characters on the second line. \\
B & Two-line Python comment & 8--23 space-separated \CodeIn{=} characters on the second line. \\
C & Shorter two-line comment & 15--30 space-separated \CodeIn{*} characters on the second line. \\
D & Single-line comment & 16 natural-language rephrasings, with no repeated-symbol line. \\
\bottomrule
\end{tabular}
\end{table}

\paragraph{Results.}
We report \(\Delta=P(8)-P(32)\) in \Cref{tab:v4-filler-family-results} (right) and \Cref{tab:v41-v31-probabilities} (left), in the same format as the base-model results.
Across these runs, the full-vocabulary top-1 prediction is always either
\CodeIn{8} or \CodeIn{32}, so the sign of \(\Delta\) also determines
whether the reference token is selected.

DeepSeek-V4-Flash-0731 exhibits a four-token cycle for A and C: it selects
\CodeIn{8} exactly when \(L\bmod4=0\), giving 4/16 correct predictions
(25\%) in each family. B follows the same pattern except at \(L=32\),
where it selects \CodeIn{32}, reducing accuracy to 3/16 (18.75\%).
With D, it selects \CodeIn{32} at every length.

DeepSeek-V4.1-Flash instead alternates between the two continuations for A, B,
and C, selecting \CodeIn{8} at odd filler lengths and \CodeIn{32} at
even lengths. Each family therefore gives 8/16 correct predictions (50\%).
For D, the model selects \CodeIn{8} at all seven lengths from 23 to 29,
then follows the same odd--even alternation, yielding 11/16 (68.75\%).
These periods describe the top-1 decisions over the measured range,
and the probability differences also vary within each residue group.

The reversals can be large. Increasing the A filler from 23 to
24 tokens changes \(\Delta\) from \(-0.910\) to \(+0.615\) in
DeepSeek-V4-Flash-0731, while the same change moves DeepSeek-V4.1-Flash from \(+0.124\)
to \(-0.555\). Thus, the sensitivity seen in the base-model example
persists with post-trained models and chat-formatted inputs, although
its period and dependence on filler content differ across models.
The reported accuracies summarize shifted variants of this example under
the selected prompt, rather than performance on a code-completion benchmark.

\begin{table}[tb]
\centering
\caption{DeepSeek-V4.1-Flash (left) and DeepSeek-V3.1-Base (right) results on the FP8 example. Entries, shading, and ``Tokens'' follow \Cref{tab:v4-filler-family-results}. DeepSeek-V4.1-Flash receives the chat-formatted prompt with fillers of 23--38 tokens, and thin rules separate pairs of lengths, matching its stride of two. DeepSeek-V3.1-Base receives the same raw-prefix inputs as DeepSeek-V4-Flash-Base, grouped in fours as in \Cref{tab:v4-filler-family-results}.}
\label{tab:v41-v31-probabilities}
\small
\setlength{\tabcolsep}{3.5pt}
\renewcommand{\arraystretch}{1.05}
\begin{tabular}{@{}rrrrr@{\hspace{14pt}}c@{\hspace{14pt}}rrrrr@{}}
\toprule
\multicolumn{5}{c}{DeepSeek-V4.1-Flash (chat prompt)} && \multicolumn{5}{c}{DeepSeek-V3.1-Base (raw prefix)} \\
\cmidrule(lr){1-5}\cmidrule(lr){7-11}
Tokens & A & B & C & D && Tokens & A & B & C & D \\
\midrule
23 & \CorrectOrder{+0.124} & \CorrectOrder{+0.245} & \CorrectOrder{+0.358} & \CorrectOrder{+0.555} && 24 & \CorrectOrder{+0.316} & \CorrectOrder{+0.362} & \CorrectOrder{+0.355} & \CorrectOrder{+0.242} \\
24 & \WrongOrder{-0.555} & \WrongOrder{-0.555} & \WrongOrder{-0.462} & \CorrectOrder{+0.124} && 25 & \CorrectOrder{+0.477} & \CorrectOrder{+0.315} & \CorrectOrder{+0.359} & \WrongOrder{-0.003} \\
\cmidrule(lr){1-5}
25 & \CorrectOrder{+0.124} & \CorrectOrder{+0.358} & \CorrectOrder{+0.462} & \CorrectOrder{+0.462} && 26 & \CorrectOrder{+0.600} & \CorrectOrder{+0.316} & \CorrectOrder{+0.348} & \CorrectOrder{+0.203} \\
26 & \WrongOrder{-0.555} & \WrongOrder{-0.635} & \WrongOrder{-0.124} & \CorrectOrder{+0.124} && 27 & \CorrectOrder{+0.274} & \CorrectOrder{+0.486} & \WrongOrder{-0.026} & \CorrectOrder{+0.140} \\
\cmidrule(lr){1-5}\cmidrule(lr){7-11}
27 & \CorrectOrder{+0.358} & \CorrectOrder{+0.555} & \CorrectOrder{+0.462} & \CorrectOrder{+0.245} && 28 & \CorrectOrder{+0.430} & \CorrectOrder{+0.363} & \CorrectOrder{+0.303} & \CorrectOrder{+0.011} \\
28 & \WrongOrder{-0.555} & \WrongOrder{-0.124} & \WrongOrder{-0.635} & \CorrectOrder{+0.124} && 29 & \CorrectOrder{+0.303} & \CorrectOrder{+0.259} & \CorrectOrder{+0.273} & \CorrectOrder{+0.463} \\
\cmidrule(lr){1-5}
29 & \CorrectOrder{+0.358} & \CorrectOrder{+0.462} & \CorrectOrder{+0.124} & \CorrectOrder{+0.124} && 30 & \CorrectOrder{+0.240} & \CorrectOrder{+0.062} & \CorrectOrder{+0.081} & \CorrectOrder{+0.140} \\
30 & \WrongOrder{-0.462} & \WrongOrder{-0.555} & \WrongOrder{-0.245} & \WrongOrder{-0.245} && 31 & \CorrectOrder{+0.392} & \CorrectOrder{+0.316} & \CorrectOrder{+0.464} & \CorrectOrder{+0.236} \\
\cmidrule(lr){1-5}\cmidrule(lr){7-11}
31 & \CorrectOrder{+0.358} & \CorrectOrder{+0.358} & \CorrectOrder{+0.358} & \CorrectOrder{+0.555} && 32 & \CorrectOrder{+0.346} & \CorrectOrder{+0.539} & \CorrectOrder{+0.026} & \WrongOrder{-0.054} \\
32 & \WrongOrder{-0.635} & \WrongOrder{-0.704} & \WrongOrder{-0.462} & \WrongOrder{-0.245} && 33 & \WrongOrder{-0.065} & \CorrectOrder{+0.284} & \CorrectOrder{+0.139} & \CorrectOrder{+0.478} \\
\cmidrule(lr){1-5}
33 & \CorrectOrder{+0.462} & \CorrectOrder{+0.245} & \CorrectOrder{+0.462} & \CorrectOrder{+0.124} && 34 & \CorrectOrder{+0.139} & \CorrectOrder{+0.146} & \CorrectOrder{+0.109} & \CorrectOrder{+0.325} \\
34 & \WrongOrder{-0.555} & \WrongOrder{-0.462} & \WrongOrder{-0.245} & \WrongOrder{-0.462} && 35 & \CorrectOrder{+0.398} & \CorrectOrder{+0.453} & \CorrectOrder{+0.062} & \CorrectOrder{+0.017} \\
\cmidrule(lr){1-5}\cmidrule(lr){7-11}
35 & \CorrectOrder{+0.555} & \CorrectOrder{+0.555} & \CorrectOrder{+0.358} & \CorrectOrder{+0.734} && 36 & \CorrectOrder{+0.166} & \CorrectOrder{+0.081} & \CorrectOrder{+0.276} & \CorrectOrder{+0.370} \\
36 & \WrongOrder{-0.635} & \WrongOrder{-0.555} & \WrongOrder{-0.245} & \WrongOrder{-0.358} && 37 & \CorrectOrder{+0.388} & \CorrectOrder{+0.498} & \CorrectOrder{+0.027} & \CorrectOrder{+0.347} \\
\cmidrule(lr){1-5}
37 & \CorrectOrder{+0.462} & \CorrectOrder{+0.635} & \CorrectOrder{+0.462} & \CorrectOrder{+0.358} && 38 & \CorrectOrder{+0.397} & \CorrectOrder{+0.324} & \CorrectOrder{+0.396} & \CorrectOrder{+0.143} \\
38 & \WrongOrder{-0.635} & \WrongOrder{-0.635} & \WrongOrder{-0.245} & \WrongOrder{-0.358} && 39 & \CorrectOrder{+0.409} & \CorrectOrder{+0.197} & \CorrectOrder{+0.404} & \CorrectOrder{+0.342} \\
\bottomrule
\end{tabular}
\end{table}
\FloatBarrier

\paragraph{A model without chunked compression.}
DeepSeek-V3.1-Base compresses each token's keys and values into a low-rank latent vector, but it does not pool tokens into compressed entries, so its attention has no stride.\footnote{DeepSeek-V3.1-Base (revision \texttt{d3d4eafd}) uses the official DeepSeek-V3 inference implementation with model parallelism eight and FP8 weights, because its release does not include inference code of its own. The comparison with DeepSeek-V4-Flash-Base therefore also differs in inference backend. With the implementation's BF16 output projection, 12 of the 64 comparisons between \CodeIn{8} and \CodeIn{32} were exact ties, so we compute only the final vocabulary projection in FP32 and leave the rest of the model unchanged. The BF16 outputs replicated exactly on a second machine, and the FP32 results come from a single run.}
The two models share a tokenizer, so we give DeepSeek-V3.1-Base the raw-prefix inputs of \Cref{app:v4-code-details}, token-for-token identical to those of DeepSeek-V4-Flash-Base.
In \Cref{tab:v41-v31-probabilities} (right), we see no periodic reversal. DeepSeek-V3.1-Base ranks \CodeIn{32} above \CodeIn{8} at only 4 of the 64 inputs, each with \(|\Delta|<0.07\), compared with 28 for DeepSeek-V4-Flash-Base.
The difference lies in the dependence on position rather than in overall confidence. In family A, the mean of \(P(8)\) is 0.609 for DeepSeek-V3.1-Base and 0.589 for DeepSeek-V4-Flash-Base, but only DeepSeek-V4-Flash-Base changes its preference systematically with the scored position. Across the four families, it errs at 15 of 16 inputs with \(t\bmod4=0\) and 13 of 16 with \(t\bmod4=3\), whereas the four errors of DeepSeek-V3.1-Base fall on three different residues.
This contrast is consistent with chunked compression being the source of the sensitivity, although the two models also differ in architecture, training, and inference backend.

%% file: Appendix/Sections/small_model_niah.tex
\section{Additional Details on the Needle-in-a-Haystack Evaluations}
\label{app:niah}

Here we describe the needle-in-a-haystack evaluations of \Cref{sec:v4-niah} and \Cref{sec:controlled}. We first give the construction and per-residue results for the DeepSeek models, and then the retrieval task and the two padding protocols that we use for our pretrained models.

\subsection{NIAH Evaluation for the DeepSeek-V4 Family}
\label{app:v4-niah-details}
Here we give the setup and per-residue results behind \Cref{sec:v4-niah}. As in \Cref{sec:modsens}, a token at zero-based position \(t\) has phase \(t\bmod S\), and the modulo-eight residues that we use to report the results are distinct from this phase. The construction shifts the target's phase across residue groups while fixing the prompt length, the query position, and the mean and variance of the target position, so that the groups differ in phase but not in these position statistics.

\paragraph{Position coordinate.}
DeepSeek-V4's CSA compression kernel uses eight-token windows at stride four,
\(W=8,S=4\)~\citep{deepseek2026v4}. We sample eight source-key position
residues \(r=t_K\bmod8\), where \(t_K\) is the target key's zero-based token
index, and call the prompts sharing a residue a residue group. For the DeepSeek-V4 models,
the eight residue groups cover two cycles of the native CSA phase \(t_K\bmod4\).
DeepSeek-V4.1-Flash compresses with stride two~\citep{deepseek2026v41}, so the eight residue groups cover
four phase cycles. The modulus eight defines the residue groups and does not
denote a compression stride of eight.

\paragraph{Data generation.}
The construction below describes the 2,048-prompt set shared by the four
DeepSeek-V4 variants. The DeepSeek-V4.1-Flash evaluation draws ten times as
many prompts from the same distribution, using ten times as many bindings,
where a binding is one fixed set of 16,000 key--value pairs. This gives
2,560 prompts per residue group, or 20,480 in total
(\Cref{tab:v4-niah-residue-accuracy}).
Each prompt contains a sequence of key--value records followed by a query asking for the value associated with one target key.
Each prompt contains 128,000 tokens. An intro with the instruction and three format examples is followed by a natural prefill of \(\ell\) tokens, which sets the target's residue, and by a padding block of about 63,000 filler tokens whose length is a multiple of eight.
\begin{CJK*}{UTF8}{gbsn}
Next come 16,000 key--value pairs, with four tokenizer tokens per pair in the form \texttt{key:value。}, and then a natural postfill of \(48-\ell\) tokens, tail padding, and the query.
\end{CJK*}
The target is the 249th pair (zero-based index 248), so its key starts at token \(63{,}976+\ell\), which is 64,000 on average. The evaluation set contains 256 prompts in each of the eight residue groups, for 2,048 prompts in total.

The key--value candidates are generated deterministically using the DeepSeek-V4 tokenizer.
We select candidate strings with tokenizer IDs of at least 10,000 as a heuristic for less common tokens.
Candidate strings are also alphabetic, contain at most 16 Unicode characters, and occupy exactly one tokenizer token.
We additionally require tokenization stability at record and query boundaries.
We construct 256 deterministic bindings, each containing 16,000 key--value pairs, and use each binding once in every residue group.
The same evaluation prompts and prompt IDs are reused across the four
DeepSeek-V4 variants, giving comparisons on identical prompts within that family.

Natural-language filler is drawn from the training split of the \texttt{agentlans/\allowbreak high-\allowbreak quality-\allowbreak english-\allowbreak sentences} dataset at revision \texttt{5dfcb23}. Sentences are shuffled and grouped by exact tokenizer length. The evaluation set uses 2,048 distinct sentences to form 1,024 sentence pairs. Each pair appears twice with reversed prefill/postfill roles. This balances the residue schedule and the target's phase. The filler text is not identical across all eight residue groups.

For residue \(r\), the natural prefill length \(\ell\) satisfies \(8 \leq \ell \leq 40\) and \(\ell \equiv r \pmod 8\). Every residue group contains 256 prompts with
\[
\mathbb{E}[\ell]=24,
\qquad
\mathbb{E}[\ell^2]=604.
\]
The natural postfill has length \(48-\ell\), so the two natural filler segments always contribute exactly 48 tokens. The finite histograms are
\[
\begin{aligned}
H_0 &= \{(8,1),(16,52),(24,150),(32,52),(40,1)\},\\
H_1 &= \{(9,2),(17,68),(25,146),(33,40)\},\\
H_2 &= \{(10,4),(18,84),(26,140),(34,28)\},\\
H_3 &= \{(11,7),(19,101),(27,129),(35,19)\},\\
H_4 &= \{(12,12),(20,116),(28,116),(36,12)\},
\end{aligned}
\]
with \(H_{8-r}\), for \(r\in\{1,2,3\}\), obtained by replacing each length \(\ell\) in \(H_r\) with \(48-\ell\).

Each prompt is assembled as
\[
[\text{intro};\text{natural prefill};\text{body prefix};
\text{pair body};\text{natural postfill};\text{tail};
\text{query separator};\text{query}].
\]
The body prefix is constructed so that the target-pair start is
\[
s = 64{,}000+\ell-24,
\]
which gives a mean start position of 64,000 and variance \(604-24^2=28\) in every residue group while preserving the desired modulo-eight residue. The tail is then filled to reach exactly 128,000 prompt tokens. The query's token position is fixed across residue groups. The terminal query is:

\begin{quote}
\texttt{Now answer the question using the key-value pairs above. What is the value for key <KEY>? Answer with only the value. The value for key <KEY> is }
\end{quote}

Generation is followed by independent retokenization. We verify the prompt length, all component boundaries, pair-record tokenization, target span, source-key residue, and gold answer before saving the evaluation set.

\paragraph{Scope of the phase comparison.}
Changing the outer filler shifts the target and distractor records together, and each four-token record spans a whole number of compression strides for both DeepSeek-V4 and DeepSeek-V4.1. These comparisons hold the query position and the first two moments of target position fixed, but do not vary the target's phase independently of the distractors. The In-sequence Padding evaluation in \Cref{sec:setup} and \Cref{app:phase-balanced-evaluation} tests variable gaps between records in our pretrained models.

\paragraph{Inference protocol and metrics.}
For each prompt in the set shared by the four DeepSeek-V4 models, we set
\[
\{\texttt{temperature}=0,\ \texttt{max\_tokens}=8\}.
\]
The task uses direct-answer prompting without chain-of-thought generation.

Let \(g_i\) be the gold value and \(y_i\) the raw completion. The primary metric is answer-prefix accuracy:
\[
\mathrm{Acc}_{\mathrm{prefix}}
=
\frac{1}{N}\sum_{i=1}^{N}
\mathbf{1}\!\left[
\operatorname{strip}(y_i)\text{ begins with }g_i
\right],
\]
where \(\operatorname{strip}\) removes leading and trailing whitespace.
This metric tolerates punctuation or parser-added text after an otherwise correct answer. It is computed from the generated answer and is distinct from full-vocabulary next-token accuracy. In \Cref{fig:v4_niah}, the shaded bands are two-sided 95\% Wilson intervals ($z=1.96$) for each residue group.

\begin{table*}[tb]
\centering
\caption{Answer-prefix accuracy (\%) by source-key residue
\(r=t_K\bmod 8\) for the five DeepSeek models shown in
\Cref{fig:v4_niah}.
\(n\) is the number of prompts in each residue group, and
\(\Delta_{\mathrm{res}}=\max_r\mathrm{Acc}_r-\min_r\mathrm{Acc}_r\) is
reported in percentage points. Accuracies are rounded to two decimal
places from exact integer correct counts.}
\label{tab:v4-niah-residue-accuracy}

\begingroup
\small
\setlength{\tabcolsep}{3pt}
\begin{tabular}{@{}lrrrrrrrrrr@{}}
\toprule
Model & \(n\) & \(r=0\) & \(r=1\) & \(r=2\) & \(r=3\) &
\(r=4\) & \(r=5\) & \(r=6\) & \(r=7\) &
\(\Delta_{\mathrm{res}}\) \\
\midrule
DeepSeek-V4-Flash-Base & 256 & 43.75 & 69.53 & 66.02 & 29.69 & 44.92 & 69.92 & 64.45 & 32.81 & 40.23 \\
DeepSeek-V4-Flash-0731 & 256 & 80.47 & 88.67 & 98.05 & 96.48 & 81.25 & 90.23 & 99.61 & 96.09 & 19.14 \\
DeepSeek-V4-Pro-Base   & 256 & 60.94 & 88.28 & 78.91 & 56.25 & 62.50 & 91.02 & 80.47 & 58.20 & 34.77 \\
DeepSeek-V4-Pro-0813   & 256 & 78.13 & 89.84 & 91.41 & 77.34 & 77.34 & 89.45 & 92.19 & 79.69 & 14.84 \\
DeepSeek-V4.1-Flash & 2560 & 95.00 & 89.22 & 95.12 & 89.69 & 94.69 & 90.31 & 95.31 & 89.73 &  6.09 \\
\bottomrule
\end{tabular}
\endgroup
\end{table*}

\paragraph{Results across model variants.}
In \Cref{tab:v4-niah-residue-accuracy}, we report the per-residue results. We evaluate DeepSeek-V4-Flash-Base, DeepSeek-V4-Flash-0731, DeepSeek-V4-Pro-Base, and DeepSeek-V4-Pro-0813 on the same 2,048 prompts. Averaged over the eight residue groups, their answer-prefix accuracies are \(52.6\%\), \(91.4\%\), \(72.1\%\), and \(84.4\%\), respectively. Their best--worst residue gaps are 40.23, 19.14, 34.77, and 14.84 percentage points, respectively. Because the eight residue groups span two cycles of DeepSeek-V4's stride of four, a four-token recurrence should make residues \(r\) and \(r+4\) nearly equal. For each model, we therefore average the absolute accuracy difference between residues \(r\) and \(r+4\) over \(r\in\{0,1,2,3\}\). This difference ranges from 1.07 to 1.95 percentage points across the four models, consistent with an approximate four-token recurrence in the accuracy pattern across residues. These measurements describe the positional pattern but do not by themselves isolate its mechanistic source. DeepSeek-V4.1-Flash, evaluated on 2,560 prompts per residue group, has a mean accuracy of \(92.38\%\) and a best--worst residue gap of 6.09 percentage points.

\subsection{NIAH Evaluation for Qwen3-based Models}
\label{app:small-model-niah}

Here we specify the retrieval task that we use to evaluate our pretrained models in \Cref{sec:controlled} and the two padding protocols, Prefix Padding and In-sequence Padding, that vary the phase of the target key. Within each protocol, the query position is fixed across phases, and the target's expected depth and expected distance to the query are the same for every phase.

We construct token-exact animal--name retrieval tasks. Each logical example independently permutes a curated pool of 64 animal keys and a disjoint pool of 64 human names, then pairs them to obtain a random one-to-one mapping. Each key and value is one token under the Qwen3 tokenizer. The target is the pair at zero-based index $j_*=24$ in the 64-pair body.

Both evaluations use a task instruction, a mapping header and all 64 pairs,
and a recall prefix followed by eight demonstrations and the target animal
key. Prefix Padding places natural-language filler before and after the
mapping body. In-sequence Padding instead uses fixed outer guards and
spacers within the body, as detailed below. The demonstrations are distinct
pairs drawn from the body, excluding the target pair. Body records and
demonstrations use
\texttt{\{key\}\{value\}\textbackslash n}: the two fields have no intervening
delimiter token, but their token strings contain leading spaces, giving
records such as \texttt{ ant Alice}. The prompt ends at the target key.
The associated name is the gold next token and is not appended to the input.
We show an abridged example in \Cref{fig:small-model-niah-example}.

\begin{figure}[tb]
\centering
\begin{tcolorbox}[
  title={Example retrieval prompt (abridged)},
  colback=black!2, colframe=black!45, colbacktitle=black!8, coltitle=black,
  fonttitle=\bfseries, boxrule=0.5pt, arc=1mm,
  left=8pt, right=8pt, top=6pt, bottom=6pt,
  before skip=0pt, after skip=0pt]
\begingroup
\small\ttfamily\setlength{\parindent}{0pt}
Remember each animal's assigned human name.\par
\textnormal{\itshape[ prefix filler omitted ]}\par\smallskip
\hspace*{0.6em}Animal-name pairs:\par
\hspace*{0.6em}snake Alex\par
\hspace*{0.6em}beetle Carl\par
\hspace*{0.6em}fox Lucy\par
\hspace*{0.6em}$\cdots$\par
\hspace*{0.6em}bear Susan\par
\hspace*{0.6em}shark Helen\par
\hspace*{0.6em}$\cdots$\par
\hspace*{0.6em}crow Nancy\par
\textnormal{\itshape[ suffix filler omitted ]}\par\smallskip
\hspace*{0.6em}Recall the requested animal names:\par
\hspace*{0.6em}camel Anna\par
\hspace*{0.6em}spider Lisa\par
\hspace*{0.6em}chicken Molly\par
\hspace*{0.6em}bee Alan\par
\hspace*{0.6em}sheep Adam\par
\hspace*{0.6em}lizard Emma\par
\hspace*{0.6em}llama Ethan\par
\hspace*{0.6em}dog Eric\par
\hspace*{0.6em}bear\par
\endgroup
\tcblower
\small\textbf{Gold next token (not part of the prompt):}
\texttt{Susan}, including its leading space.
\end{tcolorbox}
\caption{\textbf{An animal--name retrieval example with Prefix Padding.}
Natural-language fillers are abbreviated, and $\cdots$ denotes omitted mapping records.
All eight demonstrations and the final query are shown. The target
mapping, \texttt{bear Susan}, occurs at body index 24 and is excluded from
the demonstrations.}
\label{fig:small-model-niah-example}
\end{figure}

\paragraph{Scoring.}
We score the gold name using logits at the final input token. Candidate-restricted accuracy takes the argmax over all 64 candidate names (uniform-guess chance $1/64$), while full-vocabulary accuracy takes the argmax over the entire tokenizer vocabulary. We also record the gold-token probability and negative log-likelihood.

\subsection{Padding Protocols for Qwen3-based Models}
\label{app:phase-balanced-evaluation}

\paragraph{Source phase and matched examples.}
Chunked compression gives each source token a \emph{phase}: its alignment relative to the regularly spaced right boundaries of compression windows. For a model with compression stride $S$, phase is the source-token position modulo $S$. To compare models with different strides, we use the common sampling coordinate modulo $M=\operatorname{lcm}(4,6,8,12)=24$. For each of 1,000 logical examples, we construct all $M$ source-key residues, giving 24,000 prompts. The 24 prompts of a logical example share the mapping, pair order, and demonstration indices and differ only in their fillers or spacers, whose offsets are sampled separately for each residue. Let $\ell_{\rm desc}$ and $\ell_{\rm hdr}$ be the instruction and mapping-header lengths, $\ell_{\rm pair}$ the token count per pair, and $L_{\rm pre}$ the effective inserted-token offset before the target, which the two protocols below produce in different ways. All lengths are measured with the tokenizer, including template whitespace. We define
\begin{equation}\label{eq:niah-source-phase}
 c=\ell_{\rm desc}+\ell_{\rm hdr}+j_*\ell_{\rm pair},\qquad
 t_K=c+L_{\rm pre},\qquad r=t_K\bmod M.
\end{equation}
Requesting source-key residue $r$ therefore requires $L_{\rm pre}\equiv r-c\pmod M$. The target value is at $t_K+1$ in the concatenated grammar, but neither the value position nor the final query key defines the swept phase. Since each tested stride $S$ divides $M$, the model's phase is $r\bmod S$, and phase-wise results group prompts by $r\bmod S$, with equal numbers of prompts in each group. We also keep the results for each of the 24 residues alongside these stride-specific summaries.

\paragraph{Sampling offsets with equal expected filler lengths.}
To vary phase without changing the expected insertion depth, we construct
the integer-valued length variable
\begin{equation}\label{eq:niah-balanced-length-law}
 X=M+U_1+U_2+U_3+F,\qquad
 U_i\overset{\rm iid}{\sim}\operatorname{Unif}\{0,\ldots,M-1\},\quad
 F\sim\operatorname{Unif}\{0,\ldots,M+3\},
\end{equation}
with all four draws independent. For each requested residue $r$, the sampler
draws directly from the conditional length distribution and assigns the
complementary suffix filler length:
\begin{equation}\label{eq:niah-complementary-lengths}
 L_{\rm pre}\sim\mathcal L(X\mid X\equiv r-c\pmod M),\qquad
 L_{\rm post}=6M-L_{\rm pre}.
\end{equation}
For $M=24$, both lengths lie between 24 and 120 tokens, and their sum is
exactly 144 tokens in every prompt. Moreover, for every residue $r$,
\begin{equation}\label{eq:niah-balanced-moments}
 \begin{aligned}
 \mathbb E[L_{\rm pre}\mid r]
   &=\mathbb E[L_{\rm post}\mid r]=3M=72,\\
 \operatorname{Var}(L_{\rm pre}\mid r)
   &=\operatorname{Var}(L_{\rm post}\mid r)
     =\frac{M^2+2M+3}{3}=209.
 \end{aligned}
\end{equation}
Conditioning preserves these moments because every term in the expansion of $X$ or $X^2$ involves at most two random summands and so leaves at least one of $U_1,U_2,U_3$ unused, which makes the residue of the sum uniform regardless of the variables in the term. The first and second raw moments are therefore the same under every residue, and the complementary construction gives the suffix filler length the same mean and variance.

These equalities concern the sampling distribution, so the 1,000 random
draws per residue need not have identical empirical means. By contrast, the total
filler length is fixed for every example. Together with the fixed template
and body lengths, this fixes the total prompt length and the absolute query
and scoring positions across phases. The design therefore matches expected
source depth and source-to-query distance while changing the source key's phase.

\paragraph{Prefix Padding.}
\emph{Prefix Padding} places a natural-language prefix filler of the
sampled length $L_{\rm pre}$ from \Cref{eq:niah-complementary-lengths}
directly before the mapping body and a natural-language suffix filler of
length $L_{\rm post}=6M-L_{\rm pre}$ after it. Varying
$L_{\rm pre}$ moves the packed mapping body, including the target pair,
while the complementary suffix filler keeps the query and scoring positions
fixed.
Because the conditional mean of $L_{\rm pre}$ is the same for every source
phase, no phase is systematically assigned an earlier or later target.

The prefix and suffix fillers come from the training split of the high-quality-english-sentences dataset~\citep{agentlans_hqes}. Each source sentence is rendered with a final \texttt{\textbackslash n\textbackslash n}, which is included in its token length. We build exact-length bins for every length from 24 through 120, retaining at most 64 distinct sentences per bin, and exclude sentences containing any curated key or value token ID. Sentences are used without padding or truncation. The generator first samples a length using \Cref{eq:niah-complementary-lengths}, then samples a sentence uniformly from that length bin, and draws the suffix filler from the complementary bin as a different sentence. Every length from 24 through 120 must have a nonempty bin, so missing bins do not change the conditional length law. For reproducibility, we use seed 20260720 for the source shuffle and seed 20260722 for mapping, demonstration, and filler generation, and each example and residue has its own seed.

\paragraph{In-sequence Padding.}
Prefix Padding moves all records together, so \emph{In-sequence Padding} instead produces the offset inside the mapping body. It keeps the same conditional draw as a latent \emph{reference prefix length}, denoted $\widetilde L_{\rm pre}$, places a fixed 24-token guard before the mapping header and another fixed 24-token guard after the body, and defines
\[
  x=\widetilde L_{\rm pre}-24.
\]
The 64-pair body has 63 inter-record gaps. Let $g_0,\ldots,g_{62}\in\{0,\ldots,4\}$ be the numbers of one-token spacers inserted after the corresponding records. Each spacer is the fixed one-token Unicode character \texttt{U+FF5C} (FULLWIDTH VERTICAL LINE), not the tokenizer's special padding token. For the target at index $j_*=24$, we sample uniformly from the feasible bounded compositions satisfying
\begin{equation}\label{eq:niah-in-sequence-padding}
  \sum_{i=0}^{23}g_i=x,\qquad
  \sum_{i=24}^{62}g_i=126-x.
\end{equation}
Thus the effective offset before the target is $24+x=\widetilde L_{\rm pre}$, so the target obeys the same phase equation and phase-conditioned length law as Prefix Padding.

The total number of in-body spacers is always 126, so the body boundaries, query position, scoring position, and total sequence length are fixed across phases. In-sequence Padding thus changes phase by redistributing bounded gaps around the target rather than by moving the entire body with a variable-length prefix filler. The guards and spacers contribute $24+126+24=174$ tokens, compared with 144 filler tokens in Prefix Padding. Thus the query is fixed across phases within each protocol but occurs 30 tokens later in In-sequence Padding. The source-key position has the same conditional distribution in both. The two protocols also reuse the same logical mappings, pair order, demonstrations, logical example indices, and 1,000 examples per source residue, so they test the same retrieval task under two distinct ways of producing the phase-balanced offset.

\paragraph{Tokenization checks.}
Because phase is defined on exact token positions, we verify that every prompt tokenizes as constructed. For Prefix Padding, we check exact text/token round trips of each filler sentence and its additive joins with adjacent prompt components, and require the assembled prompt and the prompt followed by the gold answer to each re-encode exactly to their constructed token IDs. For In-sequence Padding, we independently retokenize every assembled prompt and verify the guard lengths, all 63 gap lengths and their sum, the two partition sums in \Cref{eq:niah-in-sequence-padding}, the target span and phase, the fixed query and scoring positions, and the gold answer.

%% file: Appendix/Sections/controlled_pretraining_details.tex
\section{Additional Details on the Architecture and Pretraining of Qwen3-based Models}\label{app:scratch}
Here we describe the architecture, attention memory, positional encoding, and training settings behind our pretrained models in \Cref{sec:controlled}. We report the complete sweep in \Cref{app:complete-controlled-sweep}.

\subsection{Backbone and Compression Kernels}

\paragraph{Backbone and reference compression kernel.}
The Qwen3-0.6B backbone has 28 layers, hidden width 1,024, 16 query heads, and head dimension $d_h=128$. We retain its feed-forward blocks, residual connections, and normalization, so that chunked KV compression is the only substantial architectural change. The reference compression kernel uses non-overlapping windows with $W=S=8$ and one KV head per layer (a key--value head, as opposed to the 16 query heads). Its K/V branches are separate and use scalar gates: payload maps are shared across offsets ($C_i^B=C^B$), while gate maps and biases depend on the offset. The same configuration is used in every layer. The ablations vary KV-head count, pooling, parameter sharing, gate dimensionality, positional bias, window overlap, positional encoding, normalization, and local-memory policy.

\paragraph{Gate granularity and parameter sharing.}
Using the notation of \Cref{sec:setup}, scalar gates have \(d_g=1\), with one gate score per token broadcast across all projected dimensions. Vector gates have \(d_g=d_h\), allowing projected dimensions to prefer different window offsets. Separate K/V branches use their own payload maps, gate maps, and positional biases. Tied K/V branches share all three and produce a common compressed representation \(\widehat K_j=\widehat V_j\), before the normalization and positional transformations described below. Within each branch, payload and gate maps can be shared independently across the whole window, within groups of offsets, or remain offset-specific. Sharing these maps does not require sharing positional biases across offsets. Uniform averaging fixes each gate-score coordinate at \(1/W\) in a full window. The evaluated configurations are listed in \Cref{tab:controlled-sweep-results}.

\paragraph{DeepSeek-style compression kernel.}
Combining these options recovers DeepSeek-V4's CSA compression kernel~\citep{deepseek2026v4}, which lets us test whether the kernel studied in \Cref{sec:modsens} produces \name\ in our backbone. It uses \(W=8\), \(S=4\), vector gates, and tied K/V branches. The preceding four-token half, with offsets \(i\in\{0,1,2,3\}\), uses the payload and gate maps \((\mC^{(0)},\mZ^{(0)})\). The current four-token half, with offsets \(i\in\{4,5,6,7\}\), uses a separate pair \((\mC^{(1)},\mZ^{(1)})\). Each pair is shared within its half. Positional biases remain offset-specific and are shared between the tied K/V branches. In a full window, one softmax spans all eight offsets for each gate dimension.

\subsection{Windows, Attention Memory, and Positional Encoding}

Each kernel above produces one compressed entry per window. To use these entries in attention, we specify when each entry is emitted, which entries a query sees alongside its local window, and at which position each entry is rotated.

\paragraph{Window indexing and initial boundary.}
Token positions start at zero, and compressed entries use one-based indices
$j\geq1$. Entry $j$ closes at $jS-1$ and pools offsets $0\leq i<W$
at source positions $jS-W+i$. When $W>S$, the first entries may extend
before the sequence start. These entries are emitted as soon as their
stride block closes. Negative positions are masked from the softmax, which
normalizes over the valid offsets only. The corresponding payload terms
are zero. Thus the full-window equation in \Cref{sec:setup} applies away
from this initial boundary.

\paragraph{Attention memory and visibility.}
For one query head, let $q_t\in\mathbb R^{d_h}$ be its query after
normalization and positional encoding. Write $k_u,v_u$ for local entries
and $\hat k_j,\hat v_j$ for compressed entries after the branch-specific
transformations below. With local-window width $T=16$, the visible raw
positions are $\mathcal L_t=\{\max(0,t-T+1),\ldots,t\}$. Let
$\mathcal J_t$ index the visible compressed entries and
$n_t=|\mathcal J_t|+|\mathcal L_t|$. Stacking vectors as rows gives
\begin{equation}\label{eq:controlled-attention-memory}
 \mathbf K_t=
 \begin{bmatrix}(\hat k_j^\top)_{j\in\mathcal J_t}\\
                 (k_u^\top)_{u\in\mathcal L_t}\end{bmatrix},
 \qquad
 \mathbf V_t=
 \begin{bmatrix}(\hat v_j^\top)_{j\in\mathcal J_t}\\
                 (v_u^\top)_{u\in\mathcal L_t}\end{bmatrix}
 \in\mathbb R^{n_t\times d_h}.
\end{equation}
One softmax normalizes scores jointly across the two memories:
\begin{equation}\label{eq:controlled-joint-attention}
 a_t=\operatorname{Softmax}\!\left(\frac{\mathbf K_tq_t}{\sqrt{d_h}}\right)
 \in\mathbb R^{n_t},
 \qquad o_t=\mathbf V_t^\top a_t\in\mathbb R^{d_h}.
\end{equation}
Each KV head supplies the memory for its corresponding group of query
heads, which retain their own queries and attention distributions.

With a fixed-width local window, compressed visibility can follow one of two policies. The first delays visibility until the entry's closure leaves the local window: $\mathcal J_t=\{j\geq1:jS-1\leq t-T\}$. The second admits every causally closed entry, $\mathcal J_t=\{j\geq1:jS-1\leq t\}$, allowing compressed and local entries to represent overlapping source tokens. Our fixed-width models use the second policy, with an exact 16-token raw window. The open-tail local memory instead admits entries as soon as they close but retains raw tokens only in the current incomplete stride block, so its raw tail is empty at the final token of each stride block.

\paragraph{Positional encoding and K/V tying.}
Let $R_p\in\mathbb R^{d_h\times d_h}$ denote the rotary transformation at
position $p$. In the reference setting, queries and local keys receive
per-head RMSNorm before RoPE, and compressed keys are normalized and rotated
after pooling:
\begin{equation}\label{eq:controlled-rope}
 \begin{aligned}
 q_t&=R_t\operatorname{RMSNorm}_Q(W_Qh_t),\qquad
 k_u=R_u\operatorname{RMSNorm}_K(W_Kh_u),\\
 \hat k_j&=R_{(j-1)S}\operatorname{RMSNorm}_K(\widehat K_j),
 \end{aligned}
\end{equation}
where $W_Q,W_K\in\mathbb R^{d_h\times d}$ are the query and local-key
projections for the relevant heads. Under the one-based entry indexing above, position $(j-1)S$ is the start of the compressed window's newest stride block.
For an even number $0<d_r\leq d_h$ of rotary dimensions,
the frequencies are constructed within that subspace:
\begin{equation}\label{eq:controlled-rope-frequencies}
 \omega_\ell=\theta^{-2\ell/d_r},\qquad
 \ell=0,\ldots,d_r/2-1,\qquad \theta=10^6.
\end{equation}
The transformation $R_p$ rotates the final $d_r$ coordinates with angles
$p\omega_\ell$, using Qwen3's split-half pairing, and leaves the first
$d_h-d_r$ coordinates unchanged. The reference uses $d_r=d_h=128$, and
partial-RoPE variants use $d_r=16$, a rotary fraction of $1/8$.
NoPE uses $d_r=0$ and $R_p=I$. No additional frequency or context-length
scaling is applied. The variant without Q/K normalization replaces the
RMSNorm operations with the identity.

The partial-RoPE model builds this 16-dimensional rotary table directly and
applies it to the last 16 of the 128 head dimensions. It does not select the
first $d_r/2$ frequencies of the full-head table, whose exponent denominator
would be $d_h$ rather than $d_r$.

With separate K/V branches, local values are $v_u=W_Vh_u$, with
$W_V\in\mathbb R^{d_h\times d}$, and compressed values are
$\hat v_j=\widehat V_j$. Neither receives key normalization or rotation.
With tied K/V branches, values instead equal the normalized, rotated keys:
$v_u=k_u$ and $\hat v_j=\hat k_j$. The combined attention output is then
inverse-rotated at the query position, $\tilde o_t=R_t^{-1}o_t$, before
concatenating query-head outputs and applying the standard output
projection. The inverse acts on the same dimensions as RoPE and applies to
the combined local and compressed readout. With separate branches,
$\tilde o_t=o_t$.

\subsection{Training Settings}
The reference 100B-token training run uses AdamW with
$(\beta_1,\beta_2)=(0.9,0.95)$, $\epsilon=10^{-8}$, weight decay $0.1$,
and gradient clipping at norm $1$. Training uses BF16 and seed 42.
The packed corpus is traversed once with sequence length 2,048 and a global
batch of 1,024 sequences. The learning rate warms up for 2\% of training to
$2\times10^{-3}$ and decays by a cosine schedule to $2\times10^{-5}$.
The seed-replication sweep additionally uses seeds 43--45.
Unless stated otherwise, results for our pretrained models use the final checkpoint at step 47,518.

%% file: Appendix/Sections/specialization_details.tex
\section{Additional Details on KV-Head Knockouts}\label{app:specialization}\label{app:representation-patching}
Here we define the KV-head knockouts of \Cref{sec:head-knockouts}. A knockout measures how retrieval changes when the prompt-specific output of an attention component is replaced by a fixed reference vector. We use the component's mean output over separate calibration prompts as this reference, which removes what the output carries about the particular prompt while retaining its average. The intervention is evaluated on held-out prompts, and its effects are resolved by the phase of the source key. This mean replacement builds on causal analyses of attention heads and their retrieval roles~\citep{olsson2022induction,wang2023interpretability,wu2025retrieval}. We first define which output is replaced and how its reference is fitted, then describe the prompts and metrics, the sweep over our pretrained models, and what the resulting effects can and cannot show. Finally, we adapt the intervention to DeepSeek-V4-Flash-Base, whose single shared KV head and long prompts call for a position-aligned variant.

\input{Appendix/Sections/representation_patching_details}

%% file: Appendix/Sections/representation_patching_details.tex
\subsection{Replaced Attention Output}
\label{app:patching-gqa}\label{app:patching-output}

Our models use grouped-query attention (GQA)~\citep{ainslie2023gqa}. A layer $\ell$ has $H_q$ query heads and $H_{kv}$ KV heads, and each KV head $h$ is read by a group $\mathcal Q(h)$ of $G=H_q/H_{kv}$ query heads, which share its keys and values but keep their own queries and attention distributions. In a compressed layer, the entries that a KV head exposes to a query combine visible compressed entries (\Cref{sec:setup}) with uncompressed local entries. Knockouts change neither these entries nor the rule that decides which ones are visible.

For query head $a$ at position $t$, the attention readout $\vr_{\ell,a,t}$ is the attention-weighted sum of the value vectors it reads. These attention weights are distinct from the compression gate scores $\valpha$ that form each compressed entry. Some tied-KV variants apply an inverse positional rotation to the readout before the output projection. We write $\vz_{\ell,a,t}$ for the readout after this rotation, with $\vz=\vr$ when none is applied, so the attention output is $\mW^O_\ell\operatorname{Concat}_a\vz_{\ell,a,t}+\vb^O_\ell$. The intervention acts on $\vz$, immediately before the output projection, where each query head's contribution is still a separate vector. With $T_x$ the final input position of prompt $x$, whose logits predict the answer, we write $\vo$ for the stacked outputs of all query heads that read KV head $h$ at this scoring position:
\begin{equation}
\vo_{\ell,h}(x)
=\operatorname{Concat}_{a\in\mathcal Q(h)}
  \vz_{\ell,a,T_x}(x)
\in\mathbb R^{Gd_v},
\label{eq:patch-group-output}
\end{equation}
where $d_v$ is the value dimension. This is a group of query-head outputs rather than an additional output of the KV head. With $H_q=16$, $H_{kv}=4$, and $d_v=128$, one knockout replaces four query-head vectors, a 512-dimensional group output. With one KV head, it replaces all 16 query-head outputs at that layer and position, which is the whole attention readout at that site.

\subsection{Calibration and Mean Replacement}
\label{app:patching-calibration}\label{app:patching-intervention}

An unmodified forward pass over a calibration set $\mathcal D_{\mathrm{cal}}=\{x^{(n)}\}_{n=1}^{N_{\mathrm{cal}}}$, disjoint from evaluation, records every query head's output at the scoring position. Because query heads that share a KV head keep their own queries and attention distributions, each query head receives its own mean,
\begin{equation}
\boldsymbol\mu_{\ell,a}
=\frac{1}{N_{\mathrm{cal}}}\sum_{n=1}^{N_{\mathrm{cal}}}
 \vz_{\ell,a,T_{x^{(n)}}}(x^{(n)}),\qquad
\bar\vo_{\ell,h}
=\operatorname{Concat}_{a\in\mathcal Q(h)}\boldsymbol\mu_{\ell,a},
\label{eq:patch-calibration-mean}
\end{equation}
so a knockout retains each head's own average output, which a mean pooled over the group would not. The calibration set is balanced over source-key residues, so no phase dominates the mean, and each mean also averages over mappings and filler texts. The reference is therefore independent of the evaluation prompt and its phase, and a phase-dependent knockout effect cannot reflect a phase-specific reference. It remains specific to the model, layer, query head, and calibration distribution. A phase-conditioned mean, or replacement by zero, would define a different intervention.

To knock out KV head $h_*$ in layer $\ell_*$, we replace the outputs of every query head that reads it at the scoring position (\Cref{fig:mean-substitution-protocol}):
\begin{equation}
\widetilde{\vz}_{\ell_*,a,t}(x)=
\begin{cases}
\boldsymbol\mu_{\ell_*,a},&a\in\mathcal Q(h_*),\ t=T_x,\\
\vz_{\ell_*,a,t}(x),&\text{otherwise}.
\end{cases}
\label{eq:patch-substitution}
\end{equation}
Other heads in the layer and all earlier computation are left unchanged, and later computation at the scoring position is rerun. Because the patch sits at the final input position of a causal decoder, it cannot change earlier context representations, and neither the model parameters nor the context cache are edited.

\begin{figure}[tb]
\centering
\includegraphics[width=\linewidth]{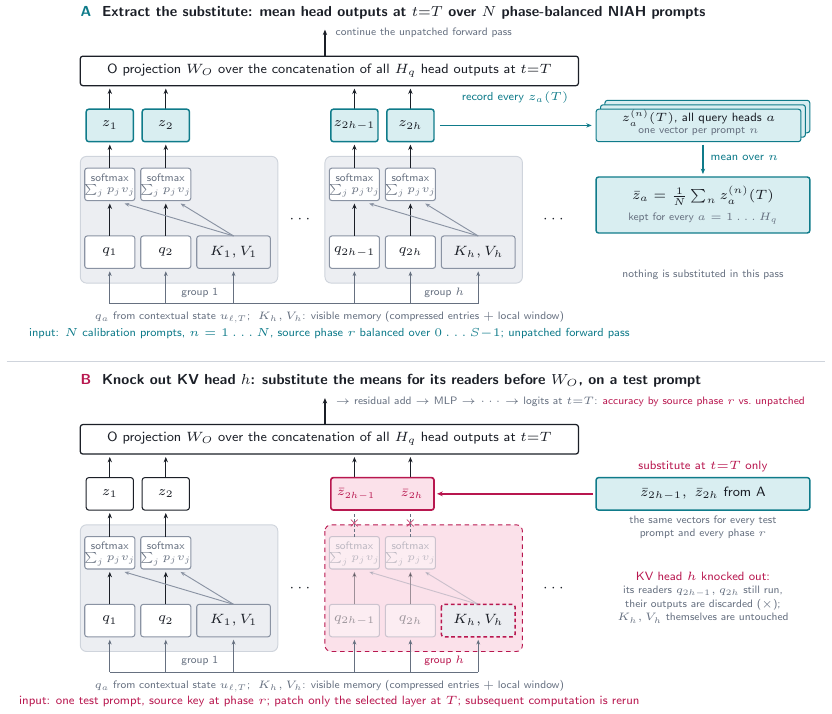}
\caption{\textbf{Calibration and mean replacement in a GQA block.} (A) Unmodified calibration prompts supply a separate mean for every query-head output at the scoring position. (B) On each evaluation prompt, the query heads that read the selected KV head receive those fixed means before the output projection, and the KV head's keys and values are left intact. The schematic shows two query heads per KV head ($G=2$) and suppresses the layer index and the inverse positional rotation. Each query head's output $z_a$ corresponds to $\vz_{\ell,a,T_x}$ and its mean $\bar z_a$ to $\boldsymbol\mu_{\ell,a}$. Concatenating the selected query-head outputs gives $\vo_{\ell,h}$ in \Cref{eq:patch-group-output}. The patch is confined to the selected layer and position, and downstream representations are recomputed.}
\label{fig:mean-substitution-protocol}
\end{figure}

\subsection{Evaluation Prompts, Metrics, and Uncertainty}
\label{app:patching-data}\label{app:patching-metrics}

Knockouts in our pretrained models use the retrieval task of \Cref{app:small-model-niah} under both padding protocols of \Cref{app:phase-balanced-evaluation}. Each protocol has its own calibration set, built with that protocol's prompts. Each prompt is indexed by its source-key residue $r=t_K\bmod 24$, where $t_K$ is the source key's zero-based position, and a model with stride $S$ sees phase $r\bmod S$. The 24 phase replicas of a logical example stay together when calibration is separated from evaluation. Logical examples $[1000,2000)$ form the calibration set and $[0,1000)$ the evaluation set, each with 24,000 prompts.

A prompt counts as correct when the gold name has the largest final logit, either among the 64 candidate names (candidate-restricted accuracy) or over the whole vocabulary (full-vocabulary accuracy). We also record the gold logit, the gold negative log-likelihood, and the total probability of the candidate set. Let $\mathcal D_r$ be the evaluation prompts at residue $r$ and $Y^c(x)\in\{0,1\}$ the correctness of prompt $x$ in the clean run ($c=0$) or under the knockout ($c=\mathrm{KO}(\ell,h)$). The paired effect, in percentage points, is
\begin{equation}
\Delta_{\ell,h}(r)
=\frac{100}{|\mathcal D_r|}
  \sum_{x\in\mathcal D_r}
  \bigl(Y^{\mathrm{KO}(\ell,h)}(x)-Y^0(x)\bigr),
\label{eq:patch-paired-effect}
\end{equation}
which is negative when the knockout impairs retrieval. Both terms use the same prompts. For stride $S$, folded effects average the residues that share a phase, for example residues $r$, $r+8$, and $r+16$ for W8/S8. Full-attention baselines have no phase and are grouped by the same residues only for comparison.

A phase with low clean accuracy has less room for a further decrease. In \Cref{fig:phase-knockouts} and the figures of \Cref{app:complete-controlled-sweep}, we therefore plot the relative accuracy change $100(A^{\mathrm{KO}}(r)-A^0(r))/A^0(r)$, which scales each loss by the clean accuracy available to lose. Here $A^c(r)$ is the accuracy at residue $r$. This ratio is undefined when clean accuracy is zero, so the figures apply a validity mask.

The independent sampling unit is the logical example, not each of its phase replicas. For an overall effect, we average paired differences within each logical example and take the standard error across examples, and we cluster by logical example in the same way for folded or between-phase contrasts. These clustered standard errors are computed but not plotted, and the heatmaps show point estimates that do not establish significance for every cell. The retrieval curves in the main text show two-sided 95\% Wilson intervals, and the clean curves in \Cref{app:complete-controlled-sweep} show pointwise 95\% cluster-bootstrap intervals over logical examples. Selecting the largest loss from a sweep is exploratory, and we do not test gate-nominated heads on held-out prompts.

\subsection{Knockout Sweeps for Qwen3-based Models}
\label{app:patching-recorded-sweeps}

To check whether phase-dependent knockout effects extend beyond the model of \Cref{fig:phase-knockouts}b, we sweep the 23 compressed models listed in \Cref{tab:controlled-sweep-results}, which have one, two, four, or eight KV heads, at checkpoint 47,518. Each model is evaluated over all 28 layers with the disjoint 24,000-prompt calibration and evaluation sets above. In the six figure groups of \Cref{app:complete-controlled-sweep}, we compare Prefix Padding and In-sequence Padding and report unfolded relative full-vocabulary accuracy changes with group-wise zero-centered scales. These are point estimates rather than independently confirmed selections.

\subsection{Scope of the Causal Interpretation}
\label{app:patching-controls}

The experiment changes an internal representation while holding the evaluation prompt and model parameters fixed. A loss of accuracy therefore measures a causal effect of this intervention on the tested retrieval behavior. It supports the importance of the selected readout under mean replacement, rather than exclusive storage of the target in that KV head. Other KV heads may contain redundant information, and a head can contribute to query processing or downstream computation in addition to carrying an answer. Conversely, a small knockout effect does not establish that a component is unused: other components can compensate. The change in the attention output is linear in $\bar\vo_{\ell_*,h_*}-\vo_{\ell_*,h_*}(x)$, but its effect on the final logits generally is not, so the effects of knocking out different heads separately need not add up to the effect of removing them together.

In compressed layers, one attention softmax covers compressed and local entries, and mean replacement acts on their combined readout. It does not selectively ablate compressed memory, remove one source association, or change a compression gate. Mean replacement preserves a reference first moment but does not guarantee an in-distribution activation or a matched norm. It removes within-head variation and the selected outputs' prompt-specific correlations with the rest of the network, and the calibrated mean can itself retain common task and template information. Consequently, the intervention is neither a literal deletion of a memory nor a proof of a unique retrieval circuit. We also do not report an identity-intervention comparison or a comparison that replaces the outputs of an equal number of other query heads in the same layer.

\subsection{Knockouts for DeepSeek-V4-Flash-Base}
\label{app:patching-deepseek}

The DeepSeek-V4-Flash-Base knockout follows the same mean-replacement logic but differs in what it replaces and at which positions. The model has one shared KV head per layer, so, as in our pretrained models with one KV head, a knockout replaces the whole attention readout of a layer. We take this readout after the attention output projection and before the hyper-connection, so the patched vector is the entire 4,096-dimensional layer output rather than a query-head slice of the pre-projection tensor in \Cref{eq:patch-group-output}. The patch can also cover several positions of the query suffix rather than only the final one.

Let $\va_{\ell,s}(x)\in\mathbb R^{4096}$ denote that output at suffix
index $s\in\{0,\ldots,15\}$. Calibration produces a separate mean for
each layer and each \emph{suffix position}:
\begin{equation}
\boldsymbol\nu_{\ell,s}
=\frac{1}{N_{\mathrm{cal}}}\sum_{n=1}^{N_{\mathrm{cal}}}
 \va_{\ell,s}(x^{(n)}),\qquad
\widetilde{\va}_{\ell,s}(x)=
\begin{cases}
\boldsymbol\nu_{\ell,s},&s\in\mathcal P,\\
\va_{\ell,s}(x),&s\notin\mathcal P.
\end{cases}
\label{eq:patch-deepseek-suffix}
\end{equation}
The mask $\mathcal P$ selects one of four conditions: the first occurrence of the query key, its second occurrence, both occurrences, or the complete 16-token suffix. In \Cref{fig:phase-knockouts}c, we show the full-suffix condition, which patches all 16 suffix positions. Suffix positions play different roles in the query, such as the two occurrences of the query key, so each receives its own mean rather than an average over the suffix, just as each query head receives its own mean in our pretrained models. The same position-aligned means are used for every evaluation prompt and source phase. Patching earlier suffix positions can also affect later suffix computation, so the full-suffix condition has a broader scope than the final-position-only intervention in our pretrained models.

For \Cref{fig:phase-knockouts}c, we use 256 evaluation prompts per residue group, each 128,000 tokens long with 16,000 key--value pairs. Calibration and evaluation use disjoint prompts with distinct texts. The source-key coordinate is modulo eight, so these residue groups span two cycles of the CSA stride of four. The 127,984-token prefix is prefilled in 1,024-token chunks, and the remaining 16 tokens form the query suffix.

We score the designated gold target token without autoregressive answer generation. This target-token accuracy is distinct from the answer-prefix accuracy used in the behavioral benchmark of \Cref{sec:modsens}. Because DeepSeek-V4-Flash-Base's outputs vary with batch size, each clean run uses the same batch composition as the intervened run it is compared with.

In the full-suffix condition of \Cref{fig:phase-knockouts}c, the largest
phase-averaged target-token accuracy losses occur at
layers 22 and 24 (zero-based). Effects vary across source residues. They establish
phase-dependent reliance on the patched layer readout at this intervention
granularity, but they do not identify a particular query head, a compressed-only
branch, or the compression operation that wrote the relevant information.

%% file: Appendix/Sections/complete_sweep.tex
\section{Complete Results for Qwen3-based Models}
\label{app:complete-controlled-sweep}

In \Cref{sec:phase-accuracy}, we plot a subset of our pretrained models. Here we give fuller results, so that each design change can be compared with the reference and the full-attention baselines. In \Cref{tab:controlled-sweep-results}, we summarize the configurations shown in \Cref{fig:phase-accuracy} and the grouped figures below, reporting validation loss alongside mean and worst-residue retrieval accuracy and the gap across the 24 evaluated source-key residues under Prefix Padding. We include 23 compressed runs and three full-attention baselines.

These summary numbers cannot show where the weak spots fall or which layers support retrieval at each phase. In the following figures, we therefore show our pretrained models at the end of training (step 47,518) in six groups that follow the blocks of compressed runs in \Cref{tab:controlled-sweep-results}. In each figure, we pair clean full-vocabulary retrieval across the 24 source-key residues with layer-by-phase relative accuracy changes under KV-head knockouts (\Cref{app:representation-patching}). Across the six groups, no tested change removes the periodic weak spots, although some narrow them substantially. Uniform averaging, for example, reduces the best-to-worst gap from 75.3 to 19.2 percentage points. Random seeds, positional encoding, and normalization reshape the accuracy pattern across phases, and the weak spots repeat with the compression stride.

We show every group under both padding protocols of \Cref{app:phase-balanced-evaluation}. Prefix Padding shifts all records together, so the distractors change phase along with the target. In-sequence Padding instead varies the gaps between records, so a pattern that appears under both protocols does not require all records to shift together. Rows 1 and 2 show clean accuracy and knockout effects under Prefix Padding, and rows 3 and 4 show the same under In-sequence Padding. Curves use 1,000 logical examples and show pointwise 95\% cluster-bootstrap intervals. Heatmaps within each group share a zero-centered scale.

\input{Appendix/Tables/controlled_sweep_results}
\input{Appendix/Tables/controlled_group_summaries}

%% file: Appendix/Tables/controlled_sweep_results.tex
\begin{table}[tb]
\centering
\caption{\textbf{Summary of the plotted pretrained models.}
Prefix Padding results for the 26 runs shown in the cited main-text and appendix figures:
23 compressed runs and three full-attention baselines, evaluated at the
end of training (step 47,518).
For full-vocabulary retrieval accuracy $A_\rho$ at source-key residue
$\rho=t_K\bmod24$, we report the mean over all 24 residues, the minimum,
and the gap $\max_\rho A_\rho-\min_\rho A_\rho$ in percentage points.
The same coordinate is used for full-attention baselines, which have no
compression phase. Validation loss is reported separately from retrieval.}
\label{tab:controlled-sweep-results}
\begingroup
\small
\setlength{\tabcolsep}{3pt}
\renewcommand{\arraystretch}{1.06}
\begin{tabular}{@{}p{0.40\textwidth}ccrrrr@{}}
\toprule
& & & & \multicolumn{3}{c}{Retrieval accuracy} \\
\cmidrule(l){5-7}
Configuration / change & $W/S$ & KV & Val. loss & Mean (\%) & Min. (\%) & Gap (pp) \\
\midrule
\multicolumn{7}{@{}l}{\textit{Full-attention baselines} (\Cref{fig:phase-accuracy})} \\
Full attention & -- & 1 & 2.466 & 18.8 & 17.2 & 3.1 \\
Full attention & -- & 4 & 2.448 & 47.1 & 44.8 & 6.1 \\
Full attention & -- & 8 & 2.437 & 61.1 & 58.4 & 4.3 \\
\addlinespace[3pt]
\multicolumn{7}{@{}l}{\textit{Reference and random-seed replication} (\Cref{fig:controlled-group-random-seed-replication})} \\
Reference, seed 42 & 8/8 & 1 & 2.478 & 50.6 & 5.7 & 75.3 \\
Reference, seed 43 & 8/8 & 1 & 2.482 & 32.1 & 2.5 & 59.2 \\
Reference, seed 44 & 8/8 & 1 & 2.477 & 56.1 & 11.6 & 70.3 \\
Reference, seed 45 & 8/8 & 1 & 2.478 & 40.5 & 15.6 & 36.6 \\
\addlinespace[3pt]
\multicolumn{7}{@{}l}{\textit{Native KV-head count} (\Cref{fig:controlled-group-native-kv-head-count})} \\
Two KV heads & 8/8 & 2 & 2.465 & 58.3 & 9.8 & 76.8 \\
Four KV heads & 8/8 & 4 & 2.448 & 68.7 & 13.6 & 78.0 \\
Eight KV heads & 8/8 & 8 & 2.431 & 59.4 & 9.9 & 71.1 \\
\addlinespace[3pt]
\multicolumn{7}{@{}l}{\textit{Compression-window size and stride} (\Cref{fig:controlled-group-compression-geometry})} \\
Non-overlapping windows & 4/4 & 1 & 2.469 & 38.5 & 8.0 & 50.5 \\
Overlapping windows & 8/4 & 1 & 2.467 & 66.9 & 46.2 & 39.8 \\
Overlapping windows & 8/6 & 1 & 2.470 & 53.4 & 38.5 & 31.9 \\
Larger windows, one KV head & 12/12 & 1 & 2.487 & 33.0 & 5.0 & 63.2 \\
\addlinespace[3pt]
\multicolumn{7}{@{}l}{\textit{Compression kernels} (\Cref{fig:controlled-group-compression-operator})} \\
Tied K/V branches & 8/8 & 1 & 2.495 & 21.8 & 3.6 & 39.1 \\
Uniform averaging & 8/8 & 1 & 2.497 & 12.5 & 4.7 & 19.2 \\
Uniform averaging, tied K/V, postnorm & 8/8 & 1 & 2.503 & 18.6 & 1.9 & 39.3 \\
Vector gates & 8/8 & 1 & 2.478 & 19.4 & 4.0 & 34.0 \\
\addlinespace[3pt]
\multicolumn{7}{@{}l}{\textit{Positional encoding and normalization} (\Cref{fig:controlled-group-position-and-normalization})} \\
No Q/K normalization & 8/8 & 1 & 2.485 & 38.5 & 23.8 & 34.6 \\
No positional gate biases & 8/8 & 1 & 2.479 & 34.3 & 4.4 & 60.7 \\
No RoPE (NoPE) & 8/8 & 1 & 2.504 & 29.4 & 5.0 & 40.9 \\
RoPE on 16 dimensions & 8/8 & 1 & 2.482 & 74.2 & 32.9 & 57.2 \\
\addlinespace[3pt]
\multicolumn{7}{@{}l}{\textit{Local-memory policy and DeepSeek-style kernels} (\Cref{fig:controlled-group-memory-policy-and-implementation})} \\
Uncompressed open tail & 8/8 & 1 & 2.520 & 75.0 & 57.4 & 33.9 \\
Uncompressed open tail & 8/8 & 8 & 2.460 & 53.0 & 25.5 & 49.3 \\
DeepSeek-style, vector gates & 8/4 & 1 & 2.476 & 35.6 & 18.7 & 38.3 \\
DeepSeek-style, scalar gates & 8/4 & 1 & 2.473 & 21.5 & 14.2 & 18.7 \\
\bottomrule
\end{tabular}
\par\smallskip
\begin{minipage}{\textwidth}
\footnotesize
\textit{Notes.} KV denotes the number of KV heads per layer.
The reference uses scalar gates, separate K/V branches, payload maps shared
across offsets, and offset-specific gate maps and positional biases
(\Cref{sec:setup} and \Cref{app:scratch}). Postnorm denotes post-projection
normalization. Each run appears once, and the seed-42 reference supplies the
one-KV-head model for the head-count comparison.
\end{minipage}
\endgroup
\end{table}
\clearpage

%% file: Appendix/Tables/controlled_group_summaries.tex
\clearpage
\begingroup
\setlength{\abovecaptionskip}{4pt}
\begin{figure}[tb]
\centering
\includegraphics[width=\textwidth,height=0.4\textheight,keepaspectratio]{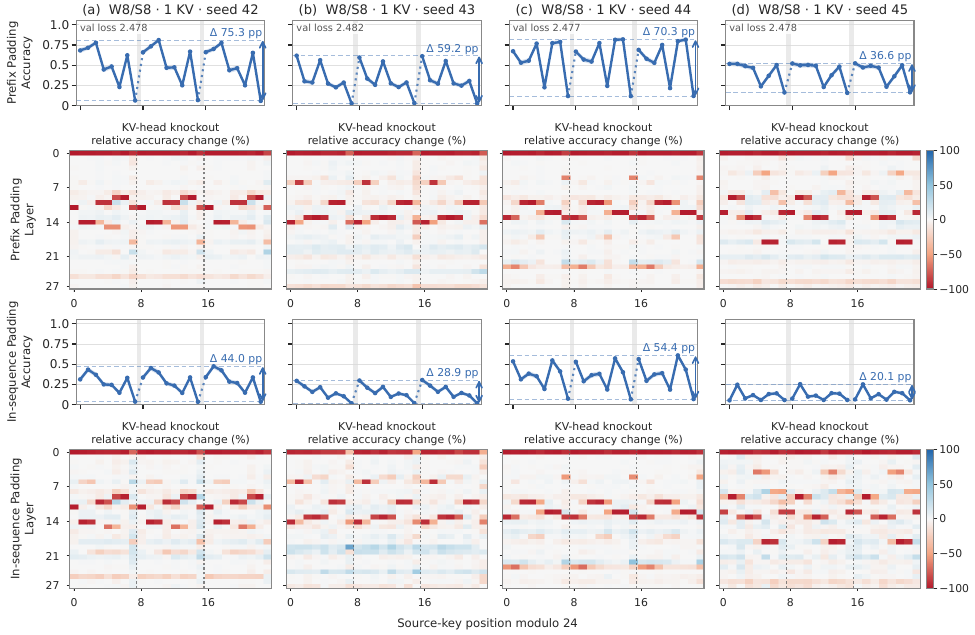}
\caption{\textbf{Random-seed replication.} Rows 1--2 use Prefix Padding and rows 3--4 use In-sequence Padding. Rows 1 and 3 show full-vocabulary accuracy, and rows 2 and 4 show relative accuracy change under KV-head knockout. Seed changes alter the accuracy pattern across phases while preserving periodic weak spots.}
\label{fig:controlled-group-random-seed-replication}
\par\vspace{9pt}
\includegraphics[width=\textwidth,height=0.4\textheight,keepaspectratio]{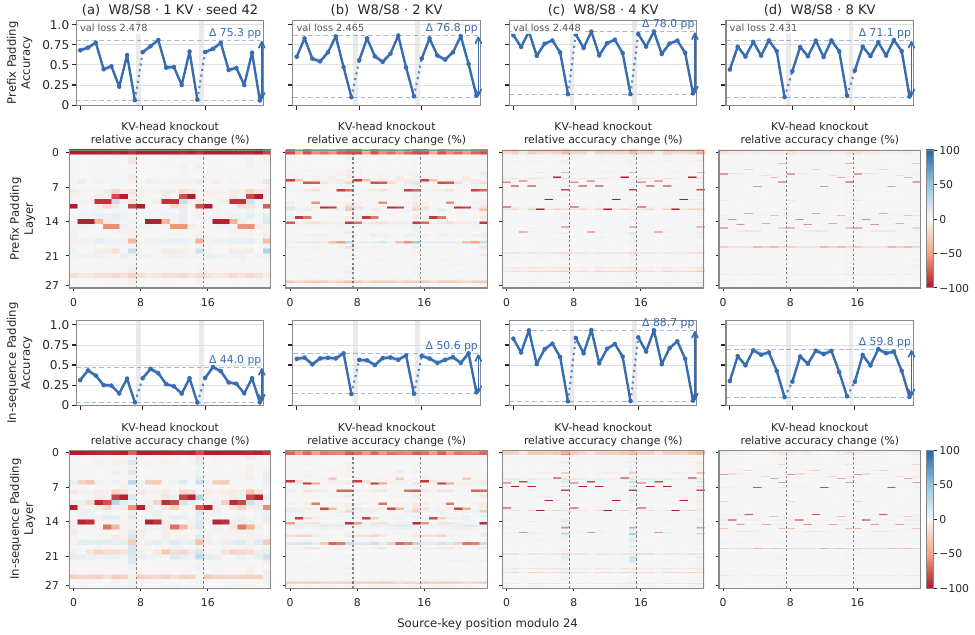}
\caption{\textbf{Number of KV heads.} Rows 1--2 use Prefix Padding and rows 3--4 use In-sequence Padding. Rows 1 and 3 show full-vocabulary accuracy, and rows 2 and 4 show relative accuracy change under KV-head knockout. Phase sensitivity persists from one to eight KV heads.}
\label{fig:controlled-group-native-kv-head-count}
\end{figure}
\clearpage
\begin{figure}[tb]
\centering
\includegraphics[width=\textwidth,height=0.4\textheight,keepaspectratio]{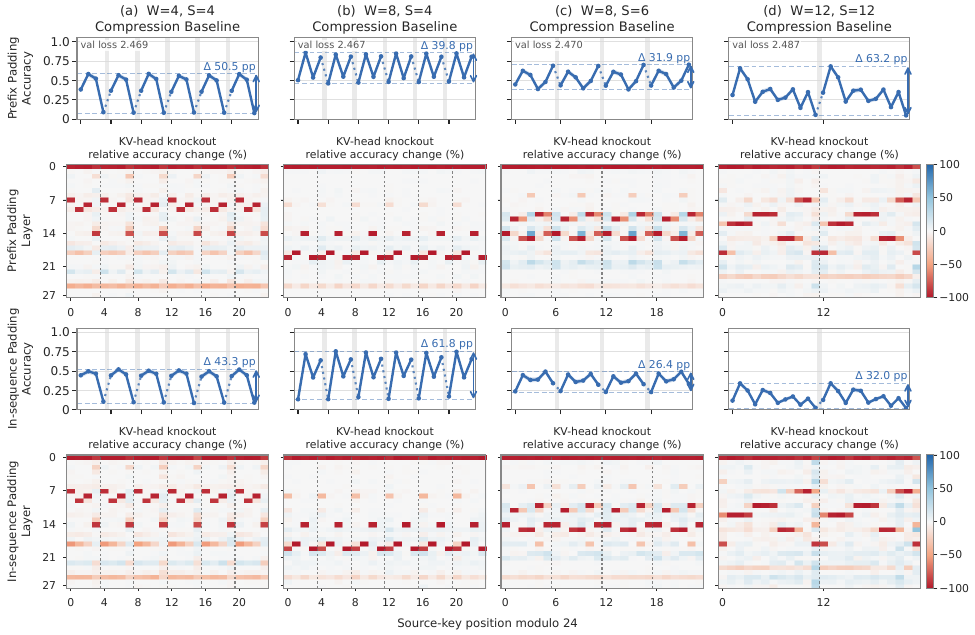}
\caption{\textbf{Compression geometry.} Rows 1--2 use Prefix Padding and rows 3--4 use In-sequence Padding. Rows 1 and 3 show full-vocabulary accuracy, and rows 2 and 4 show relative accuracy change under KV-head knockout. Weak spots repeat with the compression stride.}
\label{fig:controlled-group-compression-geometry}
\par\vspace{9pt}
\includegraphics[width=\textwidth,height=0.4\textheight,keepaspectratio]{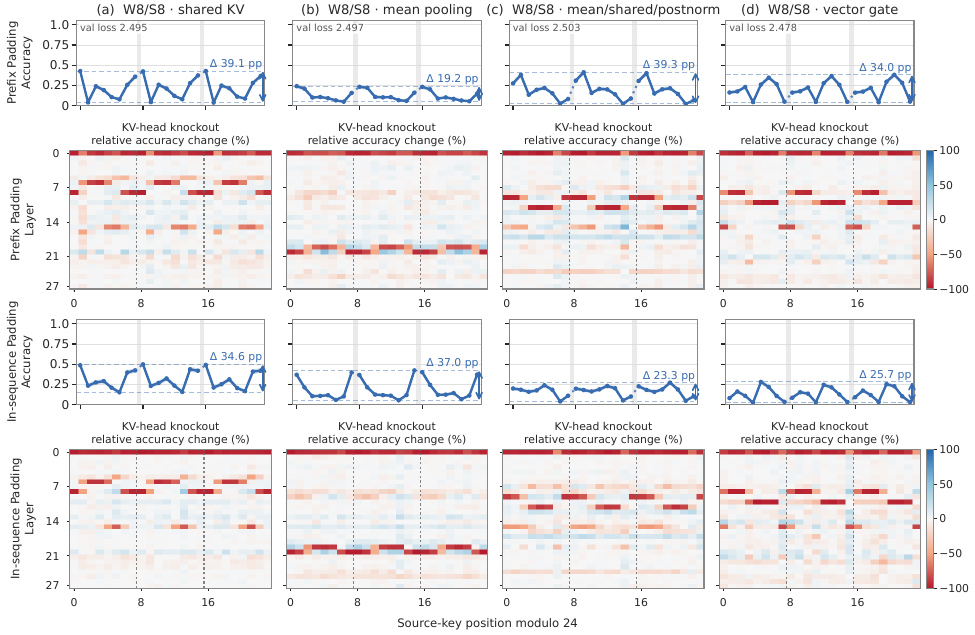}
\caption{\textbf{Compression kernels.} Rows 1--2 use Prefix Padding and rows 3--4 use In-sequence Padding. Rows 1 and 3 show full-vocabulary accuracy, and rows 2 and 4 show relative accuracy change under KV-head knockout. Phase sensitivity persists across pooling, sharing, and gate variants.}
\label{fig:controlled-group-compression-operator}
\end{figure}
\clearpage
\begin{figure}[tb]
\centering
\includegraphics[width=\textwidth,height=0.4\textheight,keepaspectratio]{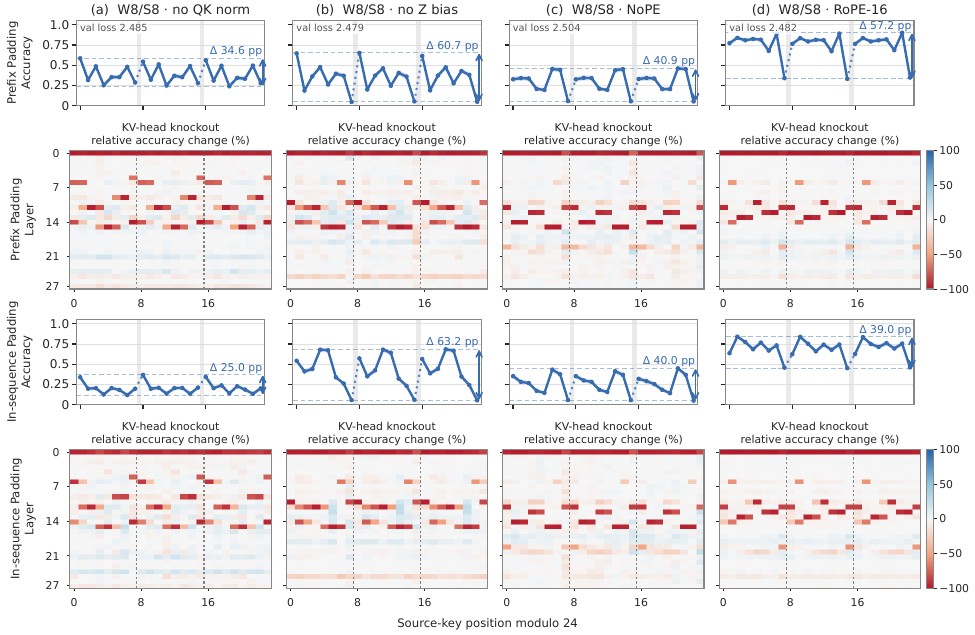}
\caption{\textbf{Position and normalization.} Rows 1--2 use Prefix Padding and rows 3--4 use In-sequence Padding. Rows 1 and 3 show full-vocabulary accuracy, and rows 2 and 4 show relative accuracy change under KV-head knockout. Position and normalization choices reshape but do not remove phase sensitivity.}
\label{fig:controlled-group-position-and-normalization}
\par\vspace{9pt}
\includegraphics[width=\textwidth,height=0.4\textheight,keepaspectratio]{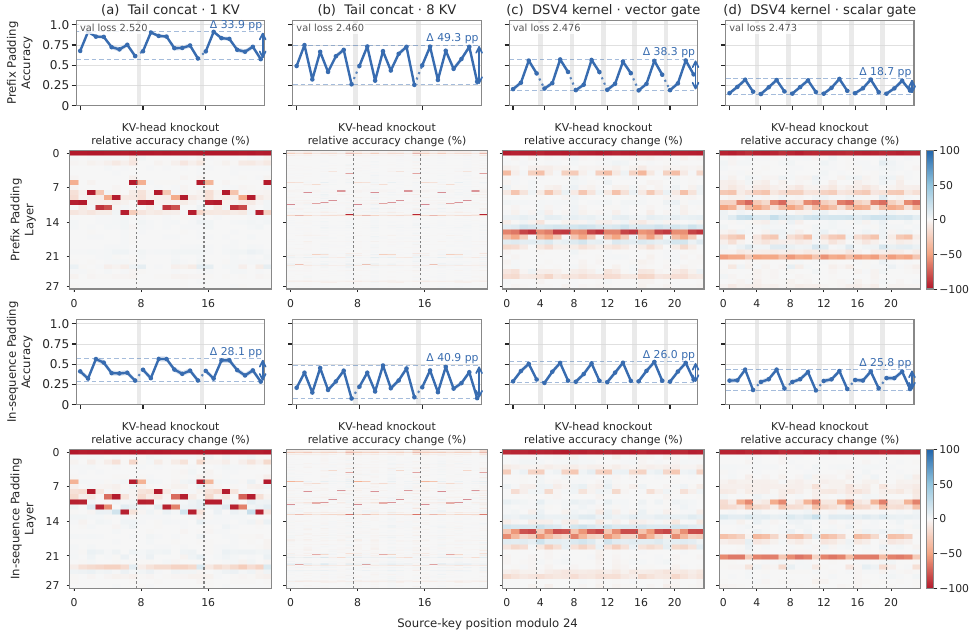}
\caption{\textbf{Memory policy and implementation variants.} Rows 1--2 use Prefix Padding and rows 3--4 use In-sequence Padding. Rows 1 and 3 show full-vocabulary accuracy, and rows 2 and 4 show relative accuracy change under KV-head knockout. Open-tail local memory and DeepSeek-style kernels remain phase-sensitive.}
\label{fig:controlled-group-memory-policy-and-implementation}
\end{figure}
\clearpage
\endgroup

%% file: Appendix/Sections/gate_knockout_comparison.tex
\section{Gate Preferences and Knockout Effects across Models}
\label{app:gate-knockout-comparison}

In \Cref{sec:qwen3-spec}, we link gate preferences to knockout effects for three heads of one model. Here we show the same comparison in more models: every head of the four W8/S8 seed replications, and every head of five other configurations. Sharp gate preferences with knockout damage at the favored phases recur across seeds and configurations, but not in every head, and the phase-folded gate curves show little preference when windows overlap.

\subsection{Measuring Gate Preferences}
\label{app:gate-concentration}

In \Cref{sec:qwen3-spec}, we summarize each gate by two effective numbers of offsets (\Cref{fig:phase-emergence}a). For one window $j$ of input $x$, collect the gate scores of its $W$ offsets into the probability vector $\valpha^B_{\ell,h,j}(x)\in\mathbb R^W$, and let $H(p)=-\sum_i p_i\log p_i$. The vertical axis is the mean per-window effective number of offsets, $\mathbb E_{x,j}[\exp(H(\valpha^B_{\ell,h,j}(x)))]$, averaged over all compression windows of a held-out validation set of natural-language pretraining text. The horizontal axis is the effective number of offsets of the mean gate, $\exp(H(\mathbb E_{x,j}[\valpha^B_{\ell,h,j}(x)]))$. Both range from 1 to $W$, and smaller values mean greater concentration. In \Cref{sec:qwen3-spec}, we call the horizontal quantity the gate's \emph{static concentration}. The vertical quantity also includes \emph{input-dependent concentration}, which places a gate below the diagonal. A small per-window value alone does not imply a stable preference. A gate that selects one offset in each window but visits every offset equally often across inputs has a uniform mean gate and lies near $(W,1)$. For vector gates, averaging over channels can conceal preferences that differ across channels, and our channel-averaged curves (\Cref{fig:gate-knockout-overlays-overlap}b) cannot reveal them.

In \Cref{fig:gate-concentration-all}, we place the key gate of every KV head in this plane at the end of training, for every run in \Cref{tab:controlled-sweep-results} with learned scalar gates and for the W12/S12 model with four KV heads used in \Cref{app:phase-cycling}. With non-overlapping windows, a few heads per model lie along the diagonal toward the lower left, while most gates end near the right edge with varying input-dependent concentration, as in \Cref{fig:phase-emergence}a. With more KV heads per layer, the most concentrated heads reach further toward the lower left. With overlapping windows at stride four, including the DeepSeek-style kernel, many gates sit well left of the right edge but below the diagonal, so they favor some offsets across windows while also varying with the input.

In the figures of this appendix, we overlay each head's gate preferences on its knockout effects. Each panel shows one KV head, with the phase $r$ of the source key on the horizontal axis. The background shows the head's knockout effect at each phase (\Cref{sec:head-knockouts}), with red indicating impairment. The black curve, $g^K(r)$, is the head's mean gate score on key tokens at phase $r$, scaled so that a uniform gate equals one (dashed line).
The purple curve, $g^V(r)$, shows the value gate read one phase later. Retrieval needs the key, to locate the entry, and the value one token later, which is the answer. For a key at phase $r$, we therefore plot the value gate's score at phase $(r+1)\bmod S$, so that both curves and the knockout background are indexed by the key's phase. A head that retains a key and its value shows both curves peaking at the same $r$, and if the head matters for retrieval, its knockout damage should peak there as well. When $W=S$, the value of a key at the last phase $S-1$ falls in the next window, and hollow squares mark these points.

\subsection{Variation across Random Seeds}

The four seed replications develop similarly sharp gate preferences, but in different places (\Cref{fig:gate-knockout-overlays-seeds}). Each seed has four to six heads with concentrated mean key gates (entropy below 90\% of a uniform gate's), whose key-gate peaks reach 2.0 to 4.0 times the uniform level and whose knockout damage falls mainly at the favored phases. These heads lie between layers 4 and 18, but no layer has such a head in all four seeds, and their preferred key phases cover all eight phases. Among these heads, the only recurring feature is a head whose key gate favors phase 0 and whose value gate favors the value token one phase later, found at layers 11, 14, 13, and 12 in seeds 42 to 45, respectively. This variation is consistent with \Cref{sec:theory-dynamics}, where in an idealized model each head settles on a sharp selection of positions that depends on initialization, and heads need not cover complementary positions. Outside these heads, knocking out layer 0 impairs retrieval at every phase in all seeds, despite its near-uniform gate.

\subsection{Other Configurations and Overlapping Windows}

Sharp preferences also arise without rotary position embeddings. Without them, heads in layers 11, 12, 14, and 15 still favor particular phases, with knockout damage at those phases (\Cref{fig:gate-knockout-overlays}a). We show the same comparison with larger windows (W12/S12, one KV head) in \Cref{fig:gate-knockout-overlays}b and with four KV heads per layer (W8/S8) in \Cref{fig:gate-knockout-overlays-heads}. With four KV heads, the preferences become sharper than with one. Many heads spike on a single offset, and their knockout losses concentrate on a single phase, as for heads 2 of layers 5 and 16 and heads 0 of layers 10 and 12. Together, these sharp heads favor all eight phases. Overlapping windows are the exception. With overlapping eight-token windows at stride four, including the DeepSeek-style kernel, most phase-folded gate curves stay close to uniform (\Cref{fig:gate-knockout-overlays-overlap}), leaving little phase structure for knockout damage to follow. Each phase sums two offsets, however, and many gates in these models do favor some offsets across windows (\Cref{fig:gate-concentration-all}), a preference that folding into phases hides. Vector-gate scores are also averaged over channels, which can hide preferences among channels. The gate account of \Cref{sec:qwen3-spec} therefore applies most directly to non-overlapping windows.

\begin{figure}[tb]
\centering
\includegraphics[width=\linewidth,height=0.8\textheight,keepaspectratio]{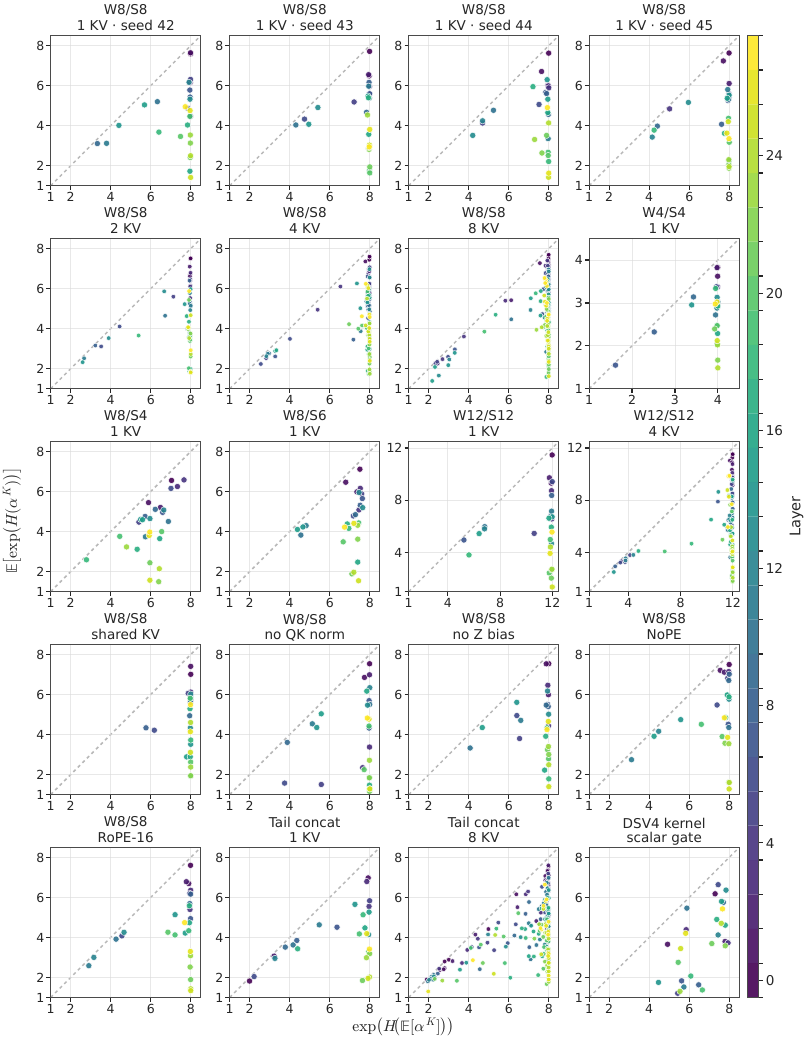}
\caption{\textbf{Final key-gate concentration across configurations} (step 47,518). Each point is one KV head's key gate, placed by its static effective number of offsets $\exp(H(\mathbb E[\valpha^K]))$ (horizontal) and its mean per-window effective number $\mathbb E[\exp(H(\valpha^K))]$ (vertical), and colored by layer. The dashed line marks equality, which holds for a gate with the same scores in every window, and lower values mean more concentration. Panels cover every run in \Cref{tab:controlled-sweep-results} with learned scalar gates and the W12/S12 model with four KV heads of \Cref{app:phase-cycling}. We omit uniform averaging, which has no learned gates, and the vector-gate runs. Both quantities are defined in \Cref{app:gate-concentration}.}
\label{fig:gate-concentration-all}
\end{figure}

\begin{figure}[tb]
\centering
\includegraphics[width=\linewidth,height=0.8\textheight,keepaspectratio]{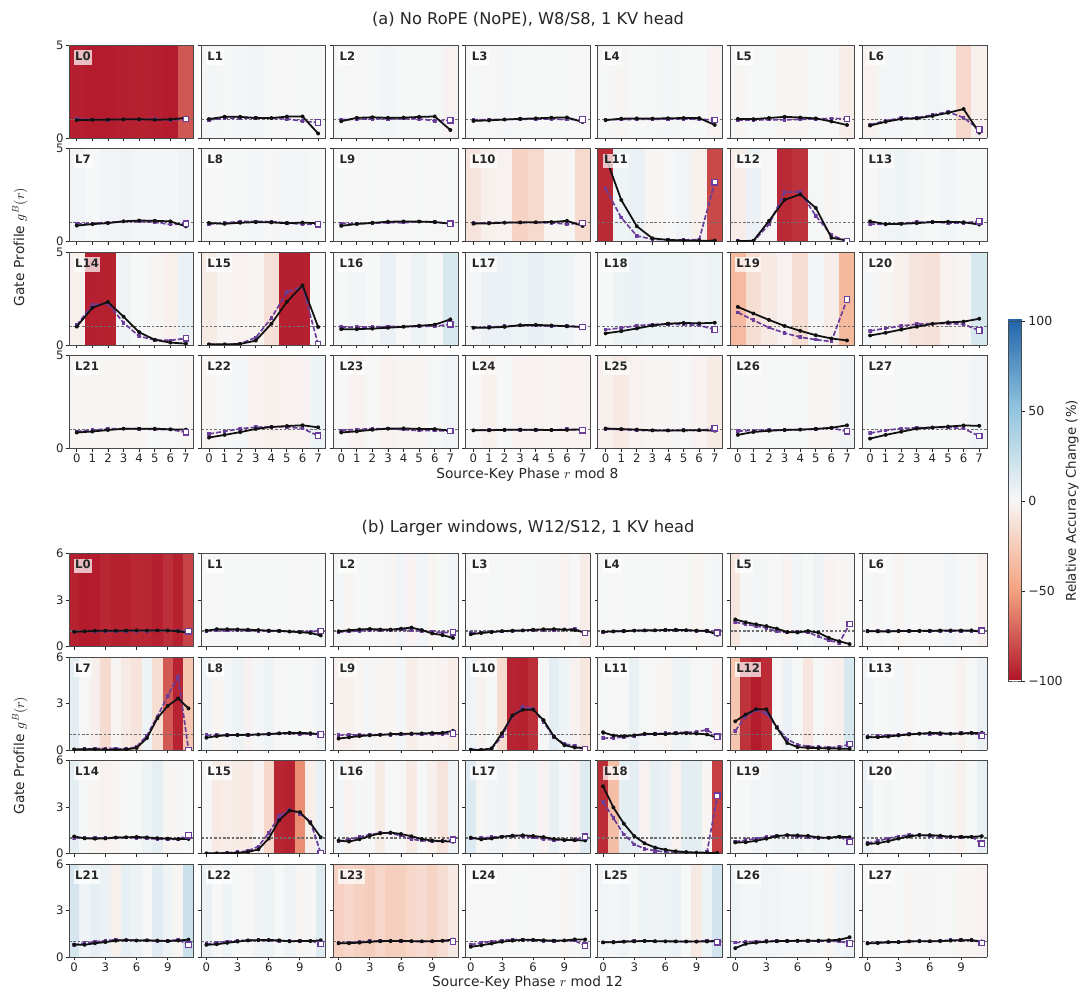}
\caption{\textbf{Gate preferences and knockout effects without positional embeddings and with larger windows.} (a) No rotary position embeddings (NoPE, W8/S8, one KV head). (b) Larger windows (W12/S12, one KV head). Without RoPE, heads still develop sharp phase preferences, with knockout damage at the favored phases (layers 11, 12, 14, and 15). Both checkpoints are final (step 47,518) and use Prefix Padding. Vertical scales are shared within each model. Encoding and gate inputs follow \Cref{fig:gate-knockout-overlays-seeds}.}
\label{fig:gate-knockout-overlays}
\end{figure}

\begin{figure}[tb]
\centering
\includegraphics[width=\linewidth,height=0.8\textheight,keepaspectratio]{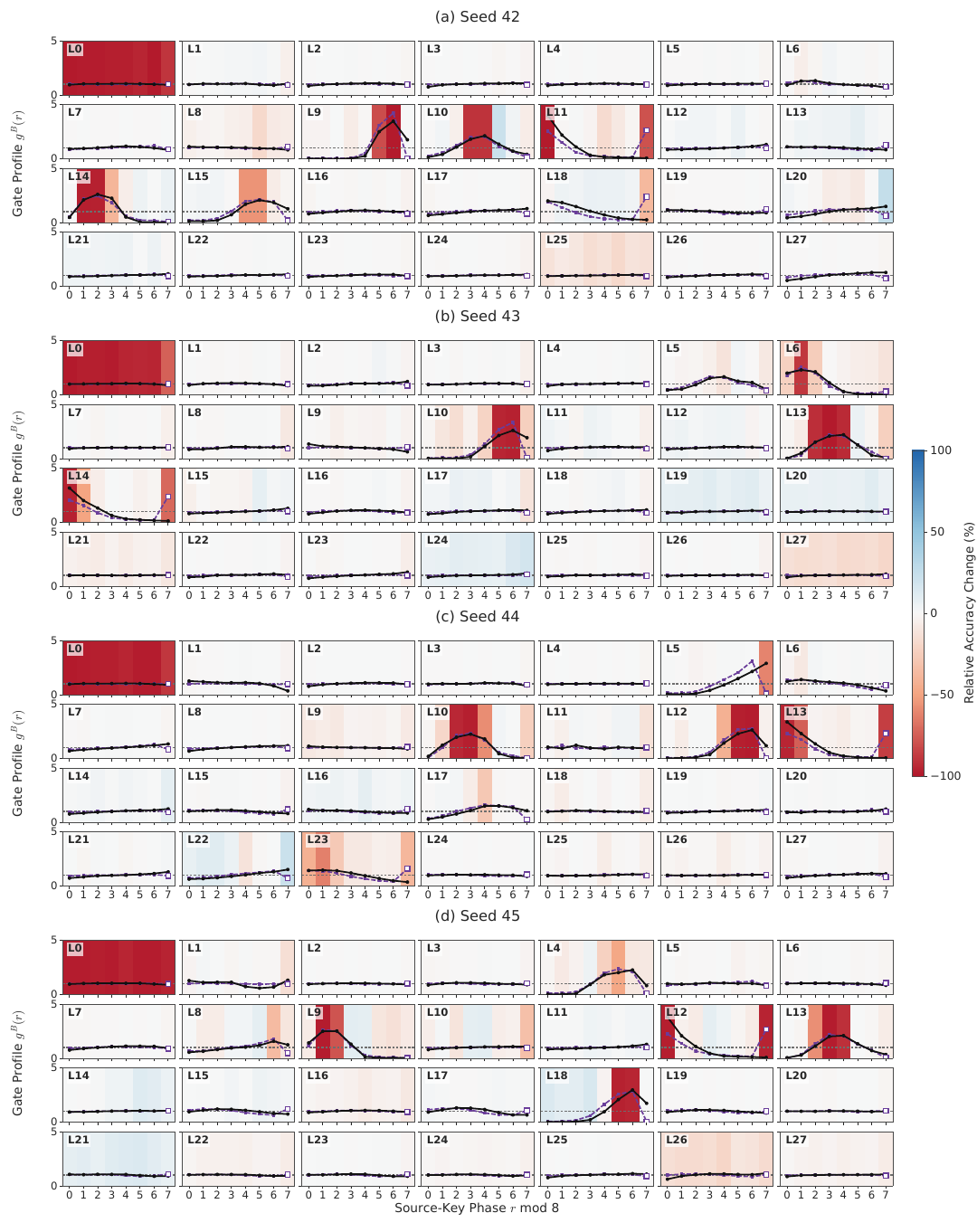}
\caption{\textbf{Gate preferences and knockout effects across random seeds} (W8/S8, one KV head, seeds 42--45, step 47,518). Each panel shows one layer's KV head, labeled L$\ell$. Black circles show the key gate $g^K$, purple squares the value gate $g^V$ read one phase later, and the dashed line a uniform gate. Hollow squares mark values that fall in the next window. The background shows the relative accuracy change under the head's knockout, saturating at $\pm100\%$, with red indicating impairment (Prefix Padding, 1,000 logical examples per phase, folded modulo $S$). Gate curves average 2,048 held-out natural-language sequences of 2,048 tokens. All seeds share one vertical scale.}
\label{fig:gate-knockout-overlays-seeds}
\end{figure}

\begin{figure}[tb]
\centering
\includegraphics[width=\linewidth,height=0.8\textheight,keepaspectratio]{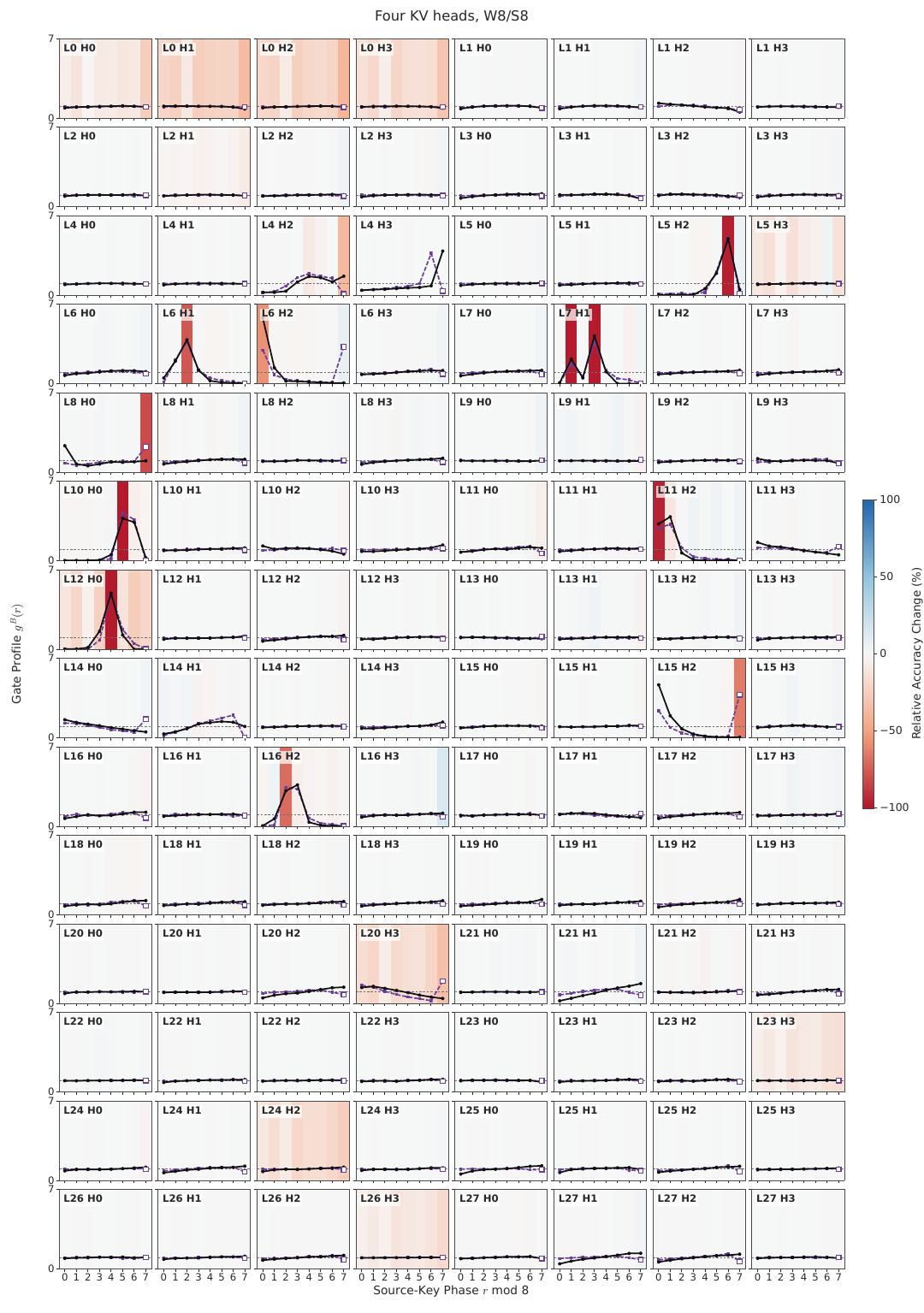}
\caption{\textbf{Gate preferences and knockout effects with four KV heads per layer} (W8/S8, step 47,518, Prefix Padding). Panels are labeled L$\ell$ H$h$ for layer $\ell$ and KV head $h$. Encoding and gate inputs follow \Cref{fig:gate-knockout-overlays-seeds}.}
\label{fig:gate-knockout-overlays-heads}
\end{figure}

\begin{figure}[tb]
\centering
\includegraphics[width=\linewidth,height=0.8\textheight,keepaspectratio]{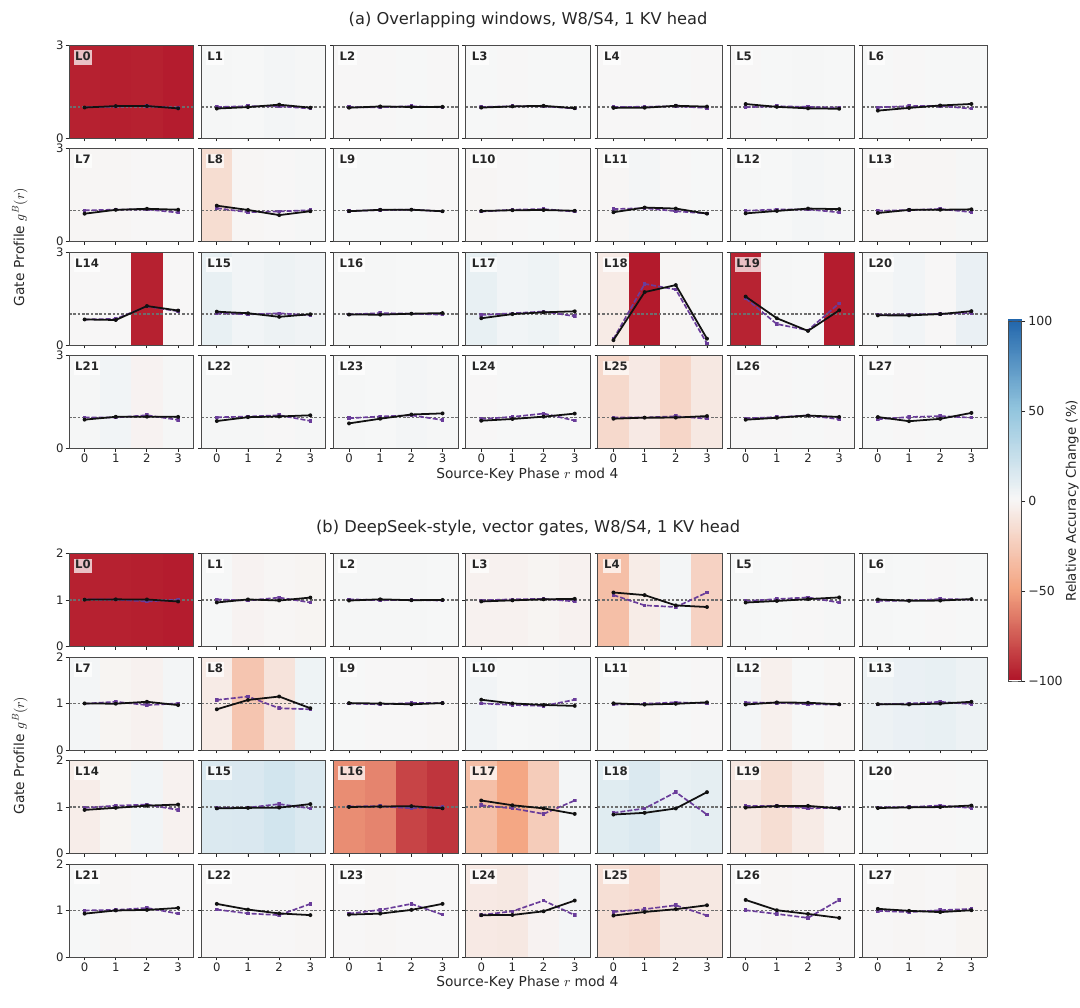}
\caption{\textbf{Gate preferences and knockout effects with overlapping windows.} (a) Overlapping windows (W8/S4, one KV head). (b) DeepSeek-style kernel with vector gates (W8/S4, one KV head). Each phase receives two of the eight window offsets, and the curves add their scores, rescaled so that a uniform gate still equals one. In (b), K and V share one gate, so the purple curve is the shared gate read one token later, and vector-gate scores are averaged over channels. No hollow markers appear because $W\neq S$. Encoding, protocol, gate inputs, and checkpoint step follow \Cref{fig:gate-knockout-overlays-seeds}.}
\label{fig:gate-knockout-overlays-overlap}
\end{figure}
\FloatBarrier
\clearpage

%% file: Appendix/Sections/gate_cycling.tex
\section{Redirecting Gate Preferences}\label{app:phase-cycling}

In \Cref{sec:qwen3-spec}, we found a few gates with sharp offset preferences that persist across windows and inputs, and knocking out three of the corresponding heads in the reference model costs the most accuracy near the phases their gates favor. If such preferences help decide which phases are weak, editing them after training may move the weak phases with them. We test this by cycling each gate's preference by a fixed number of offsets and checking whether the accuracy pattern across phases shifts by as many phases.

\begin{figure}[tb]
\centering
\includegraphics[width=0.94\textwidth,height=0.8\textheight,keepaspectratio]{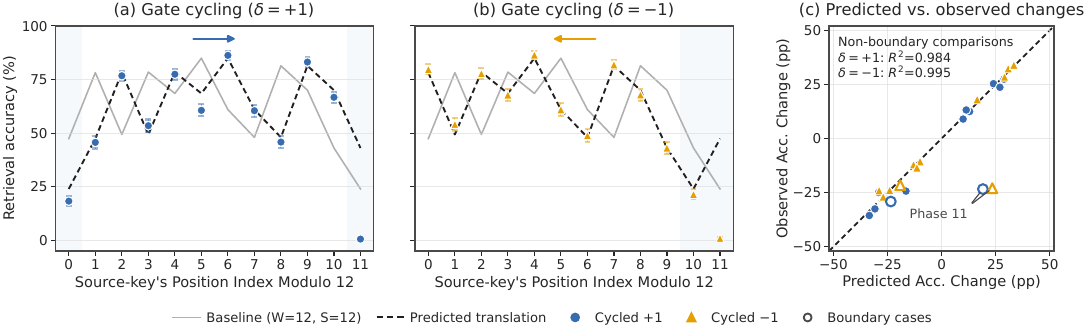}
\caption{\textbf{Gate cycling approximately shifts retrieval accuracy.} (a,b) Accuracy by source-key phase in the W12/S12 model after cycling by $\delta=+1$ and $\delta=-1$ (markers), with the uncycled model (gray) and the prediction (dashed). (c) Predicted versus observed changes and the identity line (dashed). Shaded and hollow points involve the boundary phase and are excluded from $R^2$.}
\label{fig:phase-interventions}
\end{figure}

\begin{figure}[tb]
\centering
\includegraphics[width=\textwidth,height=0.8\textheight,keepaspectratio]{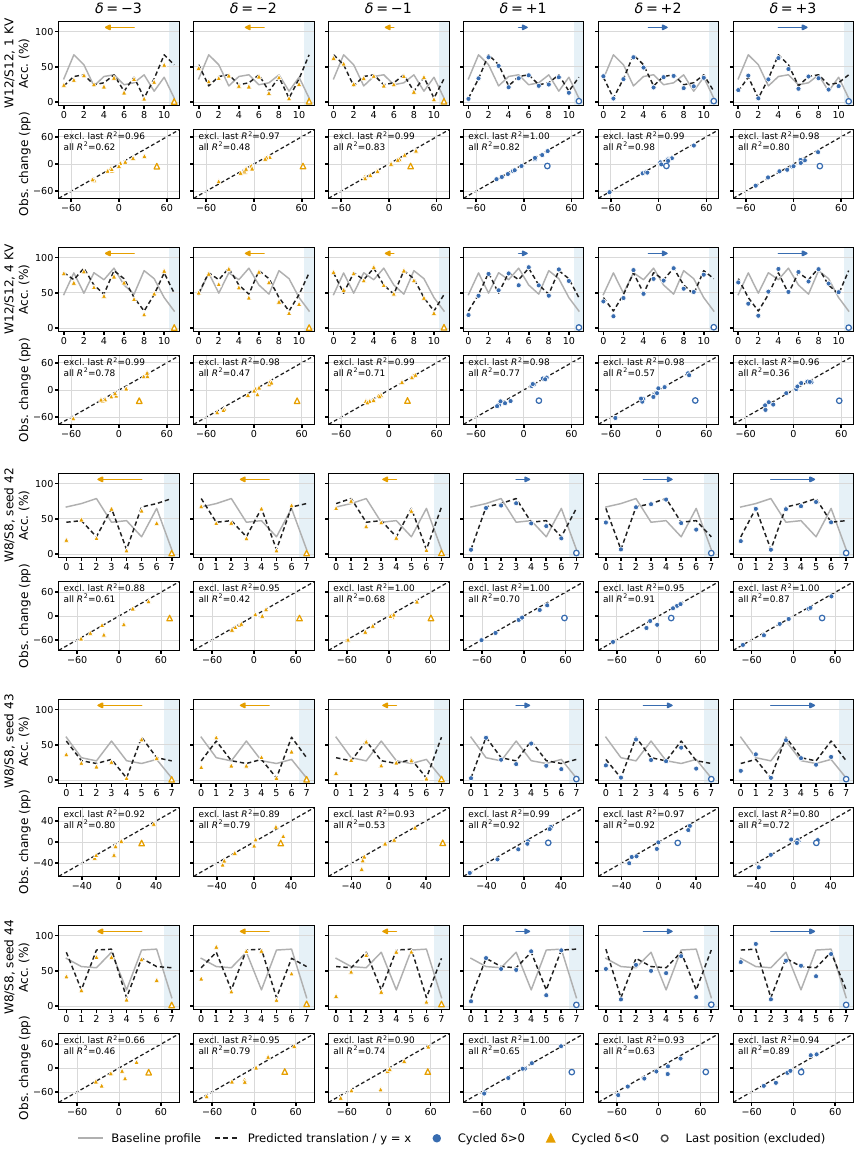}
\caption{\textbf{Gate cycling shifts retrieval accuracy across models and cycle sizes.} Each cell follows \Cref{fig:phase-interventions} for one model and one cycle $\delta$, with accuracy by source-key phase above and predicted versus observed accuracy changes below. Rows show the two W12/S12 models, the W8/S8 reference model (seed 42), and two of its seed replications (seeds 43 and 44), each W8/S8 model with one KV head. Columns show $\delta$ from $-3$ to $+3$, and arrows give its direction and size. Error bars are 95\% bootstrap intervals. Each cell reports $R^2$ without the boundary phase $S-1$ (``excl.\ last'') and over all phases (``all''). Unlike in \Cref{fig:phase-interventions}, ``excl.\ last'' excludes only comparisons whose observed phase is $S-1$.}
\label{fig:phase-cycling-models}
\end{figure}

\subsection{Gate Cycling}
\label{app:gate-cycling-method}
In models with non-overlapping windows ($W=S$), window offset $i$ is also phase $i$. A gate's preference for offset $i$ comes from its gate map $\mZ_i^B$ and positional gate bias $\vb_i^B$ (\Cref{sec:setup}). Cycling by $\delta$ moves both forward by $\delta$ offsets, while the payload maps stay fixed:
\begin{equation}\label{eq:gate-cycling}
\widetilde{\mZ}_i^B=\mZ^B_{(i-\delta)\bmod W},\qquad \widetilde{\vb}_i^B=\vb^B_{(i-\delta)\bmod W},\qquad B\in\{K,V\}.
\end{equation}
The gate parameters of offset $i$ are thus applied at offset $(i+\delta)\bmod W$. We move maps and biases together, not the biases alone, because both enter the gate logits. The hidden states stay at their original offsets, so this does not move the realized gate scores exactly, and whether the preferences move with the parameters is an empirical question. If they do, and retrieval follows the gates, the accuracy pattern should shift by $\delta$ phases,
\begin{equation}\label{eq:gate-cycling-prediction}
A_\delta(r)\approx A_0\big((r-\delta)\bmod S\big),
\end{equation}
where $A_0(r)$ and $A_\delta(r)$ are the accuracies at source-key phase $r$ before and after cycling. The prediction concerns associations within one window. Each value immediately follows its key, so the two share a window except at the boundary phase $r=S-1$, which we exclude from $R^2$ as specified for each experiment.

In each experiment, we cycle the key and value gates of every head in layers 3 to 23 (zero-based) at step 47,518 and measure full-vocabulary top-1 accuracy on 1,000 logical examples of the retrieval task in \Cref{app:small-model-niah} with Prefix Padding, folding the 24 source-key residues modulo $S$. $R^2$, the squared Pearson correlation between predicted and observed accuracy changes across phases, measures linear agreement rather than agreement with the identity line, which the figures draw. Intervals are bootstrapped over logical examples.

\subsection{One-Offset Cycles in a W12/S12 Model}
We first cycle a separate W12/S12 model with four KV heads per layer by one offset in each direction. Away from the boundary phase, both cycles shift the accuracy pattern approximately as predicted (\Cref{fig:phase-interventions}), with $R^2$ of 0.984 for $\delta=+1$ and 0.995 for $\delta=-1$ over comparisons in which phase 11 is neither the observed nor the source phase. Intervals use 2,000 bootstrap resamples. This supports a causal contribution of gate preferences to within-window retrieval in this model. Phase 11 departs from the prediction under both cycles, which limits a purely within-window account but does not by itself identify a separate retrieval circuit.

\subsection{Replication across Models and Cycle Sizes}
One model and one-offset cycles leave open whether the shift is specific to this configuration. We therefore repeat the intervention on five models: the four-KV-head W12/S12 model above, a W12/S12 model with one KV head, the W8/S8 reference model of \Cref{sec:qwen3-spec} (seed 42), and two of its seed replications (seeds 43 and 44), each W8/S8 model with one KV head. We apply each cycle $\delta\in\{\pm1,\pm2,\pm3\}$ to the original parameters. Unlike in the one-offset experiment, $R^2$ here excludes only the observed phase $S-1$, and we also report it over all phases. Intervals again use 2,000 bootstrap resamples. For the four-KV-head W12/S12 model, the cycles $\delta=\pm1$ reproduce the accuracies of \Cref{fig:phase-interventions}.

Away from the boundary phase, the observed accuracy patterns largely follow the prediction for every model and cycle (\Cref{fig:phase-cycling-models}). Excluding phase $S-1$, $R^2$ ranges from 0.96 to 1.00 for both W12/S12 models, from 0.88 to 1.00 for the W8/S8 reference model, and from 0.66 to 1.00 for the two seed replications. Editing the gates alone thus largely shifts the weak within-window phases in all five models, most closely in the W12/S12 models. It does not restore the boundary phase, whose accuracy stays at or below 2.9\% under every cycle, compared with 2.7\% to 23.9\% in the uncycled models, even where the prediction is high. Including this phase lowers $R^2$ to between 0.36 and 0.98. Cycling also lowers overall full-vocabulary accuracy by 3.3 to 18.4 percentage points, so it redirects the weak within-window phases rather than repairing the model.

Beyond \Cref{fig:phase-cycling-models}, a fourth W8/S8 seed (45) behaves like the others, with $R^2$ excluding phase $S-1$ between 0.68 and 0.99, but agreement weakens as W8/S8 models gain KV heads. The minimum over cycles is 0.75, 0.43, and 0.36 with 2, 4, and 8 KV heads, and with 8 KV heads no cycle exceeds 0.86. Because every head in layers 3 to 23 is cycled together, these results do not show whether individual heads in these models lack preferences that would shift accuracy when cycled. The layer range was carried over from the one-offset experiment rather than selected for each model.

\FloatBarrier

%% file: Appendix/Sections/extended_related_work.tex
\section{Extended Related Work}\label{app:related}
\paragraph{Learned summary memories.}
Compressive Transformer down-samples old activations into a FIFO compressed
memory using stride-matched pooling or convolution
\citep{rae2020compressive}. Gist tokens and Landmark Attention use special
tokens to condense prompts or route random access to blocks
\citep{mu2023gist,mohtashami2023landmark}. LoMA adds an explicit
reconstruction objective \citep{wang2024loma}. Activation Beacon inserts
regularly spaced learned beacon tokens whose KVs replace raw activations, and
MELODI builds a recurrent short-term memory across multiple layers and a long-term memory in a single middle layer
\citep{zhang2025activation,chen2025melodi}. CAT is particularly close at the
architectural level: later chunks attend to compressed prior-chunk vectors,
and its appendix reports high loss at the first token of each new chunk and
recall degradation as chunks grow \citep{prakash2025cat}. None of these works
examines retrieval as a function of source phase within compressed memory.

\paragraph{Cross-chunk failure modes.}
\citet{deng2025silver} systematically compare recurrent, coarse-KV, and fine-grained-KV gist compression. Their reconstruction probe shows a severe information bottleneck as compression rises, and their ``lost by the boundary'' diagnostic finds a segment-scale sawtooth: with 2K segments, generation is weakest immediately after a segment boundary, when predictions depend most strongly on compressed prior segments, and improves as raw tokens accumulate in the current segment. They also document semantically surprising facts being dropped and long exact strings being lost during generation. KSA's chunk-aligned sliding window addresses a separate case in which a partially visible chunk is neither locally complete nor yet represented by a summary \citep{chu2026ksa}. HMT uses the phrase ``periodical recall pattern'' for successful tracking of an artificially periodic topic-switch sequence, not an architecture-induced failure \citep{he2025hmt}. Among these works, the sawtooth of \citet{deng2025silver} is the direct precedent for recurring cross-chunk degradation, but its phase variable is the \emph{query/generation position} within the current 2K segment. Our phase variable is the \emph{source position} modulo the much smaller compression stride inside already-compressed memory.

\paragraph{Sparse and hybrid routing.}
Sparse Transformer explicitly defines a strided head with modulo connectivity
and notes that the pattern may route poorly on nonperiodic text
\citep{child2019sparse}. It is thus a precedent for residue-class visibility as a general
attention concept. NSA forms overlapping compressed block
representations for coarse selection and retains an exact fine-grained branch
\citep{yuan2025nsa}. MoBA routes queries to selected raw-token blocks
\citep{lu2025moba}, and the Sparse Frontier compares the quality--efficiency trade
space across such sparse patterns \citep{nawrot2025sparsefrontier}. DeepSeek-V4
selects sparsely among compressed KV entries and adds an uncompressed local
window \citep{deepseek2026v4}. In hybrids with an exact branch over distant
tokens, such as NSA, that branch could mask a failure of the compressed branch
alone. DeepSeek-V4 keeps exact tokens only in its local window, and in
\Cref{sec:modsens}, we find phase sensitivity in the full models.

\paragraph{Cache allocation and specialized retrieval components.}
H2O and SnapKV evict or cluster tokens using estimated importance
\citep{zhang2023h2o,li2024snapkv}. SqueezeAttention allocates different cache
budgets across layers \citep{wang2024squeeze}, while RazorAttention and
DuoAttention retain long-range caches for retrieval heads and use compressed or
streaming policies elsewhere \citep{tang2024razor,xiao2024duo}.
\citet{wu2025retrieval} causally identify a small subset of heads that dominate
long-context factual retrieval. Generic heterogeneity across heads and layers is
therefore already established. Our empirical claim is more specific: KV-head
knockouts in compressed models have phase-dependent effects that, in the
examined heads, match the phases their gates prefer. A supplementary
gate-cycling intervention on five models with non-overlapping windows complements this association (\Cref{app:phase-cycling}).

\paragraph{Retrieval diagnostics and position bias.}
Zoology's multi-query associative recall isolates a major capability gap
between attention and efficient sequence models \citep{arora2023zoology}. The
copying results of \citet{jelassi2024repeat} and induction-circuit studies
\citep{olsson2022induction,bietti2023birth} motivate our synthetic retrieval task.
Lost in the Middle establishes broad U-shaped sensitivity to absolute target
position \citep{liu2024lost}, and RULER shows that easy needle tests can
overstate usable context length \citep{hsieh2024ruler}. The Pitfalls of KV
Cache Compression further shows that aggregate scores hide instruction- and
method-specific failures \citep{chen2026pitfalls}. We therefore match
the mean and variance of the target's absolute position across residue groups
and report retrieval by phase rather than only mean accuracy.

\paragraph{Optimization and specialization.}
Max-margin analyses show that gradient descent can turn softmax attention into
hard token selectors \citep{tarzanagh2023svm,tarzanagh2023maxmargin}, although entropy
concentration can also be pathological \citep{zhai2023entropy}.
\citet{chen2024dynamics} prove staged multi-head task allocation in in-context
linear regression, and \citet{yang2025twomixture} analyze population gradient
flow in a symmetric two-headed classification transformer. Most closely related,
\citet{yuksel2026incremental} analyze near-symmetric heads that first compete
for a dominant sparse pattern and later specialize through competition and
cooperation. High-dimensional SGD likewise exhibits sequential head
specialization \citep{sagitova2026specialization}, while generic softmax
factorization can bias gradient flow toward low-entropy outputs
\citep{varre2026polarizes}. We therefore do not claim the first theory of
attention hardening, symmetry breaking, or division of labor. Our theory
(\Cref{sec:theory}) concerns an idealized induction model of retrieval from
compressed memory. In this model, optimal compressors select sharply, gradient
flow from a small initialization reaches, with high probability, a selection that
persists across inputs, and nothing guarantees
that heads choose distinct positions.